\documentclass{article}

\usepackage{arxiv}
\usepackage{cmap}
\usepackage[T1]{fontenc}
\usepackage[utf8]{inputenc}
\usepackage[english]{babel}
\usepackage{amsmath,amsthm,amsfonts,amssymb}
\usepackage{mathtools}
\usepackage{bm}
\usepackage{enumitem}
\usepackage{booktabs}
\usepackage{array}
\usepackage{multicol}
\usepackage{algorithm}
\usepackage{algpseudocode}
\usepackage{float}
\usepackage{microtype}
\usepackage[numbers,square,sort&compress]{natbib}
\usepackage{url}
\usepackage{xcolor}
\usepackage{hyperref}

\DeclareFontShape{T1}{ptm}{m}{scit}{<->ssub * ptm/m/sc}{}
\hypersetup{colorlinks=true,linkcolor=blue!55!black,citecolor=blue!55!black,
            urlcolor=blue!55!black,
            pdftitle={A Convergence Framework for Deep V-Learning: Error Propagation and Sharp Action-Gap Bounds},
            pdfauthor={Yury Kolomeytsev},
            pdfkeywords={reinforcement learning, deep learning, deep V-learning,
                         Markov decision processes, convergence analysis,
                         error propagation, Bellman residuals, concentrability coefficients,
                         action-gap regularity}}

\newcommand{\R}{\mathbb{R}}
\newcommand{\E}{\mathbb{E}}
\newcommand{\Prob}{\mathbb{P}}
\newcommand{\Indic}{\mathbf{1}}
\newcommand{\Stilde}{\widetilde{\mathcal{S}}}
\newcommand{\stilde}{\tilde{s}}
\newcommand{\geff}{\gamma}
\newcommand{\Tcal}{\mathcal{T}}
\newcommand{\Tpi}{\Tcal^{\pi}}
\newcommand{\Vstar}{V^{\star}}
\newcommand{\Vbar}{\bar{V}}
\newcommand{\FVn}{\mathcal{F}^{V}_{n}}
\newcommand{\Gclk}{\mathcal{G}^{\mathrm{clk}}_0}
\newcommand{\Vmax}{V_{\max}}
\newcommand{\Rmax}{R_{\max}}
\newcommand{\eps}{\varepsilon}
\newcommand{\Aspace}{\mathcal{A}}
\newcommand{\Pkernel}{\mathcal{P}}
\newcommand{\Phat}{\widehat{P}}
\newcommand{\Stildeo}{\Stilde^{\circ}}
\newcommand{\Pjoint}{\mathsf{P}}
\newcommand{\Bb}{B_{b}}
\newcommand{\epsactr}[1]{\eps_{\mathrm{act},k}^{(#1)}}
\newcommand{\epsact}{\epsactr{2}}
\newcommand{\epsactp}{\epsactr{p}}
\newcommand{\epsbuf}{\eps_{\mathrm{buf},k}}
\newcommand{\Vrho}{V_{\max,k}^{(\rho)}}
\newcommand{\dexp}{\Delta^{\mathrm{exp}}_k}
\newcommand{\Dexp}{\bar\Delta^{\mathrm{exp}}_k}
\newcommand{\Dbar}{\bar D^{\mathrm{exp}}_k}
\newcommand{\Vrhobar}{\overline{V}^{(\rho)}_{\max}}
\newcommand{\subopt}{\mathrm{subopt}}
\newcommand{\Arun}{\mathsf{A}_{\mathrm{run}}}
\newcommand{\Aiid}{\mathsf{A}_{\mathrm{iid}}}
\newcommand{\Alev}{\mathsf{A}_{\mathrm{iid}}^{\mathrm{lev}}}
\newcommand{\Rabs}{\mathsf{R}}
\newcommand{\Gen}{\mathsf{G}}
\newcommand{\obj}[1]{\textnormal{\ \small$[\,#1\,]$}}
\newcolumntype{P}[1]{>{\raggedright\arraybackslash}p{#1}}
\newcommand{\etascore}{\eta_{\mathrm{sc}}}
\newcommand{\etascorek}{\eta_{\mathrm{sc},k}}
\newcommand{\etascoreK}{\eta_{\mathrm{sc},K}}

\newcommand{\dnet}{\delta_{V,k}}
\newcommand{\dpath}{\delta^{\mathrm{path}}_{V,k}}
\newcommand{\dsnap}{\delta^{\mathrm{snap}}_{V,k}}
\newcommand{\pibar}{\bar\pi^{\mathrm{on}}_k}
\newcommand{\Qoracle}{Q^{\mathrm{or}}}

\newtheorem{theorem}{Theorem}[section]
\newtheorem{lemma}[theorem]{Lemma}
\newtheorem{proposition}[theorem]{Proposition}
\newtheorem{remark}[theorem]{Remark}
\newtheorem{corollary}[theorem]{Corollary}
\newtheorem{definition}[theorem]{Definition}
\newtheorem{assumption}[theorem]{Assumption}

\title{A Convergence Framework for Deep $V$-Learning: Error Propagation and Sharp Action-Gap Bounds}
\author{Yury Kolomeytsev\\
  {\normalfont Faculty of Computational Mathematics and Cybernetics}\\
  {\normalfont Lomonosov Moscow State University}\\
  {\normalfont\href{mailto:yury.kolomeytsev@gmail.com}{\textcolor{black}{\texttt{yury.kolomeytsev@gmail.com}}}}}
\date{}
\renewcommand{\headeright}{}
\renewcommand{\undertitle}{}
\renewcommand{\shorttitle}{Deep $V$-Learning: Convergence Framework}

\begin{document}
\maketitle

\begin{abstract}\noindent
We establish convergence bounds for deep $V$-learning with horizon $H$.
The algorithm fits a scalar value function to targets from executed
transitions and selects actions using a predictive model and the learned
value function.  For current observed-successor targets with fresh
true-kernel outcomes, the conditional mean is $\Tcal^\beta V$, which averages
over the behavior policy's actions.
The Bellman optimality update is $\Tcal V$.  We decompose the update error
into six residuals: fitting, transition reuse, target construction, replay,
action selection, and exploration.  Under $L^s$ concentrability, their
$L^p$ norms, with $p=s/(s-1)$, control expected $L^1$ policy loss.
The bound has explicit weights on residuals from only the most recent
$H-1$ update blocks, plus an initialization term for shorter runs.
We quantify the cost of a shared sampling distribution across horizon levels.
For statistical error bounds proportional to $n^{-\nu}$, we derive optimal
continuous allocations and an integer allocation whose statistical objective
is within a factor $2^\nu$ of the constrained optimum.
A margin condition with exponent $\alpha$ gives action error of order
$\Lambda^{1+\alpha/p}$, where $\Lambda$ combines network drift and score
error; a matching one-step construction proves the exponent sharp.
Bounds on the distance between frozen and optimal scores transfer an
optimal-gap condition to frozen-iterate gap bounds while retaining the mass
of optimal ties.  Survival probabilities and coverage conditions at
deployment yield bounds for policies selected with approximate scores.
Separate spatial ReLU networks for each horizon level give a conditional
neural regression rate, and the finite-state specialization gives a log-free
expected fit rate.  These results establish expected policy-loss consistency
for the fixed-horizon generative-reset approximate-ERM procedure with exact
action scores and provide an explicit residual-decay criterion for
FIFO/interleaved SGD.
\end{abstract}


\section{Introduction}\label{sec:intro}

Deep $V$-learning fits a scalar state-value function to targets from executed
transitions and selects actions using a predictive model and the learned
value function.  Its population target can differ from the Bellman optimality
target used in fitted value
iteration \cite{munos_szepesvari_2008}.  With a
frozen bootstrap $V$ and fresh true-kernel outcomes, its observed-successor
target has population response $\Tcal^\beta V$ for the acting policy
$\beta$; policy improvement is governed by $\Tcal V$.  At any state where
$\beta$ assigns positive probability to an action that does not maximize
$Q_V$, the two operators differ, even with unlimited data and exact
optimization.

We derive a finite-horizon policy-loss bound by decomposing the error
relative to the Bellman optimality update into six terms: fitting, transition
reuse, target construction, replay, action selection, and exploration.
The action-selection term accounts for network drift and score approximation.
Bounds for a particular replay scheme, optimizer, or state representation
enter the policy-loss theorem through the corresponding residuals.

Model-based action selection with a scalar value function appears in CADRL
and SA-CADRL \cite{CADRL,SACADRL}.  In the SARL pseudocode, a stored transition receives a
frozen-target label when it is sampled for a minibatch, the
\textsc{sample}/\allowbreak\textsc{obs} convention.  EB-CADRL forms
the corresponding observed-successor label when the transition is stored,
the \textsc{store}/\allowbreak\textsc{obs} convention
\cite{SARL,EBCADRL}.  These choices determine whether the bootstrap uses the
current or an older target network.  Building on EB-CADRL,
HMP-DRL~\cite{HMPDRL} uses deep $V$-learning for local control in long-range
navigation.  It incorporates checkpoints from a graph-based global planner
into the state representation and reward function.
For the analysis, a full simulator or
belief state is augmented with the remaining-step clock whenever it satisfies
\eqref{eq:markov_sufficiency}.  Other observation vectors are treated as
measurable compressions, with aliasing, closure, and score errors handled by
Proposition~\ref{prop:representation_routes}.

\subsection{Algorithms and analysis}\label{subsec:objects}

We distinguish the replay-based algorithm $\Arun$, the generative-reset
procedure $\Aiid$, and the residual recursion $\Rabs$ used to bound policy
loss.  The recursion applies whenever the six errors satisfy their stated
envelopes.  For $\Aiid$, we derive statistical fit bounds using fresh blocks
from a single reset distribution.  For $\Arun$, bounds on the six residuals
connect the implemented updates to the same recursion:
\begin{equation}
\Arun\ \xrightarrow{\ \text{six residual bounds}\ }\ \Rabs,
\qquad
\Aiid\ \xrightarrow{\ \text{statistical fit bound}\ }\ \Rabs .
\label{eq:object_contract}
\end{equation}
For the matched-budget comparison, $\Alev$ draws samples separately at each
horizon level.  Statement tags identify the setting or procedure to which a
result applies; $\Gen$ denotes the general setting.

\paragraph{Main results.}
\begin{enumerate}[label=\textnormal{(C\arabic*)},leftmargin=3.1em,itemsep=2pt,topsep=3pt]
\item \emph{Convergence through residual propagation.}
The expected policy-loss bound is a weighted sum of the six residual
envelopes and an initialization term.  Its explicit weights vanish outside
the most recent $H-1$ update blocks, and the initialization term vanishes
for $K\ge H-1$ (Theorem~\ref{thm:hmpdrl_end_to_end}).
\item \emph{Sharp horizon coefficients and sample allocation.}
We characterize the horizon dependence of the shared-law propagation
coefficient and construct a family that attains the bound
(Theorem~\ref{prop:conc_budget} and Corollary~\ref{prop:clock_scaling}).
Direct level sampling gives a separate propagation bound and an exact sample
count (Theorem~\ref{thm:level_indexed}).  Under matched label budgets, we
derive the optimal continuous allocation, its constrained water-filling
solution, and an integer schedule within a factor $2^\nu$ of the constrained
statistical optimum (Theorem~\ref{thm:matched_budget}).
The resulting horizon orders are given in
Corollary~\ref{cor:matched_budget_geometry}.
\item \emph{Sharp action-gap and deployment bounds.}
The action-residual exponent depends on the norm used to measure error.
We derive these exponents and prove one-step sharpness
(Theorem~\ref{thm:act_upper} and Proposition~\ref{prop:act_tight}).
An optimal-gap condition and bounds on score approximation yield gap bounds
for the frozen iterates while accounting for optimal ties
(Proposition~\ref{prop:reference_margin_transfer}).
A separate construction shows that score error can sustain a positive
policy-loss floor (Proposition~\ref{prop:model_floor_lower}).
We also bound the loss of the controller selected by the implemented score
(Theorem~\ref{thm:deployed_transfer} and
Corollary~\ref{cor:deployed_transfer_margin}).
\item \emph{Nonasymptotic neural and tabular convergence.}
For the generative-reset procedure, we obtain neural regression and
policy-loss rates, along with a log-free expected fit rate in the tabular
case (Proposition~\ref{prop:vreg_rate}, Theorem~\ref{thm:nonasymptotic}, and
Corollary~\ref{cor:tabular_discharge}).  With exact action scores, the
procedure achieves expected policy-loss consistency
(Corollary~\ref{cor:fresh_consistency}).  We verify Bellman closure and the
network-class assumptions in a non-tabular finite-rank model
(Propositions~\ref{prop:finite_rank_closure}
and~\ref{prop:routed_relu_admissibility}).
\end{enumerate}

Action-indexed regression retains the executed action as an input.
Expected SARSA uses this structure and averages the next action
\cite[\S\S6.1, 6.6]{sutton_barto_2018}.  Remark~\ref{rem:general_target}
contrasts these targets with state-only regression, and
Section~\ref{subsec:related} compares the
corresponding fitted-value, neural-$Q$, action-gap, and replay analyses.

\section{Model and algorithm}\label{sec:algorithm}

Algorithm~\ref{alg:hmpdrl_training} specifies the replay-based procedure
$\Arun$.  The analysis distinguishes six design choices:
\begin{enumerate}[label=\textnormal{(S\arabic*)},leftmargin=2.8em,nosep]
\item a scalar value head $V_\theta$;
\item a one-step score built from an implemented reward and transition model,
\begin{equation}
\widehat Q_V(\stilde,a):=\widehat R(\stilde,a)+\geff V\bigl(\Phat(\stilde,a)\bigr),
\qquad
\widehat a^{\star}(\stilde;V):=\arg\max_{a\in\Aspace}\widehat Q_V(\stilde,a),
\label{eq:greedy_action}
\end{equation}
used in place of $\max_aQ$;
\item the network, online or frozen, at which that score is evaluated for
\emph{behavior};
\item a replay law;
\item an exploration rule;
\item a \emph{target construction}: \textsc{when} the bootstrap is evaluated
(at sampling time, or once at storage time) and \textsc{from} which successor
(the observed one, or a model prediction).
\end{enumerate}
Slots~(S3) and~(S6) generate the drift and target-construction residuals.

\begin{algorithm}[H]\footnotesize
\caption{Deep $V$-learning template for $\Arun$.}
\label{alg:hmpdrl_training}
\begin{algorithmic}[1]
\State \textbf{Input:} $\geff,\{\eps_k,U_k,c^{\mathrm{env}}_{k,j},g_{k,j}\}$, $B,B_{\mathrm{mb}},\xi_k$; fix normalization and the two switches in~\textnormal{(S6)}
\State Initialize clipped $V_\theta$, FIFO buffer $E$, and $\hat V\leftarrow\Vbar_\theta$ \Comment{$V_0$}
\For{target period $k=0,1,2,\ldots$}
  \State Freeze $V_k\leftarrow\hat V$; set $V_{\theta_{k,0}}\leftarrow\Vbar_\theta$
  \For{round $j=0,1,\ldots,U_k-1$}
    \State Collect $c^{\mathrm{env}}_{k,j}$ transitions with the online $\eps_k$-greedy score; store raw tuples or stored targets according to~\textnormal{(S6)}, evicting beyond $B$
    \For{$g_{k,j}$ minibatches from $E$}
      \State If \textsc{sample}, form $y_i=r_i+\geff(1-d_i)\hat V(\mathsf{succ}_{\textsc{from}}(x_i,\mathbf a_i,x_i'))$; take one semi-gradient step on $B_{\mathrm{mb}}^{-1}\sum_i(\Vbar_\theta(x_i)-y_i)^2$
    \EndFor
  \EndFor
  \State \textbf{Copy:} $\hat V\leftarrow\Vbar_\theta$ \Comment{$V_{k+1}$}
\EndFor
\end{algorithmic}
\end{algorithm}

\noindent Here $\mathsf{succ}_{\textsc{obs}}(x,\mathbf a,x')=x'$,
$\mathsf{succ}_{\textsc{pred}}(x,\mathbf a,x')=\Phat(x,\mathbf a)$, and a
stored target uses both the target network and, for \textsc{pred}, the
predictor installed at collection time.  The observed mask makes terminal
labels equal their observed rewards.  The target network is constant within
a period and treated as a constant in the semi-gradient; the behavior network
may drift.

\paragraph{Generative-reset variant.}  For $\Aiid$, fix the frozen $V_k$, an
acting snapshot $W_k\in\mathcal V^{\mathrm{clip}}$, and its $\eps_k$-greedy
law $\beta_k$, then draw conditionally independently
\begin{equation}
S_i\sim\varsigma,\quad
\mathbf a_i\sim\beta_k(\cdot\mid S_i),\quad
(r_i,\stilde'_i,d_i)\sim\Pjoint(\cdot\mid S_i,\mathbf a_i)\ \text{fresh},\quad
Y_i=r_i+\geff(1-d_i)V_k(\stilde'_i),
\label{eq:generative_variant}
\end{equation}
for $i\le n_k$.  Each outcome is used once, and $V_{k+1}$ is a measurable
$\zeta_{n_k}$-approximate ERM over $\mathcal F^{V,0}_{n_k}$.  Here
$\rho_{k,S}=\varsigma$ by Assumption~\ref{ass:generative}.  This defines the
generative-reset procedure $\Aiid$ in~\eqref{eq:object_contract}.  It fits the
scalar target generated by the executed action; action-indexed FVI uses a
different response~\cite{munos_szepesvari_2008}.  The variant $\Alev$ samples
directly from the slice laws of
Theorem~\ref{thm:level_indexed}.

\paragraph{Score-access convention.}
A reset call in~\eqref{eq:generative_variant} returns one observed
reward--successor pair.  Access to the conditional expectation
$Q_V(\stilde,a)=R(\stilde,a)+\geff\int V\,d\Pkernel(\cdot\mid\stilde,a)$ is
represented by the exact-score oracle $\Qoracle_V:=Q_V$, distinct from the
point score $\widehat Q_V$ in~\eqref{eq:greedy_action}.  Every result using
$\Qoracle$ assumes that additional access and replaces slot~(S2) by
$\Qoracle$; exact reward and successor maps provide it for a deterministic
model.  Otherwise a quantitative score estimator enters the residual bound
through $\etascorek$; Appendix~\ref{supp:oa6} gives one Monte Carlo bound under
the stated stronger query access.

\noindent We use three histories: $\mathcal F_{k,j}$ for literal
rounds; $\mathcal H_k$ for the frozen $\Aiid$ objects before its conditionally
i.i.d.\ block; and $\mathcal J_k$ for the abstract block.  The latter is
pre-block information under population accounting and $\mathcal B_k$ under
realized-buffer accounting; for $\Aiid$, $\mathcal J_k=\mathcal H_k$.  Each
history is generated by a random element taking values in a standard-Borel
space.  All design/replay laws and responses in one theorem application are
measurable for that same history.  Normalization is fixed and every network is
clipped by~\eqref{eq:levelwise_clip}.

\begin{table}[H]\scriptsize\centering
\setlength{\tabcolsep}{3pt}
\renewcommand{\arraystretch}{1.08}
\caption{Conditioning guide.  ``Deterministic'' means uniform over the
histories in the theorem application.}\label{tab:notation_status}
\begin{tabular}{@{}P{0.20\linewidth}P{0.34\linewidth}P{0.40\linewidth}@{}}
\toprule
\textbf{Role} & \textbf{Principal symbols} & \textbf{Probability status and use} \\
\midrule
Histories and frozen objects & $\mathcal F_{k,j},\mathcal H_k,\mathcal J_k$;
$V_k,W_k,\rho_{k,S},G_k$ & Information before a literal round, a fresh block,
or an abstract block; the latter objects are $\mathcal J_k$-measurable. \\
Responses and residuals & $a_{\mathrm{alias},k},\Dexp$;
$\eps_{\mathrm{fit},k},\eps_{\mathrm{ker},k},\eps_{\mathrm{tgt},k},
\epsbuf,\epsactp,\eps_k\Dbar$ & The first two quantities are
$\mathcal J_k$-measurable.  The fit envelope bounds a conditional mean; the
other five residual envelopes and $\Dbar$ are deterministic almost-sure bounds. \\
Propagation & $\mathcal B_K,w^{(H)}_{K,k},d_s(m),c_2^{\mathrm{lev}}(m)$ &
Deterministic boundary, residual weights, and coverage coefficients. \\
Terminal-window allocation & $n_i,L_i,\mathsf N,b_i,\nu,
\mathsf C_{\nu,K},\Psi_\nu,F_K$ & Planned accepted-label allocations,
validity thresholds, and statistical and floor constants in
Theorem~\ref{thm:matched_budget}. \\
\bottomrule
\end{tabular}
\end{table}

\subsection{Model assumptions, clipping, and the \texorpdfstring{$V$}{V}-backup}\label{sec:setting}

\noindent\textbf{Standing assumptions.}  Throughout,
\begin{equation}
H\in\mathbb N,\qquad 0<\geff<1,\qquad 0<\Rmax<\infty,\qquad \Vmax^{(0)}:=0 ,
\label{eq:standing}
\end{equation}
and $2\le |\Aspace|<\infty$.  The physical state space
$(\mathcal S,\Sigma_{\mathcal S})$ is nonempty and standard Borel,
$\Aspace$ has the discrete $\sigma$-algebra and a fixed total order, and
\[
\Stilde=(\mathcal S\times\{1,\ldots,H\})\sqcup
\{\stilde_{\mathrm{term}}\}
\]
has the corresponding disjoint-union standard-Borel structure.  Observation
and parameter spaces used below are also standard Borel.  Put
\begin{equation}
\Stildeo:=\Stilde\setminus\{\stilde_{\mathrm{term}}\},\qquad
\mu_S^\circ(B):=\mu_S(B\cap\Stildeo),\qquad
R(\stilde,a):=\int r\,\Pjoint(dr,d\stilde'\mid\stilde,a).
\label{eq:nonterminal_objects}
\end{equation}
For the next-state marginal $\Pkernel$ of the joint kernel in
Assumption~\ref{ass:markov_state}, define, for measurable $B\subseteq\Stildeo$,
\begin{equation}
\Pkernel_\circ(B\mid\stilde,a):=\Pkernel(B\mid\stilde,a),\qquad
\Pkernel^\pi_\circ(B\mid\stilde):=
\int_{\Aspace}\Pkernel_\circ(B\mid\stilde,a)\,\pi(da\mid\stilde).
\label{eq:nonterminal_kernel}
\end{equation}
Thus $\mu_S^\circ$ is a subprobability restriction and
$\Pkernel^\pi_\circ$ is a substochastic kernel that integrates only over the
nonterminal space.  Values vanish at $\stilde_{\mathrm{term}}$, so $L^1(\mu_S)$
and $L^1(\mu_S^\circ)$ losses coincide.  Function inequalities are pointwise
unless a measure is named.  The \emph{action gap} is
$\Delta_Q^{(k)}(\stilde):=\max_aQ_{V_k}(\stilde,a)-\max_{a\ne a^{\mathrm{tgt}}_k(\stilde)}Q_{V_k}(\stilde,a)$
at the frozen $Q_{V_k}$, with $a^{\mathrm{tgt}}_k$ selected by the fixed Borel
tie-breaking rule, so that $\Delta_Q^{(k)}=0$ exactly at a tie.
For comparison with a fixed optimal-gap hypothesis, write
\begin{equation}
\Delta_Q^\star(\stilde):=
\max_aQ_{\Vstar}(\stilde,a)
-\max_{a\ne a^\star(\stilde;\Vstar)}Q_{\Vstar}(\stilde,a),
\label{eq:optimal_gap}
\end{equation}
using the same rule; again $\Delta_Q^\star=0$ exactly at an optimal tie.

The horizon is $H$ steps and the augmented state is $\stilde=(s,h)$ with $h$ the
number of steps remaining.  A policy called stationary below is stationary on
this augmented state; viewed on the physical state alone, its clock dependence
makes it nonstationary.  Rewards obey $|r|\le\Rmax$.
For every measurable stationary augmented-state policy $\pi$, let $V^\pi$ be
the unique fixed point of $\Tpi$.  Bellman optimality gives
$\Vstar\ge V^\pi$ pointwise, so the expected discounted regret from the
evaluation law is exactly
\begin{equation}
\|\Vstar-V^\pi\|_{1,\mu_S}
=\int(\Vstar-V^\pi)\,d\mu_S .
\label{eq:policy_loss_regret}
\end{equation}
The level-$h$ value bound and its recursion are
\begin{align}
\Vmax^{(h)}&:=\Rmax\,\frac{1-\geff^{\,h}}{1-\geff},\qquad \Vmax:=\Vmax^{(H)},
\label{eq:vmax_level}\\
\Rmax+\geff\,\Vmax^{(h-1)}&=\Rmax\,\frac{(1-\geff)+\geff-\geff^{\,h}}{1-\geff}
=\Vmax^{(h)}\qquad(1\le h\le H).
\label{eq:vmax_recursion}
\end{align}
Equality in~\eqref{eq:vmax_recursion} is the recursion used to prove that the
level-wise band below is $\Tcal$-invariant.  To agree with terminal masking,
work on
\[
\begin{aligned}
\mathcal V&:=\bigl\{V:\Stilde\to\R\ \text{bounded and measurable}:\
V(\stilde_{\mathrm{term}})=0\bigr\}\subset\Bb(\Stilde),\\
\|V\|_\infty&:=\sup_{\stilde}|V(\stilde)|.
\end{aligned}
\]
This is a closed Banach subspace of bounded measurable functions
\cite{kreyszig1989functional}.  A fixed Borel tie rule supplies measurable
greedy selectors.  Every network is evaluated through
\begin{equation}
\Vbar_\theta(s,h)=\operatorname{clip}\bigl(V^{\mathrm{raw}}_\theta(s,h),
-\Vmax^{(h)},\Vmax^{(h)}\bigr),
\label{eq:levelwise_clip}
\end{equation}
whose image is the band
$\mathcal V^{\mathrm{clip}}:=\{V\in\mathcal V:|V(s,h)|\le\Vmax^{(h)}\ \forall(s,h)\}$.
The Bellman optimality operator and the induced one-step $Q$-function are
\begin{equation}
(\Tcal V)(\stilde):=\max_{a\in\Aspace}Q_V(\stilde,a),\qquad
Q_V(\stilde,a):=R(\stilde,a)+\geff\int V(\stilde')\,\Pkernel(d\stilde'\mid\stilde,a),
\label{eq:bellman_def}
\end{equation}
and $\Tpi$ denotes the policy operator obtained by averaging $Q_V$ over $\pi$
instead of maximizing.  The next three results establish the clipping
property, the $V$-backup identity, and existence of the fixed point $\Vstar$.

\begin{assumption}[Markov-sufficient analysis state, joint kernel, and terminal convention\obj{\Gen}]\label{ass:markov_state}
Let $\mathcal H_t^{\mathrm{sys}}$ denote the full controlled history of the
underlying system, including latent variables when the analysis uses them; the
resulting analysis state need not be the implemented input.  The augmented
representation is required to be controlled Markov
for the joint reward--successor law: there are a measurable map $\Psi$ and a
measurable kernel $\Pjoint(dr,d\stilde'\mid\stilde,a)$ on $\R\times\Stilde$
such that $\stilde_t=\Psi(\mathcal H_t^{\mathrm{sys}})$ includes the
remaining-step clock and, under every admissible control law, for every bounded
measurable $f:\R\times\Stilde\to\R$,
\begin{equation}
\E\!\left[f(r_t,\stilde_{t+1})\mid
\mathcal H_t^{\mathrm{sys}},a_t\right]
=\int f(r,\stilde')\,\Pjoint(dr,d\stilde'\mid\stilde_t,a_t)
\quad\text{a.s.}
\label{eq:markov_sufficiency}
\end{equation}
This controlled-kernel identity includes deterministic interventions and does
not condition on a possibly zero-probability event $\{a_t=a\}$.  Thus histories
with the same represented state have the same conditional joint law under
every action.  This is the Markov-sufficiency requirement; appending a clock
supplies the temporal coordinate, while the controlled-kernel identity supplies
the required state sufficiency.  Rewards satisfy $|r|\le\Rmax$
a.s.  Write $\Pkernel$ for the next-state marginal.  Every
terminating transition has successor $\stilde_{\mathrm{term}}$, and
$\Pjoint(\{0\}\times\{\stilde_{\mathrm{term}}\}\mid\stilde_{\mathrm{term}},a)=1$.
Define the terminal mask as the following function of the successor:
\begin{equation}
d:=\Indic\{\stilde'=\stilde_{\mathrm{term}}\} .
\label{eq:mask_def}
\end{equation}
The remaining-step coordinate is a genuine clock: from
$(s,h)$ with $h>1$ every \emph{nonterminal} successor lies one level lower,
while the episode may terminate at any level; from $h=1$ absorption is certain,
\begin{equation}
\Pjoint\!\Bigl(\R\times\bigl((\mathcal S\times\{h-1\})\cup\{\stilde_{\mathrm{term}}\}\bigr)
\;\Big|\;(s,h),a\Bigr)=1\ \ (h>1),
\qquad
\Pjoint\!\bigl(\R\times\{\stilde_{\mathrm{term}}\}\mid(s,1),a\bigr)=1 .
\label{eq:time_decrement}
\end{equation}
Hence the substochastic nonterminal kernel $\Pkernel^\pi_\circ$ of
\S\ref{subsec:propagation} is nilpotent, and so is every product of $H$ of
them: $\Pkernel^{\pi_1}_\circ\cdots\Pkernel^{\pi_H}_\circ=0$ for policies that
need not coincide.
\end{assumption}

\begin{lemma}[Measurable selectors, mixtures, and conditional responses\obj{\Gen}]
\label{lem:measurable_framework}
Under the standing standard-Borel assumptions:
\begin{enumerate}[label=\textnormal{(\roman*)},leftmargin=2.2em,nosep]
\item the fixed-order maximizer of any jointly measurable real score on
$\Stildeo\times\Aspace$ is measurable;
\item a finite state-dependent convex mixture of measurable Markov policies is
a measurable Markov policy;
\item if
$\mathsf M(d\stilde,da)=\sum_{j=1}^J\omega_j\mu_j(d\stilde)
\pi_j(da\mid\stilde)$ and
$\mathsf M_S=\sum_j\omega_j\mu_j$, where $\omega_j\ge0$ and
$\sum_j\omega_j=1$ (with all these measures and policies
possibly depending measurably on a standard-Borel history), then jointly
measurable versions
$f_j=d(\omega_j\mu_j)/d\mathsf M_S\in[0,1]$ may be chosen.  With
$z=\sum_jf_j$ and any fixed measurable reference policy $\pi_0$,
\begin{equation}
\bar\pi(da\mid\stilde)
:=\begin{cases}
z(\stilde)^{-1}\sum_{j=1}^J f_j(\stilde)
\pi_j(da\mid\stilde),&z(\stilde)>0,\\
\pi_0(da\mid\stilde),&z(\stilde)=0,
\end{cases}
\label{eq:measurable_mixture}
\end{equation}
is a Markov kernel and a version of the conditional action law
$\mathsf M(da\mid\stilde)$;
\item every bounded measurable label generated from a standard-Borel history,
state, action, and fresh kernel draw admits a measurable conditional-response
version $G_k(\stilde)=\E[Y_k\mid S=\stilde,\mathcal J_k]$.
\end{enumerate}
All policies, replay disintegrations, and responses in the sequel refer to
these fixed versions.
\end{lemma}

\begin{proof}
For finite ordered $\Aspace$, each selector event is a finite intersection of
measurable score comparisons, proving~(i); kernel integration and finite sums
give~(ii).  Since $\omega_j\mu_j\le\mathsf M_S$, including when
$\omega_j=0$, bounded density versions exist with $\sum_jf_j=1$
$\mathsf M_S$-a.e.; substitution against measurable rectangles proves~(iii).
The parameterized disintegration theorem on standard-Borel spaces supplies
the jointly measurable conditional kernel (equivalently the displayed finite
mixture weights) and its almost-sure uniqueness.  Integrating the bounded
label against this kernel proves~(iv)
\cite[Proposition~7.27 and Corollary~7.27.1]{bertsekas_shreve_1978}.
\end{proof}

\begin{lemma}[Level-wise clipping is a projection\obj{\Gen}]\label{lem:clip_projection}
Let $\Pi^{\mathrm{clip}}$ be the statewise clip of~\eqref{eq:levelwise_clip}.
For $g\in\mathcal V^{\mathrm{clip}}$,
$|\Pi^{\mathrm{clip}}f-g|\le|f-g|$ pointwise; hence it is the $L^2(\rho)$
metric projection, is nonexpansive in $L^2(\rho)$ and supremum norm, and does
not increase covering or Bellman-approximation error.  Moreover,
\begin{equation}
V\in\mathcal V^{\mathrm{clip}}\quad\Longrightarrow\quad
|Q_V(s,h,a)|\le\Vmax^{(h)},\qquad |Y_i|\le\Vmax,\qquad
\Tcal V\in\mathcal V^{\mathrm{clip}}.
\label{eq:clip_consequences}
\end{equation}
These level-wise facts follow from~\eqref{eq:vmax_recursion}; a single global
clip need not make the same score and label bounds invariant.
\end{lemma}

\noindent\emph{Proof:} See Appendix~\ref{app:proofs}.

\begin{lemma}[Bellman $V$-backup through the induced $Q$-function\obj{\Gen}]\label{lem:v_q_backup}
Under Assumption~\ref{ass:markov_state}, for every $V\in\mathcal V$:
\begin{enumerate}[label=\textnormal{(\roman*)},leftmargin=2.2em,nosep]
\item if $(r,\stilde',d)\sim\Pjoint(\cdot,\cdot\mid\stilde,a)$ with mask $d$,
then $\E[r+\geff(1-d)V(\stilde')\mid\stilde,a]=Q_V(\stilde,a)$;
\item $(\Tcal V)(\stilde)=\max_a Q_V(\stilde,a)$;
\item if $a_t=a^\star(\stilde_t;V)$ then
$\E[r_t+\geff(1-d_t)V(\stilde_{t+1})\mid\stilde_t,a_t]=(\Tcal V)(\stilde_t)$;
\item if $V\in\mathcal V^{\mathrm{clip}}$ and $a_t$ is $\eps$-greedy for $Q_V$
under exploration law $\nu$, the discrepancy equals $\eps$ times the
dispersion of $Q_V$ over $\nu$:
\begin{equation}
(\Tcal V)(\stilde)-\E\bigl[r_t+\geff(1-d_t)V(\stilde_{t+1})\bigm|\stilde_t=\stilde\bigr]
\;=\;\eps\,\Delta^{\mathrm{exp}}[V,\nu](\stilde)\;\ge\;0,
\label{eq:dispersion_identity}
\end{equation}
where
\[
\begin{aligned}
\Delta^{\mathrm{exp}}[V,\nu](\stilde)
&:=\max_aQ_V(\stilde,a)-\int Q_V(\stilde,a)\,\nu(da\mid\stilde),\\
0&\le\Delta^{\mathrm{exp}}[V,\nu](s,h)\le2\Vmax^{(h)}.
\end{aligned}
\]
\end{enumerate}
\end{lemma}

\begin{proof}
(i) is~\eqref{eq:bellman_def} with $V(\stilde_{\mathrm{term}})=0$; (ii) is the
definition; and (iii) follows from (i)--(ii).  For (iv), under exploration law
$\nu$,
an $\eps$-greedy rule puts mass $1-\eps$ on the maximizer and $\eps$ on $\nu$,
so by (i) and (ii)
\[
\E[Q_V(\stilde,a)\mid\stilde]-\max_a Q_V(\stilde,a)
=\eps\bigl(\E_\nu[Q_V]-\max_a Q_V\bigr)
=-\eps\,\Delta^{\mathrm{exp}}[V,\nu](\stilde),
\]
which is~\eqref{eq:dispersion_identity}.  Its sign is fixed because the
maximum dominates any average.  For the statewise
range, the successor of a state at level $h$ lies at level $h-1$ or is terminal,
so the level recursion~\eqref{eq:vmax_recursion} gives
$|Q_V(s,h,a)|\le\Rmax+\geff\Vmax^{(h-1)}=\Vmax^{(h)}$ and hence
$\Delta^{\mathrm{exp}}[V,\nu](s,h)\in[0,2\Vmax^{(h)}]$.
Lemma~\ref{lem:composition} uses this identity.  The upper
endpoint is attained only where the maximum score equals $\Vmax^{(h)}$ and
$\nu$ concentrates on actions whose score is $-\Vmax^{(h)}$.
\end{proof}

\begin{remark}[Executed-transition targets and action-indexed heads\obj{\Rabs}]\label{rem:general_target}
The operator analysis needs only
\textnormal{(E1)} $\E[Y\mid\stilde,a]=Q_{u_k}(\stilde,a)$ for the executed
action and \textnormal{(E2)} a state-only regressand, whose response is
$\Tcal^\beta u_k=\int Q_{u_k}(\cdot,a)\beta(da\mid\cdot)$.  It separates
consistency $\|G_k-\Tcal^\beta u_k\|$ from optimality
$\|\Tcal^\beta u_k-\Tcal u_k\|$.  Thus the abstract results extend to any
head satisfying \textnormal{(E1)}--\textnormal{(E2)}.

A $Q$-head and Expected SARSA index the regressand by the action and therefore
fall outside \textnormal{(E2)}.  One-step state-value TD satisfies it and
evaluates $\Tcal^\beta$ unless policy improvement is added.  The routed
scalar-output entropy bound in Section~\ref{sec:vreg} includes the $H$-head
factor in~\eqref{eq:routed_entropy} and no output-coordinate factor
$|\Aspace|$; its approximation constants may still depend on the action set.
\end{remark}

\begin{lemma}[Existence and uniqueness of $\Vstar$\obj{\Gen}]\label{lem:exist_unique}
Under Assumption~\ref{ass:markov_state} the operator $\Tcal$ is a
$\geff$-contraction on $\mathcal V$ and leaves $\mathcal V^{\mathrm{clip}}$
invariant, so it has a unique fixed point $\Vstar\in\mathcal V^{\mathrm{clip}}$,
with $|\Vstar(s,h)|\le\Vmax^{(h)}$ at every level, and
$\Vstar(\stilde)=\sup_\pi\E_\pi[\sum_{t\ge0}\geff^tR_t\mid\stilde_0=\stilde]$.
\end{lemma}

\begin{proof}
For $V,W\in\mathcal V$, finiteness of $\Aspace$ and kernel integration give
$|\Tcal V(\stilde)-\Tcal W(\stilde)|\le
\geff\max_a\int|V-W|\,d\Pkernel\le\geff\|V-W\|_\infty$.
The terminal convention makes $\mathcal V$ a closed, complete subspace of
$\Bb(\Stilde)$, so Banach's theorem gives a unique fixed point; invariance from
Lemma~\ref{lem:clip_projection} puts it in $\mathcal V^{\mathrm{clip}}$.
Backward induction over the clock bounds every history-dependent randomized
policy by this recursion, and the measurable greedy selector attains it
\cite{bertsekas1996neuro,bertsekas_shreve_1978,hernandez_lerma_lasserre_1996}.
\end{proof}

\section{Residual decomposition and policy-loss propagation}\label{sec:decomposition}

\subsection{Coverage and residuals}
\label{subsec:residual_interface}

Residuals are measured under replay marginals $\rho_{k,S}$, while policy
loss is evaluated from $\mu_S$.  Concentrability bounds the density ratios
between the state distributions reached from $\mu_S$ and the replay laws.
One way to verify Assumption~\ref{ass:concentrability} is domination by a
reset law $\varsigma$:
\begin{equation}
\mu_S^\circ\Pkernel^{\pi_1}_\circ\cdots\Pkernel^{\pi_m}_\circ\;\le\;\bar b\,\varsigma
\qquad\text{for all }1\le m\le H\text{ and all policy sequences},
\label{eq:reset_domination}
\end{equation}
then any replay law of the mixture form
\begin{equation}
\rho_{k,S}:=(1-\kappa)\sigma_{k,S}+\kappa\varsigma,\qquad\kappa\in(0,1],
\label{eq:kappa_mixture}
\end{equation}
satisfies $\rho_{k,S}\ge\kappa\varsigma$ and therefore
$c_{2,k}(m)\le\bar b/\kappa$.  Finite coverage requires reset mass on every
reachable clock slice; Proposition~\ref{prop:conc_budget} quantifies the cost
when one probability law covers all slices.  Sampling i.i.d.\ from the
mixture~\eqref{eq:kappa_mixture} requires arbitrary-state reset access or a
snapshot-replay oracle; under $\Aiid$, Assumption~\ref{ass:generative} fixes
$\rho_{k,S}=\varsigma$.  Lemma~\ref{lem:varying_propagation} pairs an
$L^p(\rho_{k,S})$ residual with an $L^s(\rho_{k,S})$ density, and the same
conjugate H\"older pairing determines the margin exponent in
Proposition~\ref{prop:l2_routing}.

\medskip
\noindent The residuals are defined as envelopes over one block.  Let
$\mathcal I_k^{\mathrm{opt}}$ index the gradient updates in block $k$ and let
$\mathcal I_k^{\mathrm{act}}$ index every online-network snapshot actually
used to select a collection action, including the initial and any post-update
snapshots that act.  Put
$\mathcal I_k=\mathcal I_k^{\mathrm{opt}}\cup\mathcal I_k^{\mathrm{act}}$.
With $\mathcal D:=\Stildeo\times\Aspace$ as the common comparison domain,
\begin{equation}
\begin{aligned}
\sup_{t\in\mathcal I_k^{\mathrm{act}}}\|V_{\theta_{k,t}}-V_k\|_\infty
&\le\dnet,\\
\sup_{V\in\mathcal F}\ \sup_{(\stilde,a)\in\mathcal D}
  \bigl|\widehat Q_{V}(\stilde,a)-Q_{V}(\stilde,a)\bigr|
&\le\etascorek,\\
\bigl\|\Tcal^{\beta^{\mathrm{rep}}_k}V_k-
\Tcal^{\pi^{\mathrm{on}}_k}V_k\bigr\|_{2,\rho_{k,S}}
&\le\epsbuf,
\qquad \Dexp\le\Dbar.
\end{aligned}
\label{eq:envelopes}
\end{equation}
These are network drift, score error, replay shift, and exploration
dispersion.  The symmetric class
$\mathcal F\subseteq\mathcal V^{\mathrm{clip}}$ contains every frozen iterate,
every acting snapshot, and their negatives (e.g.\
$\pm\bigcup_n\FVn$).  Restriction to $\mathcal F$ is essential: over the full
clipped band, fixed kernel mass away from a point prediction can make the
continuation error $\Omega(\Vmax)$ even with exact rewards.  The unindexed
$\etascore:=\sup_k\etascorek$ denotes a deterministic uniform envelope when
finite.  The
notation $\widehat Q_V$ suppresses a possible block index on the learned
model; a fixed implemented model has constant $\etascorek$.  In
\eqref{eq:envelopes} and the action section, $\widehat Q$ denotes the score
installed in slot~(S2): the point score for $\Arun$ and for the point-score
$\Aiid$ analysis, or the distinct $\Qoracle$ only in statements explicitly
tagged as oracle results.  Proposition~\ref{prop:dispersion} concerns the
point-score case.

\begin{proposition}[Dispersion control of the implemented score\obj{\Gen}]\label{prop:dispersion}
Suppose each nonterminal clock slice carries a metric $\mathrm{dist}$ whose
distance map is jointly measurable, the displayed conditional distances are
integrable, and every $V\in\mathcal F$ has the same nondecreasing concave
modulus $\omega$ on that slice, with $\omega(0)=0$.  Suppose also that
termination condition~\textnormal{(T)} of
Lemma~\ref{lem:tgt_residual} holds pointwise on $\mathcal D$, and that the
point model respects the clock and terminal conventions.  With
$R(\stilde,a):=\E[r\mid\stilde,a]$, the block score error satisfies
\begin{equation}
\sup_{V\in\mathcal F}\sup_{(\stilde,a)\in\mathcal D}
\bigl|\widehat Q_V(\stilde,a)-Q_V(\stilde,a)\bigr|
\;\le\;\sup_{(\stilde,a)\in\mathcal D}
\bigl|\widehat R(\stilde,a)-R(\stilde,a)\bigr|
\;+\;\geff\,\omega\Bigl(\sup_{(\stilde,a)\in\mathcal D}\E\bigl[\,\mathrm{dist}\bigl(\stilde',\Phat(\stilde,a)\bigr)
\bigm|\stilde,a\,\bigr]\Bigr).
\label{eq:dispersion_bound}
\end{equation}
Consequently any deterministic almost-sure majorant of the right-hand side is
an admissible choice of $\etascorek$ in~\eqref{eq:envelopes}.
\end{proposition}

\begin{proof}
At terminal pairs both continuations vanish.  Otherwise $\stilde'$ and
$\Phat(\stilde,a)$ lie on the same clock slice, so for every $V\in\mathcal F$,
\[
|\widehat Q_V-Q_V|\le |\widehat R-R|
+\geff\E\omega\!\left(\mathrm{dist}(\stilde',\Phat(\stilde,a))\right)
\le |\widehat R-R|+\geff\omega\!\left(\E\mathrm{dist}(\stilde',\Phat(\stilde,a))\right)
\]
by concavity and Jensen.  Take the suprema over $V$ and $\mathcal D$.
\end{proof}

For deterministic dynamics Proposition~\ref{prop:dispersion} becomes
$\sup_{\mathcal D}|\widehat R-R|+
\geff\omega(\sup_{\mathcal D}\|\Phat-F\|)$.  For stochastic dynamics,
dispersion enlarges only this upper bound.  The common modulus is an
equicontinuity assumption on $\mathcal F$, imposed separately from finite ReLU
representation.

\medskip
\noindent All six residual bounds and $\Dbar$ are deterministic almost-sure
envelopes.  Write $Z=(r,\stilde',d)$ for a selected record's outcome and
$M_k$ for its standard-Borel label-construction metadata, so that
$Y_k=\ell_k(S,\mathbf a,M_k,Z)$ for a bounded measurable label map fixed by
$\mathcal J_k$.  For stored labels, $M_k$ includes the target-copy age and the
network parameters and, for \textsc{pred}, predictor parameters used at writing;
these are accounting
variables and need not all be retained in the buffer.  No independence
between $M_k$ and $Z$ is assumed.  For sample-time labels whose network and
predictor are fixed by $\mathcal J_k$, $M_k$ may be constant.
Let $G_k$ be a chosen version of $\E[Y_k\mid S,\mathcal J_k]$.  Define its
ideal counterpart by retaining $(S,\mathbf a,M_k)$ and redrawing only the
outcome:
\begin{equation}
\begin{aligned}
\mathcal L(Z^\circ\mid S,\mathbf a,M_k,\mathcal J_k)
&=\Pjoint(\cdot\mid S,\mathbf a),\\
G_k^{\mathrm{ideal}}(S)
&:=\E[\ell_k(S,\mathbf a,M_k,Z^\circ)\mid S,\mathcal J_k].
\end{aligned}
\label{eq:ideal_metadata_redraw}
\end{equation}
The laws
$\rho_{k,S},\beta_k^{\mathrm{rep}}$ and both responses are
$\mathcal J_k$-measurable.  The fit link is controlled in conditional mean;
the other five links are controlled almost surely by their displayed
deterministic envelopes.  Thus
\begin{align}
\E\bigl[\|V_{k+1}-G_k\|_{2,\rho_{k,S}}\bigm|\mathcal J_k\bigr]
&\;\le\;\eps_{\mathrm{fit},k}
&&\text{(\emph{fit}: the solver against its own response)},
\label{eq:fit_residual}\\
\bigl\|G_k-G^{\mathrm{ideal}}_k\bigr\|_{2,\rho_{k,S}}
&\;\le\;\eps_{\mathrm{ker},k}
&&\text{(\emph{kernel realization}: reuse of realized outcomes)},
\label{eq:ker_residual}\\
\bigl\|G^{\mathrm{ideal}}_k-\Tcal^{\beta^{\mathrm{rep}}_k}V_k\bigr\|_{2,\rho_{k,S}}
&\;\le\;\eps_{\mathrm{tgt},k}
&&\text{(\emph{target construction}: the switches of slot~(S6))},
\label{eq:tgt_residual}
\end{align}
Under population-law accounting, suppose that a selected record satisfies the
following conditional kernel-retention hypothesis, including the
label-construction metadata: for every bounded
measurable $f$,
\begin{equation}
\E\bigl[f(r,\stilde',d)\mid S,\mathbf a,M_k,\mathcal J_k\bigr]
=\int f\,\mathrm d\Pjoint(\cdot\mid S,\mathbf a)
\quad\text{a.s.}
\label{eq:conditional_kernel_retention}
\end{equation}
Conditioning first on $(S,\mathbf a,M_k,\mathcal J_k)$ and then averaging
shows that $G_k=G_k^{\mathrm{ideal}}$, so $\eps_{\mathrm{ker},k}=0$.
For the current observed-successor label, the tower property also gives the
replayed operator,
\begin{equation}
\E\bigl[\,r+\geff(1-d)V_k(\stilde')\,\bigm|\,S=\stilde,\mathcal J_k\,\bigr]
=(\Tcal^{\beta^{\mathrm{rep}}_k}V_k)(\stilde),
\label{eq:population_target}
\end{equation}
where $\mathcal J_k$ is pre-block information.  For sample-time labels fixed
by $\mathcal J_k$, conditioning only on $(S,\mathbf a,\mathcal J_k)$ in
\eqref{eq:conditional_kernel_retention} suffices.  For stored labels it need
not suffice: record age or writing-time parameters can remain correlated
with the selected outcome.  The full condition holds for a fresh true-kernel
draw made after the metadata are fixed.  A sample-splitting or
outcome-independent retention argument must establish this conditional law
given both the chosen history and the metadata; it must be verified for
adaptive FIFO replay.  Under
realized-buffer accounting, $\mathcal J_k=\mathcal B_k$ and
$\eps_{\mathrm{ker},k}$ measures the empirical response's deviation from its
fresh-outcome counterpart, ready for a process-specific replay bound.  For $\Aiid$,
$\eps_{\mathrm{ker},k}=\eps_{\mathrm{tgt},k}=0$.

Lemma~\ref{lem:tgt_residual} gives the target-switch accounting:
$(\textsc{sample},\textsc{obs})$ has zero target residual, while
\textsc{store} charges conditional target-network staleness and
\textsc{pred} charges current continuation-score error (plus its stated
terminal-mask correction).  Their combination also charges any change
between the writing-time and current predictors.
An executed-transition target therefore pays behavior, online-action, and
exploration links.  A model-built target
$\max_a\{\widehat R(\stilde,a)+\geff V_k(\Phat(\stilde,a))\}$ removes those
links but inserts score error into every target; neither design is uniformly
preferred.  The replay residual $\epsbuf$ is a same-state operator distance.

\begin{lemma}[Target-switch envelope\obj{\Rabs}]\label{lem:tgt_residual}
Let $A\in\{0,\ldots,k\}$ be a replayed record's age in target copies and
$\bar\delta_k(A):=\|V_k-V_{k-A}\|_\infty$.  Write $\Phat_k$ for the current
predictor fixed by $\mathcal J_k$ (the block index previously suppressed)
and $\Phat_{\mathrm{write}}$ for the predictor used to construct a stored
prediction label, identified by $M_k$.  Its version need not be determined by
$A$.  For this cell define
\begin{equation}
C_{\mathrm{pred},k}(\stilde):=\geff\,\E\!\left[
\left|V_k(\Phat_{\mathrm{write}}(S,\mathbf a))
-V_k(\Phat_k(S,\mathbf a))\right|
\bigm| S=\stilde,\mathcal J_k\right].
\label{eq:tgt_predictor_staleness}
\end{equation}
This quantity vanishes for a fixed predictor; set it to zero in the other
cells.  Let
$\eta_k^{\mathrm{cont}}:=\sup_{V\in\mathcal F,\stilde,a}
\geff|V(\Phat_k(\stilde,a))-\int V\,d\Pkernel(\cdot\mid\stilde,a)|$.
When the point score is installed in slot~\textnormal{(S2)}, symmetry gives
$\eta_k^{\mathrm{cont}}\le\etascorek$; under other score access it remains a
separate target-construction quantity.  If the terminal event is determined
by $(\stilde,a)$, then
the four cells of slot~\textnormal{(S6)} satisfy
\begin{equation}
\begin{aligned}
\|G_k^{\mathrm{ideal}}-\Tcal^{\beta_k^{\mathrm{rep}}}V_k\|_{2,\rho_{k,S}}
&\le\Indic\{\textsc{from}=\textsc{pred}\}\eta_k^{\mathrm{cont}}\\[-2pt]
&\quad+\Indic\{\textsc{when}=\textsc{store}\}\geff
\|\E[\bar\delta_k(A)\mid S,\mathcal J_k]\|_{2,\rho_{k,S}}\\
&\quad+\Indic\{(\textsc{when},\textsc{from})=(\textsc{store},\textsc{pred})\}
\|C_{\mathrm{pred},k}\|_{2,\rho_{k,S}}.
\end{aligned}
\label{eq:tgt_residual_bound}
\end{equation}
In the general termination case, the predicted-successor cells add
$\|C_{\mathrm{mask},k}\|_{2,\rho_{k,S}}$, where
\[
C_{\mathrm{mask},k}(\stilde):=\geff\int
\Pr(d=1\mid\stilde,a)|V_k(\Phat_k(\stilde,a))|
\,\beta_k^{\mathrm{rep}}(da\mid\stilde).
\]
Thus the replay-action average is taken before the state norm, and both
staleness quantities remain conditional on $S$ and $\mathcal J_k$.
Any deterministic almost-sure majorant of the resulting right-hand side is
an admissible $\eps_{\mathrm{tgt},k}$.
\end{lemma}

\noindent\emph{Proof:} See Appendix~\ref{app:proofs}.

Finally, $\dpath$, $\dsnap$, $\Dexp$, and realized fit errors are random.
The symbols $\dnet$, $\etascorek$, $\etascore$, and $\etascoreK$ are
deterministic, as are
$\eps_{\mathrm{fit},k}$, $\eps_{\mathrm{ker},k}$,
$\eps_{\mathrm{tgt},k}$, $\epsbuf$, $\epsactr{r}$ for $1\le r\le2$,
$\Dbar$, and $\Lambda_k$.
Random quantities enter later bounds only inside expectations or through such
envelopes.

\subsection{Concentrability and finite-horizon propagation}\label{subsec:propagation}

\begin{assumption}[Finite-horizon $L^s$ concentrability of population replay\obj{\Rabs}]\label{ass:concentrability}
Fix $s\in[2,\infty]$ and let $p=s/(s-1)$, with $p=1$ when $s=\infty$.
For every $1\le m\le H$ and every sequence $\pi_1,\ldots,\pi_m$ of
\emph{randomized} Markov policies (measurable kernels from $\Stildeo$ into the
simplex over $\Aspace$, including deterministic selectors), assume
$\mu_S^\circ\Pkernel^{\pi_1}_\circ\cdots\Pkernel^{\pi_m}_\circ\ll\rho_{k,S}$ and
set
\begin{equation}
d_{s,k}(m):=\sup_{\pi_1,\ldots,\pi_m}
\Bigl\|\tfrac{d(\mu_S^\circ\Pkernel^{\pi_1}_\circ\cdots\Pkernel^{\pi_m}_\circ)}{d\rho_{k,S}}\Bigr\|_{s,\rho_{k,S}} .
\label{eq:concentrability_ck}
\end{equation}
We assume the pathwise form \textnormal{(C-strong)}: a deterministic
$d_s(m)<\infty$ with $d_{s,k}(m)\le d_s(m)$ almost surely for all $k$; by
clock nilpotence take $d_s(H)=0$.  The supremum in
\eqref{eq:concentrability_ck} is pointwise over all policy sequences: after a
history is fixed it therefore includes policies selected from that history,
although every factor remains a Markov kernel in the current state.  Define
\begin{equation}
\phi_{s,K}^{(H)}:=\sum_{m=1}^{H-1}\min\{K,m\}\geff^m d_s(m),
\qquad
\phi_s^{(H)}:=\phi_{s,H-1}^{(H)}=\sum_{m=1}^{H-1}m\geff^m d_s(m).
\label{eq:concentrability}
\end{equation}
At $s=2$ write $c_{2,k}(m):=d_{2,k}(m)^2$, $c_2(m):=d_2(m)^2$, and
$\phi^{(H)}_{\mu_S,\rho}:=\phi_2^{(H)}$; at $s=\infty$ write
$c_\infty(m):=d_\infty(m)$.  H\"older pairs the density with an
$L^p(\rho_{k,S})$ residual.  Existing $L^2$ envelopes for the nonaction links
remain valid because $p\le2$ and $\rho_{k,S}$ is a probability law; the action
link may use its sharper $p$-specific envelope.
\end{assumption}

\begin{assumption}[Generative reset access\obj{\Aiid,\Alev}]\label{ass:generative}
For $\Aiid$, the simulator can be reset to a state drawn from the designer's
single reset law $\varsigma$ on $\Stildeo$, after which one behavior-policy
action is executed and one true-kernel transition is observed.  For $\Alev$,
it can be reset independently to each $\mathcal J_k$-measurable
clock-slice law $\rho^{(h)}_{k,S}$ specified in
Theorem~\ref{thm:level_indexed}.  These are distinct access assumptions.
If slice laws are obtained by rejection from $\varsigma$, rejected draws must
be added to the sample count; Theorem~\ref{thm:level_indexed} counts direct
slice-reset draws only.
\end{assumption}

\begin{definition}[Top-slice evaluation and horizon-indexed families\obj{\Gen}]
\label{def:horizon_family}
For a fixed horizon $H$, \emph{top-slice evaluation} means
\begin{equation}
\mu_S(\mathcal S\times\{H\})=1.
\label{eq:top_slice_evaluation}
\end{equation}
A \emph{horizon-indexed family} $\{\mathcal I_H\}_{H\ge2}$ consists of one
instance of the augmented model and its evaluation/design objects for each
$H$, with discount $\geff_H$, iteration index $K_H$, evaluation law
$\mu_{S,H}$, replay laws $\rho_{k,S,H}$, survival masses $q_{H,m}$, and
concentrability envelopes $d_{s,H}(m)$.  Within a fixed member, the $H$
subscripts are suppressed.  Every asymptotic symbol in $H$ refers to this
family, and every constant declared uniform is independent of $H$.  The family
is \emph{near-unit-discount} when
$(1-\geff_H)H\to\varkappa\in[0,\infty)$, and is
\emph{fixed-discount} when $\geff_H\equiv\geff$ for one
$\geff\in(0,1)$.
\end{definition}

\begin{theorem}[Finite-$K$ $L^s/L^p$ clock tradeoff\obj{\Rabs}]\label{prop:conc_budget}
Assume \textnormal{(C-strong)}, top-slice evaluation in the sense of
Definition~\ref{def:horizon_family}, $H\ge2$, and $K\ge1$.
For
$\mu_m^{\pi_{1:m}}:=\mu_S^\circ\Pkernel^{\pi_1}_\circ\cdots
\Pkernel^{\pi_m}_\circ$, set
$q_m:=\sup_{\pi_{1:m}}\mu_m^{\pi_{1:m}}(\Stildeo)$,
$a_{K,m}:=\min\{K,m\}\geff^m$, and
$\theta:=p/(p+1)$ (equivalently $s/(2s-1)$ for finite $s$, and $1/2$ at
$s=\infty$).  Every shared design law satisfies
\begin{equation}
 \sum_{m=1}^{H-1}\left(\frac{q_m}{d_s(m)}\right)^p\le1,
 \qquad
 \phi_{s,K}^{(H)}\ge
 \left\{\sum_{m=1}^{H-1}(a_{K,m}q_m)^\theta\right\}^{1/\theta}.
\label{eq:conc_budget}
\end{equation}
Zero-survival coordinates are omitted.  The second bound is an exact
relaxation: on the deterministic one-state-per-level chain with $q_m=1$, the
$K$-specific clock masses
\begin{equation}
r_{H-m}=\frac{a_{K,m}^\theta}{\sum_{j=1}^{H-1}a_{K,j}^\theta}
\label{eq:clock_attainment}
\end{equation}
give $d_s(m)=r_{H-m}^{-1/p}$ and equality in~\eqref{eq:conc_budget}.
\end{theorem}

\begin{proof}
The depth-$m$ law is supported on $L_m=\mathcal S\times\{H-m\}$.  If
$r_{H-m}=\rho_{k,S}(L_m)$ and $f_m$ is its density, slice-supported H\"older
gives $q_m\le\|f_m\|_s r_{H-m}^{1/p}\le d_s(m)r_{H-m}^{1/p}$; summing over the
disjoint slices proves the first inequality.  Applying H\"older with exponents
$1/\theta$ and $p+1$ to
$(a_{K,m}d_s(m))^\theta(q_m/d_s(m))^\theta$ proves the second.  On the stated
chain the depth law is a point mass on its unique slice, so substitution of
\eqref{eq:clock_attainment} gives equality.
\end{proof}

\begin{corollary}[Sharp horizon regimes for the shared-law clock coefficient\obj{\Rabs}]
\label{prop:clock_scaling}\label{cor:clock_regimes}
Apply Theorem~\ref{prop:conc_budget} to a top-slice horizon-indexed family and
assume the uniform survival floor
$\inf_{H\ge2}\inf_{1\le m<H}q_{H,m}\ge q>0$.
\begin{enumerate}[label=\textnormal{(\alph*)},leftmargin=2.2em,itemsep=2pt,topsep=3pt]
\item In the near-unit-discount regime, if
$K_H/H\to\tau\in(0,\infty]$, every shared design law obeys
\begin{equation}
\phi_{s,K_H}^{(H)}\ge qA_{s,\tau}(\varkappa)H^{3-1/s}(1+o(1)),\quad
A_{s,\tau}(\varkappa):=
\left\{\int_0^1\min\{\tau,x\}^\theta e^{-\varkappa\theta x}\,dx\right\}^{1/\theta},
\label{eq:clock_scaling}
\end{equation}
where $\min\{\infty,x\}=x$.
\item In the same near-unit-discount regime, if instead
$1\le K_H=o(H)$, every shared law obeys the order
$\Omega(K_H\,H^{2-1/s})$.
\item In the fixed-discount regime, the sharp shared-law relaxation is
$\Theta(1)$ for every sequence $K_H\ge1$, whereas the uniform $H$-slice law
on the one-state-per-level family has order $\Theta(H^{1-1/s})$.
\end{enumerate}
The deterministic one-state-per-level family has $q_{H,m}=1$, and its
$K_H$-specific law~\eqref{eq:clock_attainment} attains the orders in
\textnormal{(a)}--\textnormal{(c)}.  Thus each order is best possible for the
shared-law clock-coefficient problem of Theorem~\ref{prop:conc_budget}.
\end{corollary}

\begin{proof}
For~\textnormal{(a)}, divide
$\sum_{m<H}a_{K_H,m}^\theta$ by $H^{1+\theta}$ and use the Riemann sum in
\eqref{eq:clock_scaling}; $(1+\theta)/\theta=3-1/s$.  For $K_H=o(H)$, split at
$m=K_H$: the tail is $\Theta(K_H^\theta H)$ and dominates the
$O(K_H^{1+\theta})$ initial part, proving~\textnormal{(b)}.  Theorem~\ref{prop:conc_budget}
gives both lower bounds and its one-state family gives equality.  For
\textnormal{(c)}, $a_{K_H,m}\le m\geff^m$ makes
$\sum_m a_{K_H,m}^\theta$ uniformly finite and bounded away from zero.  The
$K_H$-specific masses attain this constant order, while the uniform clock law
has $d_s(m)=H^{1/p}$ and a weight sum uniformly bounded above and away from
zero, giving $H^{1/p}=H^{1-1/s}$.
\end{proof}

\begin{corollary}[Endpoint attainment and direct-level comparison\obj{\Rabs}]
\label{cor:clock_optimal_attainment}
For the deterministic one-state-per-level family in the near-unit-discount
regime with $K_H\ge H-1$, the attaining masses
in~\eqref{eq:clock_attainment} are proportional to
$(m\geff_H^m)^{2/3}$ at $s=2$ and to $(m\geff_H^m)^{1/2}$ at $s=\infty$.
These attain the optimal shared-law propagation-coefficient orders
$H^{5/2}$ and $H^3$, respectively.  The uniform $H$-slice law on the same
family has the same orders, with different constants.  Under the distinct
direct-level-reset access model, uniformly bounded $L^2$ level coefficients
give propagation coefficient $O(H^2)$.
\end{corollary}

Theorem~\ref{thm:level_indexed} formalizes the distinct sampling scheme that
pairs each propagated depth with its own clock-slice law.

\begin{lemma}[Comparison kernel for the absolute Bellman difference\obj{\Gen}]\label{lem:comparison_kernel}
Let $V,W\in\mathcal V$ and $\Delta_{V,W}:=V-W$.  There exists a measurable
substochastic Markov kernel $\Pkernel^{V,W}_\circ$ on $\Stildeo$ such that,
pointwise on $\Stildeo$,
\begin{equation*}
|\Tcal V-\Tcal W|(\stilde)
\le\geff\bigl(\Pkernel^{V,W}_\circ|\Delta_{V,W}|\bigr)(\stilde).
\end{equation*}
Moreover, $\Pkernel^{V,W}_\circ$ may be chosen as the kernel
$\Pkernel^\pi_\circ$ of a measurable deterministic policy $\pi$ that selects
pointwise between the greedy selectors $a^\star(\cdot;V)$ and
$a^\star(\cdot;W)$.
\end{lemma}

\begin{proof}
Put $a_V:=a^\star(\cdot;V)$ and $a_W:=a^\star(\cdot;W)$, and write
$P_a:=\Pkernel_\circ(\cdot\mid\stilde,a)$.  Optimality gives
\begin{align*}
(\Tcal V-\Tcal W)(\stilde)
&\le\geff\!\int|\Delta_{V,W}|\,dP_{a_V(\stilde)},\\[-2pt]
(\Tcal W-\Tcal V)(\stilde)
&\le\geff\!\int|\Delta_{V,W}|\,dP_{a_W(\stilde)} .
\end{align*}
Choose at each state whichever nonnegative integral is larger.  Kernel
integration and the two Borel selectors make this choice measurable.  The
result is a deterministic policy kernel that dominates both one-sided bounds.
Thus~\eqref{eq:concentrability_ck} covers every product formed below, although
the policies in a product need not coincide.
\end{proof}

Each transition between nonterminal states reduces the remaining horizon,
so a product of $H$ nonterminal transition kernels is zero.  This truncates error propagation
and gives the finite update window in the following lemma.

\begin{lemma}[Finite-horizon residual propagation\obj{\Rabs}]\label{lem:varying_propagation}
Assume~\ref{ass:markov_state} and~\ref{ass:concentrability}.  Let
$V_0,\ldots,V_K\in\mathcal V$ satisfy $\|V_k\|_\infty\le\Vmax$, set
$e_k:=V_{k+1}-\Tcal V_k$, let $K\ge1$, and let $\pi_K$ be greedy with respect
to $V_K$.  Put $p_h:=\mu_S^\circ(\mathcal S\times\{h\})$ and choose
deterministic $D_{0,h}$ such that
\[
\sup_{s\in\mathcal S}|\Vstar(s,h)-V_0(s,h)|\le D_{0,h}
\quad\text{almost surely}.
\]
The global bound permits $D_{0,h}=V_{\max}^{(h)}+\Vmax$;
for deterministic $V_0$, one may use its actual slice error.  Define
the deterministic initialization envelope
\begin{equation}
\mathcal B_K:=2\sum_{h=1}^{H}p_h\sum_{m=K+1}^{h-1}
\geff^mD_{0,h-m}.
\label{eq:finite_K_boundary_phi}
\end{equation}
With deterministic weights
\begin{equation}
w^{(H)}_{K,k}:=2\!\!\sum_{\substack{\ell\ge0:\\1\le\ell+K-k<H}}\!\!
\geff^{\ell+K-k}d_s(\ell+K-k),\qquad 0\le k<K,
\label{eq:propagation_weights}
\end{equation}
the following bound holds pathwise:
\begin{equation}
\|\Vstar-V^{\pi_K}\|_{1,\mu_S}
\le\mathcal B_K
+\sum_{k=\max\{0,K-H+1\}}^{K-1}w^{(H)}_{K,k}\|e_k\|_{p,\rho_{k,S}},
\qquad
\sum_{k=0}^{K-1}w^{(H)}_{K,k}=2\phi_{s,K}^{(H)}.
\label{eq:varying_propagation}
\end{equation}
For top-slice evaluation,
\[
\mathcal B_K=2\sum_{m=K+1}^{H-1}\geff^mD_{0,H-m}.
\]
If $V_0\in\mathcal V^{\mathrm{clip}}$ almost surely, choosing
$D_{0,h}=2V_{\max}^{(h)}$ gives
\[
\mathcal B_K\le4\sum_{m=K+1}^{H-1}\geff^mV_{\max}^{(H-m)}.
\]
In all cases, $\mathcal B_K=0$ for $K\ge H-1$.
\end{lemma}

\begin{proof}[Proof roadmap]
Lemma~\ref{lem:comparison_kernel} gives the one-step comparison recursion.
Unrolling it and applying the nonnegative loss resolvent produces two
policy-kernel branches of total depth $m=\ell+K-k$.  Clock nilpotence removes
$m\ge H$; H\"older and Assumption~\ref{ass:concentrability} give the stated
weights, while counting the $\min\{K,m\}$ admissible indices gives their sum.
Appendix~\ref{app:proofs} supplies the pathwise occupancy construction and the
deterministic initialization-envelope bound.
\end{proof}

\subsection{Residual definitions and composition}\label{subsec:composition}

\begin{assumption}[Behavior policies, the replay residual, and the online-action residual\obj{\Rabs}]\label{ass:residuals}
\emph{(Behavior.)}  The data of block $k$ is collected by $\eps_k$-greedy
policies whose greedy branch is the one implemented in
Algorithm~\ref{alg:hmpdrl_training}: it maximizes the point-model score
$\widehat Q_V$ of~\eqref{eq:greedy_action} at the acting network.  The oracle
consistency results use the separate true conditional-expectation score
$\Qoracle_V$.  The action analysis applies to either
choice through its generic perturbation score $q_k$.  For $\Aiid$ there is one
such policy, the snapshot law $\beta_k$
built from $W_k$.  For $\Arun$ there is one per round, and the period-level
object $\pibar$ is obtained by \emph{disintegration} of the joint collection
measure,
\begin{equation}
\mathsf M_k(d\stilde,da)=\sum_j\omega_{k,j}\,\mu_{k,j}(d\stilde)\,
\pi^{\mathrm{on}}_{k,j}(da\mid\stilde)
=\mathsf M_{k,S}(d\stilde)\,\pibar(da\mid\stilde),
\label{eq:disintegration}
\end{equation}
$\omega_{k,j}$ being the sample share and $\mu_{k,j}$ the state law of round
$j$.  The weights of $\pibar$ are \emph{state-dependent}: a time average of the
$\pi^{\mathrm{on}}_{k,j}$ is in general not the conditional action law of the
collected data.  Fix the density-weighted version from
Lemma~\ref{lem:measurable_framework} $\mathsf M_{k,S}$-a.e. and its declared
measurable mixture extension on the $\mathsf M_{k,S}$-null complement; use the
same convention for $\beta^{\mathrm{rep}}_k$ below.  If $\rho_{k,S}$ charges
that null complement, the disintegration identity alone imposes no relation
there.  Therefore, for $\Arun$, require
$\rho_{k,S}\ll\mathsf M_{k,S}$ whenever $\pibar$ is compared under
$\rho_{k,S}$; the fixed extension then serves only measurability and cannot
change any displayed residual norm.  Write
$\pi^{\mathrm{on}}_k$ for $\beta_k$ under $\Aiid$ and for $\pibar$ under
$\Arun$, and $\pi^{\mathrm{tgt}}_k$ for the $\eps_k$-greedy policy whose greedy
branch maximizes the true $Q_{V_k}$; all share $\eps_k$ and $\xi_k$.

\emph{(The two residuals.)}  Let $\beta^{\mathrm{rep}}_k$ be the conditional
action law of the distribution actually replayed.  Every operator written
$\Tcal^{\beta^{\mathrm{rep}}_k}$ refers to this law, and
$\beta^{\mathrm{rep}}_k=\beta_k$ for $\Aiid$, whose outcomes are fresh and used
once.  For each block $k$, let the deterministic block-indexed envelopes
$\epsbuf$ and $\epsactr{r}$, $1\le r\le2$, satisfy
\begin{align}
\bigl\|\Tcal^{\beta^{\mathrm{rep}}_k}V_k-\Tcal^{\pi^{\mathrm{on}}_k}V_k\bigr\|_{2,\rho_{k,S}}
&\le\epsbuf,
\label{eq:slow_buffer}\\
\bigl\|\Tcal^{\pi^{\mathrm{on}}_k}V_k-\Tcal^{\pi^{\mathrm{tgt}}_k}V_k\bigr\|_{r,\rho_{k,S}}
&\le\epsactr{r},\qquad 1\le r\le2,
\label{eq:online_action_residual}
\end{align}
together with the fit bound~\eqref{eq:fit_residual}.  Here
$\epsact=\epsactr{2}$; the norm-indexed envelopes need not be equal.  These are
\emph{operator} distances, which is what Lemma~\ref{lem:composition} consumes.
\end{assumption}

\begin{remark}[A total-variation envelope for replay shift\obj{\Rabs}]\label{rem:replay_tv}
Write
\begin{equation}
\begin{aligned}
\Vrho&:=\bigl\|\Vmax^{(h(\cdot))}\bigr\|_{2,\rho_{k,S}}
=\Bigl(\textstyle\sum_{h=1}^{H}p_{k,h}\bigl(\Vmax^{(h)}\bigr)^2\Bigr)^{1/2},\\
&\hspace{4em}\Vrho\le\Vrhobar\le\Vmax\quad\text{a.s. for every $k$},
\end{aligned}
\label{eq:Vrho_def}
\end{equation}
where $p_{k,h}:=\rho_{k,S}(\mathcal S\times\{h\})$ and \(\Vrhobar\) is a
deterministic uniform majorant (safely, \(\Vmax\)); hence
\(\Dbar:=2\Vrhobar\) is admissible in~\eqref{eq:envelopes}.  Use the convention
$d_{\mathrm{TV}}(P,Q):=\sup_A|P(A)-Q(A)|$.  Since
$|Q_{V_k}(s,h,a)|\le\Vmax^{(h)}$, the replay link itself satisfies
\begin{equation}
\|\Tcal^{\beta_k^{\mathrm{rep}}}V_k-\Tcal^{\pi_k^{\mathrm{on}}}V_k\|_{2,\rho_{k,S}}
\le2\|\Vmax^{(h(\cdot))}d_{\mathrm{TV}}(\beta_k^{\mathrm{rep}},
\pi_k^{\mathrm{on}})\|_{2,\rho_{k,S}}
\le2\Vrho\operatorname*{ess\,sup}_{\rho_{k,S}}d_{\mathrm{TV}}(\beta_k^{\mathrm{rep}},\pi_k^{\mathrm{on}}),
\label{eq:replay_tv_envelope}
\end{equation}
Any deterministic majorant is admissible for $\epsbuf$; $2\Vmax$ is safe.
A quantitative FIFO decay rate follows from stabilization, age/mixing, or
buffer-scaling control.
\end{remark}

\begin{lemma}[Executed-transition response versus Bellman response\obj{\Rabs}]\label{lem:composition}
For $r\in[1,2]$, set
$\widetilde\eps_{k,r}:=\|V_{k+1}-\Tcal V_k\|_{r,\rho_{k,S}}$.  Under
Assumptions~\ref{ass:markov_state} and~\ref{ass:residuals}, for $0\le k<K$,
\begin{equation}
\E[\widetilde\eps_{k,r}\mid\mathcal J_k]\le
\eps_{\mathrm{fit},k}+\eps_{\mathrm{ker},k}+\eps_{\mathrm{tgt},k}
+\epsbuf+\epsactr{r}+\eps_k\Dbar.
\label{eq:composition_residual}
\end{equation}
\end{lemma}

\begin{proof}
Apply the triangle inequality along
\[
V_{k+1}\to G_k\to G^{\mathrm{ideal}}_k\to
\Tcal^{\beta^{\mathrm{rep}}_k}V_k\to\Tcal^{\pi^{\mathrm{on}}_k}V_k\to
\Tcal^{\pi^{\mathrm{tgt}}_k}V_k\to\Tcal V_k .
\]
Every link except the online-action link is first bounded by its named $L^2$
envelope and hence by the same envelope in $L^r$, since $r\le2$ and
$\rho_{k,S}$ is a probability law.  The online-action link uses
\eqref{eq:online_action_residual}.  The last link is exactly
$\eps_k\dexp$ by~\eqref{eq:dispersion_identity}, where
\begin{equation}
\begin{aligned}
\dexp(\stilde)&:=\Delta^{\mathrm{exp}}[V_k,\xi_k](\stilde)\\
&=\max_aQ_{V_k}(\stilde,a)-\!\int\! Q_{V_k}(\stilde,a)\,\xi_k(da\mid\stilde),
\qquad
\Dexp:=\bigl\|\dexp\bigr\|_{2,\rho_{k,S}} .
\end{aligned}
\label{eq:dispersion_def}
\end{equation}
$|Q_{V_k}(s,h,a)|\le\Vmax^{(h)}$ gives
\begin{equation}
\Dexp\;\le\;2\Vrho\;\le\;2\Vrhobar\;\le\;2\Vmax
\label{eq:dispersion_chain}
\end{equation}
and $\|\dexp\|_{r,\rho_{k,S}}\le\Dexp\le\Dbar$
by~\eqref{eq:envelopes}.  Taking conditional expectations proves
\eqref{eq:composition_residual}.
\end{proof}

\begin{definition}[Visible-target aliasing\obj{\Rabs}]\label{def:visible_aliasing}
Fix one realization of $\mathcal J_k$.  Let
$O:\Stildeo\to\mathcal O$ be the measurable representation visible to the
fitted function, and let $\Pi_{O,k}$ be conditional expectation given
$\sigma(O)$ under $\rho_{k,S}$.  Set
\[
a_{\mathrm{alias},k}
:=\bigl\|\Pi_{O,k}G_k-G_k\bigr\|_{2,\rho_{k,S}} .
\]
This is the exact distance from the block response in~\eqref{eq:fit_residual}
to the closed subspace of $\sigma(O)$-measurable $L^2(\rho_{k,S})$ functions.
It is $\mathcal J_k$-measurable and is not assumed zero.  When a deterministic envelope is needed, write
$a_{\mathrm{alias},k}\le\eps_{\mathrm{alias},k}$ almost surely.
\end{definition}

\begin{lemma}[Aliasing is a component of the fit residual\obj{\Rabs}]\label{lem:alias}
Suppose the regression class consists of $\sigma(O)$-measurable functions and
the solver satisfies the regression bound relative to the \emph{best visible}
target,
$\E[\|V_{k+1}-\Pi_{O,k}G_k\|_{2,\rho_{k,S}}\mid\mathcal J_k]
\le\eps^{\mathrm{vis}}_{\mathrm{fit},k}$.  Then
$\|V-G_k\|_{2,\rho_{k,S}}\ge a_{\mathrm{alias},k}$ for every visible $V$, and
\[
\E[\|V_{k+1}-G_k\|_{2,\rho_{k,S}}\mid\mathcal J_k]
\le\eps^{\mathrm{vis}}_{\mathrm{fit},k}+\eps_{\mathrm{alias},k}.
\]
Thus $\eps_{\mathrm{fit},k}:=\eps^{\mathrm{vis}}_{\mathrm{fit},k}
+\eps_{\mathrm{alias},k}$; aliasing is not charged again.
\end{lemma}

\begin{proof}
$\Pi_{O,k}$ is the orthogonal projection onto the visible subspace, so
$\|V-G_k\|^2_{2,\rho_{k,S}}
=\|V-\Pi_{O,k}G_k\|^2_{2,\rho_{k,S}}+a_{\mathrm{alias},k}^2$,
and the upper bound follows by the triangle inequality.
\end{proof}

\begin{remark}[What removes the floor: target sufficiency of the representation\obj{\Rabs}]\label{rem:alias_remedy}
For the fixed history, $a_{\mathrm{alias},k}=0$ exactly when
$G_k=g_k\circ O$ $\rho_{k,S}$-almost surely for some measurable $g_k$: the
observation determines the block target mean.  Hidden reward, continuation, or
clock variables can therefore leave a fit floor; augmenting the representation
or passing to a belief state can remove it.
\end{remark}

\begin{proposition}[Markov-sufficient and compressed-observation routes\obj{\Gen,\Rabs}]
\label{prop:representation_routes}
There are two compatible routes from an implemented input to the abstract
state.
\begin{enumerate}[label=\textnormal{(\roman*)},leftmargin=2.2em,itemsep=2pt,topsep=3pt]
\item If the implemented input $O_t$, including its clock, satisfies
\eqref{eq:markov_sufficiency} with $\stilde_t=O_t$, then it may be used as the
analysis state.  In particular, for a $(\textsc{sample},\textsc{obs})$ target,
an $O$-measurable acting policy, and an $O$-measurable frozen value, the
population response is $O$-measurable and $a_{\mathrm{alias},k}=0$.
\item If $O=\omega(\stilde)$ is a measurable compression of a state satisfying
Assumption~\ref{ass:markov_state}, then every observation-only value and policy
is still a measurable value and Markov policy on $\Stilde$.  The residual
decomposition therefore applies on the Markov analysis state with
$\eps_{\mathrm{fit},k}=\eps^{\mathrm{vis}}_{\mathrm{fit},k}
+\eps_{\mathrm{alias},k}$ as in Lemma~\ref{lem:alias}.  An observation-only
reward/transition score is lifted in the same way, and its discrepancy from
$Q_V$ is charged by $\etascorek$; closure and concentrability are likewise
checked on $\Stilde$.
\end{enumerate}
A Markov-sufficient representation gives the direct convergence route, while
a compressed observation gives a quantitative performance route whose
representation and score floors remain visible in the policy-loss bound.
\end{proposition}

\begin{proof}
Part~\textnormal{(i)} is Assumption~\ref{ass:markov_state} on the represented
state space.  Conditional expectation of the observed-successor label through
its joint kernel is then a measurable function of $O_t$, so its projection
onto $\sigma(O)$ is itself.  For~\textnormal{(ii)}, composition with $\omega$
lifts every implemented value, score, and policy to a measurable object on
$\Stilde$; Lemma~\ref{lem:alias} supplies the fit link, and the remaining links
are exactly those in Assumption~\ref{ass:residuals}.
\end{proof}

\subsection{Finite-horizon policy-loss bound}\label{subsec:end_to_end}

\begin{theorem}[Finite-$K$ policy-loss bound under six residuals\obj{\Rabs}]\label{thm:hmpdrl_end_to_end}
Assume Assumptions~\ref{ass:markov_state}, \ref{ass:concentrability}
\textnormal{(C-strong)}, and~\ref{ass:residuals}, with
$K\ge1$, $V_0,\ldots,V_K\in\mathcal V^{\mathrm{clip}}$, and $\Vstar$ from
Lemma~\ref{lem:exist_unique}.  Let $\pi_K$ be true-$Q_{V_K}$ greedy after $K$
target copies; it uses the true kernel and is not directly deployable
(\S\ref{subsec:deployment}).  Set
\begin{equation}
e_{k,p}^{\mathrm{Bell}}:=\eps_{\mathrm{fit},k}+\eps_{\mathrm{ker},k}+\eps_{\mathrm{tgt},k}
+\epsbuf+\epsactp+\eps_k\Dbar,
\qquad \overline e_{K,H,p}^{\mathrm{Bell}}:=
\max_{\max\{0,K-H+1\}\le k<K}e_{k,p}^{\mathrm{Bell}}.
\label{eq:rk_def}
\end{equation}
Then
\begin{equation}
\E\!\left[\|\Vstar-V^{\pi_K}\|_{1,\mu_S}\right]
\le \mathcal B_K+\sum_{k=0}^{K-1}w^{(H)}_{K,k}e_{k,p}^{\mathrm{Bell}}
\le \mathcal B_K+2\phi_{s,K}^{(H)}\overline e_{K,H,p}^{\mathrm{Bell}}.
\label{eq:hmpdrl_end_to_end}
\end{equation}
\end{theorem}

\begin{proof}
Lemma~\ref{lem:composition} at $r=p$ and the tower property give
$\E\widetilde\eps_{k,p}
=\E[\,\E[\widetilde\eps_{k,p}\mid\mathcal J_k]\,]\le e_{k,p}^{\mathrm{Bell}}$.
The six summands of~\eqref{eq:rk_def} are deterministic.  For the exploration
link the summand is the majorant $\Dbar$ of~\eqref{eq:envelopes}, not the
random dispersion $\Dexp$.  Apply
Lemma~\ref{lem:varying_propagation} to $e_k=V_{k+1}-\Tcal V_k$ and take
expectations termwise, since the weights are deterministic.  The second
inequality uses the active-window maximum $\overline e_{K,H,p}^{\mathrm{Bell}}$
and the weight-sum identity.  Because the weights are deterministic, this
step does not interchange $\E$ and $\max$.
\end{proof}

\begin{theorem}[Level-indexed residual propagation and sampling\obj{\Rabs,\Alev}]\label{thm:level_indexed}
Assume Assumption~\ref{ass:markov_state}, $H\ge2$, and $K\ge1$; let
$\mu_S^\circ$ be carried by $\{h=H\}$, and put
$L_h:=\mathcal S\times\{h\}$.  For every block $k$ and
$1\le h<H$, let $\rho^{(h)}_{k,S}$ be a $\mathcal J_k$-measurable probability
law carried by $L_h$.
For $1\le m<H$, assume the level-indexed coefficient
\begin{equation}
c^{\mathrm{lev}}_{2,k}(m):=\sup_{\pi_{1:m}}
\left\|\frac{d(\mu_S^\circ\Pkernel^{\pi_1}_\circ\cdots
\Pkernel^{\pi_m}_\circ)}{d\rho^{(H-m)}_{k,S}}\right\|^2_{2,\rho^{(H-m)}_{k,S}}
\le c^{\mathrm{lev}}_2(m)<\infty
\label{eq:level_concentrability}
\end{equation}
almost surely, with a deterministic upper envelope uniform in $k$ and in the
policy sequence.  Define the six residual links levelwise by replacing each
$L^2(\rho_{k,S})$ norm in~\eqref{eq:fit_residual}--\eqref{eq:tgt_residual}
and Assumption~\ref{ass:residuals} by $L^2(\rho^{(h)}_{k,S})$, retaining the
same conditional-mean convention for fit and deterministic almost-sure
convention for the other links, and write
\begin{equation}
e^{\mathrm{Bell}}_{k,h}:=\eps_{\mathrm{fit},k,h}
+\eps_{\mathrm{ker},k,h}+\eps_{\mathrm{tgt},k,h}
+\eps_{\mathrm{buf},k,h}+\eps_{\mathrm{act},k,h}
+\eps_k\bar D^{\mathrm{exp}}_{k,h}.
\label{eq:level_residual}
\end{equation}
Then every clipped abstract recursion and its true-score greedy policy satisfy
\begin{align}
\E\|\Vstar-V^{\pi_K}\|_{1,\mu_S}
&\le\mathcal B_K
+2\sum_{k=0}^{K-1}\ \sum_{m=K-k}^{H-1}
\geff^m\sqrt{c^{\mathrm{lev}}_2(m)}\,
e^{\mathrm{Bell}}_{k,H-m} \notag\\
&\le\mathcal B_K
+2\phi^{(H)}_{\mathrm{lev}}\max_{\substack{0\le k<K\\1\le h<H}}
e^{\mathrm{Bell}}_{k,h},
\qquad
\phi^{(H)}_{\mathrm{lev}}:=\sum_{m=1}^{H-1}
m\geff^m\sqrt{c^{\mathrm{lev}}_2(m)}.
\label{eq:level_propagation}
\end{align}
Empty inner sums are zero.  In particular, if
$c^{\mathrm{lev}}_2(m)\le\bar c$, then
$\phi^{(H)}_{\mathrm{lev}}\le\sqrt{\bar c}\,H(H-1)/2=O(H^2\sqrt{\bar c})$.

Under the $\Alev$ clause of Assumption~\ref{ass:generative}, a realization in
a class closed under slice assembly draws
$n_{k,h}$ conditionally i.i.d.\ states from $\rho^{(h)}_{k,S}$, executes the
level-$h$ behavior law, observes fresh true-kernel outcomes, and fits
$V_{k+1}|_{L_h}$ with a separate head or regressor.  Thus only the fit slot at level $h$ is learned from those
$n_{k,h}$ labels; the other five slots in~\eqref{eq:level_residual} are
assumed or separately bounded under the same $\rho^{(h)}_{k,S}$.  The exact
sample count is
\begin{equation}
N^{\mathrm{lev}}_{0:K-1}:=\sum_{k=0}^{K-1}\sum_{h=1}^{H-1}n_{k,h};
\qquad n_{k,h}\equiv n\ \Longrightarrow\
N^{\mathrm{lev}}_{0:K-1}=K(H-1)n.
\label{eq:level_sample_count}
\end{equation}
This counts direct slice-reset outcomes.  Rejection sampling from a single
reset law instead has an additional clock-mass-dependent cost.
An optional top-slice fit adds $\sum_{k<K}n_{k,H}$ samples while leaving the
propagation bound unchanged.  A shared-parameter multitask implementation can
use the same propagation formula once its joint fit residual is controlled.
\end{theorem}

\begin{proof}
In the proof of Lemma~\ref{lem:varying_propagation}, every depth-$m$ occupancy
lies on $L_{H-m}$; Cauchy--Schwarz with~\eqref{eq:level_concentrability} bounds
its residual integral by
$\sqrt{c_2^{\mathrm{lev}}(m)}\|e_k\|_{2,\rho^{(H-m)}_{k,S}}$.
The two resolvent branches and the exact count $\min\{K,m\}$ give
\eqref{eq:level_propagation}; summing disjoint direct-reset blocks gives
\eqref{eq:level_sample_count}.
\end{proof}

\subsection{Matched-budget propagation}\label{subsec:matched_budget}

A shared-reset run uses one sampling distribution across horizon levels,
whereas a direct-level run samples each level separately.  We compare their
policy-loss bounds under a common label budget for the updates that receive
nonzero propagation weights.  Coordinate $i$ has a statistical term
$b_i n_i^{-\nu}$; the exponent $\nu$ covers parametric, tabular, and
nonparametric rates.

\begin{theorem}[Matched-budget propagation and optimal planned allocation\obj{\Aiid,\Alev}]
\label{thm:matched_budget}
Fix a terminal copy $K$ and $\nu\in(0,1]$, and apply
Theorem~\ref{thm:hmpdrl_end_to_end} at $s=2$ to a shared-reset run and
Theorem~\ref{thm:level_indexed} to a direct-level-reset run on the same model,
with the same evaluation law and initialization.  Denote their terminal
true-score greedy policies by $\pi_K^{\mathrm{sh}}$ and
$\pi_K^{\mathrm{lev}}$, respectively, and define the active
index sets
\begin{equation}
\begin{aligned}
\mathcal K_K&:=\{k:\max\{0,K-H+1\}\le k<K\},\\
\mathcal I_K&:=\{(k,h):0\le k<K,\ 1\le h<H,\ K-k\le H-h\},
\end{aligned}
\label{eq:budget_active_sets}
\end{equation}
and, for $(k,h)\in\mathcal I_K$, put
\begin{equation}
A^{\mathrm{lev}}_{k,h}:=
2\geff^{H-h}\sqrt{c_2^{\mathrm{lev}}(H-h)}.
\label{eq:level_budget_weight}
\end{equation}
Suppose the Bellman-residual envelopes separate into deterministic floors and
statistical terms,
\begin{equation}
e_{k,2}^{\mathrm{Bell}}\le r_k^{\mathrm{sh}}
+b_k^{\mathrm{sh}}n_k^{-\nu},\qquad
e_{k,h}^{\mathrm{Bell}}\le r_{k,h}^{\mathrm{lev}}
+b_{k,h}^{\mathrm{lev}}n_{k,h}^{-\nu},
\label{eq:budget_rate_envelopes}
\end{equation}
on their respective active sets.  The floors may contain any label-independent
component of the six links.  The coefficients $b_i$ are fixed
before the allocation is chosen and, in particular, may not absorb a factor
such as $(\log n_i)^q$ that varies with the coordinate allocation.  Omit
coordinates whose statistical objective coefficient $c_i$ is zero and
define
\begin{align}
\mathsf C_{\nu,K}^{\mathrm{sh}}
&:=\left\{\sum_{k\in\mathcal K_K}
\bigl(w^{(H)}_{K,k}b_k^{\mathrm{sh}}\bigr)^{1/(1+\nu)}\right\}^{1+\nu},
&F_K^{\mathrm{sh}}&:=\sum_{k\in\mathcal K_K}
w^{(H)}_{K,k}r_k^{\mathrm{sh}},
\label{eq:shared_budget_constant}\\
\mathsf C_{\nu,K}^{\mathrm{lev}}
&:=\left\{\sum_{(k,h)\in\mathcal I_K}
\bigl(A^{\mathrm{lev}}_{k,h}b_{k,h}^{\mathrm{lev}}\bigr)^{1/(1+\nu)}\right\}^{1+\nu},
&F_K^{\mathrm{lev}}&:=\sum_{(k,h)\in\mathcal I_K}
A^{\mathrm{lev}}_{k,h}r_{k,h}^{\mathrm{lev}}.
\label{eq:level_budget_constant}
\end{align}
Under continuous terminal active-window budgets
$\sum_{k\in\mathcal K_K}n_k=\mathsf N_{\mathrm{sh}}$ and
$\sum_{(k,h)\in\mathcal I_K}n_{k,h}=\mathsf N_{\mathrm{lev}}$, the unique
optimal allocations on positive-weight coordinates are
\begin{equation}
n_i=\mathsf N\,
\frac{c_i^{1/(1+\nu)}}{\sum_jc_j^{1/(1+\nu)}},
\qquad
\min_{n_i>0:\,\sum_i n_i=\mathsf N}\sum_i c_i n_i^{-\nu}
=\frac{\bigl(\sum_i c_i^{1/(1+\nu)}\bigr)^{1+\nu}}{\mathsf N^\nu},
\label{eq:matched_budget_allocation}
\end{equation}
with $c_i=w^{(H)}_{K,k}b_k^{\mathrm{sh}}$ for shared reset and
$c_i=A^{\mathrm{lev}}_{k,h}b_{k,h}^{\mathrm{lev}}$ for direct level reset.
Both terminal-window budgets count observed true-kernel labels and charge only
the displayed active coordinates.  Labels consumed outside these sets,
including earlier zero-weight blocks when $K>H-1$, must be added separately
when reporting total run cost.  The quantity $\mathsf N_{\mathrm{lev}}$ counts
direct slice-reset outcomes, with any rejection overhead added before
comparing physical simulator calls.  The resulting terminal-policy bounds are
\begin{align}
\E\|\Vstar-V^{\pi_K^{\mathrm{sh}}}\|_{1,\mu_S}
&\le\mathcal B_K+F_K^{\mathrm{sh}}
+\mathsf C_{\nu,K}^{\mathrm{sh}}\mathsf N_{\mathrm{sh}}^{-\nu},
\label{eq:shared_matched_budget}\\
\E\|\Vstar-V^{\pi_K^{\mathrm{lev}}}\|_{1,\mu_S}
&\le\mathcal B_K+F_K^{\mathrm{lev}}
+\mathsf C_{\nu,K}^{\mathrm{lev}}\mathsf N_{\mathrm{lev}}^{-\nu}.
\label{eq:level_matched_budget}
\end{align}

If the rate on coordinate $i$ is valid only for $n_i\ge L_i$, let
$L_i\in\mathbb N$ with $L_i\ge1$, $\mathsf N\ge\sum_iL_i$, and define
\begin{equation}
\Psi_\nu(c,L,\mathsf N)
:=\min_{n_i\ge L_i:\,\sum_i n_i=\mathsf N}
\sum_i c_i n_i^{-\nu}.
\label{eq:constrained_budget_value}
\end{equation}
For $\mathsf N>\sum_iL_i$, its unique continuous minimizer is
\begin{equation}
n_i^{L}=\max\left\{L_i,
\left(\frac{\nu c_i}{\lambda}\right)^{1/(1+\nu)}\right\},
\qquad \sum_i n_i^{L}=\mathsf N,
\label{eq:water_filling_allocation}
\end{equation}
where the budget equation determines a unique $\lambda>0$; if
$\mathsf N=\sum_iL_i$, the unique feasible allocation is $n_i^L=L_i$.
When lower bounds apply, replace the last statistical term in
\eqref{eq:shared_matched_budget} or~\eqref{eq:level_matched_budget} by the
corresponding value $\Psi_\nu$.  Formula~\eqref{eq:matched_budget_allocation}
and the closed forms~\eqref{eq:shared_matched_budget}--
\eqref{eq:level_matched_budget} remain valid exactly when every unconstrained
proportional allocation satisfies its lower bound.

There is also an implementable integer schedule.  Starting from
$\lfloor n_i^L\rfloor$, distribute the remaining
$\mathsf N-\sum_i\lfloor n_i^L\rfloor$ labels one at a time to a coordinate
with largest current marginal decrease
\begin{equation}
\Delta_i(n_i):=c_i\{n_i^{-\nu}-(n_i+1)^{-\nu}\}.
\label{eq:integer_budget_allocation}
\end{equation}
breaking ties by a fixed index order.  The result is feasible, uses the full
budget, and has statistical objective at most
$2^\nu\Psi_\nu(c,L,\mathsf N)$.  Here
$|\mathcal K_K|=\min\{K,H-1\}$ and, with $J=\min\{K,H-1\}$,
$|\mathcal I_K|=JH-J(J+1)/2$ before zero-weight coordinates are omitted.
\end{theorem}

\noindent\emph{Proof:} See Appendix~\ref{supp:oa3}.

\begin{corollary}[Matched-budget horizon geometry on the clock witness\obj{\Aiid,\Alev}]
\label{cor:matched_budget_geometry}
Consider the one-state-per-level horizon-indexed family with unit survival,
$K_H\ge H-1$, and $s=2$.  Use fresh
$(\textsc{sample},\textsc{obs})$ outcomes, the declared exact-score oracle,
$W_k=V_k$, greedy collection with $\eps_k=0$, $O=S$, and a fitted procedure
with the displayed statistical envelope.  Assume the unconstrained planned
allocations satisfy every validity threshold $L_i$; otherwise the exact
comparison is the constrained value $\Psi_\nu$ from
\eqref{eq:constrained_budget_value}.  Under these choices, the kernel, target,
replay, action, exploration, aliasing, and drift links vanish, the top-slice boundary
is zero, and the fit link is the only nonzero term.  Suppose its statistical
constants satisfy $b_k^{\mathrm{sh}}\asymp b_H^{\mathrm{sh}}$ and
$b_{k,h}^{\mathrm{lev}}\asymp b_H^{\mathrm{lev}}$ uniformly on the active
sets.  For direct level reset take $c_2^{\mathrm{lev}}(m)=1$.  For shared
reset compare the uniform clock law with the coefficient-optimal law
\eqref{eq:clock_attainment}.  Then:
\begin{enumerate}[label=\textnormal{(\alph*)},leftmargin=2.2em,itemsep=2pt,topsep=3pt]
\item In the near-unit-discount regime,
\begin{equation}
\mathsf C_{\nu,K_H}^{\mathrm{sh,unif}}
\asymp\mathsf C_{\nu,K_H}^{\mathrm{sh,opt}}
\asymp b_H^{\mathrm{sh}}H^{\nu+5/2},
\qquad
\mathsf C_{\nu,K_H}^{\mathrm{lev}}
\asymp b_H^{\mathrm{lev}}H^{2\nu+2}.
\label{eq:matched_near_unit_geometry}
\end{equation}
For the root-$n$ case $\nu=1/2$, both clock geometries therefore contribute
$H^3/\sqrt{\mathsf N}$ before their regression constants are inserted.
\item In the fixed-discount regime,
\begin{equation}
\mathsf C_{\nu,K_H}^{\mathrm{sh,unif}}\asymp
b_H^{\mathrm{sh}}H^{1/2},\qquad
\mathsf C_{\nu,K_H}^{\mathrm{sh,opt}}\asymp b_H^{\mathrm{sh}},\qquad
\mathsf C_{\nu,K_H}^{\mathrm{lev}}\asymp b_H^{\mathrm{lev}}.
\label{eq:matched_fixed_geometry}
\end{equation}
\end{enumerate}
All constants are uniform in $H$.  Direct slice access and an optimized
shared clock law have the same fixed-discount budget order, while direct access
achieves it without tuning cross-slice clock masses.  In the one-state-per-level
tabular specialization, the fit bound~\eqref{eq:tabular_rate_expected} gives
$b_H^{\mathrm{sh}}\asymp\Vmax\sqrt H$ and
$b_H^{\mathrm{lev}}\asymp\Vmax$.  At $\nu=1/2$, the near-unit statistical
terms are therefore respectively
\begin{equation}
\Theta\!\left(\frac{\Vmax H^{7/2}}{\sqrt{\mathsf N_{\mathrm{sh}}}}\right)
\quad\text{and}\quad
\Theta\!\left(\frac{\Vmax H^3}{\sqrt{\mathsf N_{\mathrm{lev}}}}\right),
\label{eq:matched_tabular_geometry}
\end{equation}
for the displayed bounds.  This is a factor-$\sqrt H$ statistical
advantage for direct slice reset under equal sufficiently large terminal-window
label budgets.
\end{corollary}

\noindent\emph{Proof:} See Appendix~\ref{supp:oa3}.

\begin{remark}[Routed neural specialization\obj{\Aiid,\Alev}]
\label{rem:routed_neural_budget}
For a fixed horizon, the neural rate of Proposition~\ref{prop:vreg_rate} can
be compared under a common terminal-window budget while retaining its actual
validity thresholds $L_i$ and allocation-dependent logarithms.  One may either
upper-bound each $(\log n_i)^q$ by the allocation-independent
$(\log\mathsf N)^q$ before applying Theorem~\ref{thm:matched_budget}, or
optimize the logarithmic objective directly.  The regression constant may
depend on $H$ through $\Vmax$ and the routed approximation and entropy
constants.  Corollary~\ref{cor:matched_budget_geometry} therefore establishes the
factor-$\sqrt H$ comparison for its displayed tabular constants; an analogous
growing-horizon neural comparison requires uniform-in-$H$ approximation,
entropy, and reward-normalization bounds.
\end{remark}

\section{Online-action residual}\label{sec:action}

The behavior policy may score actions with a network that differs from the
frozen iterate $V_k$.  Using the optimization and acting-snapshot index sets
defined before~\eqref{eq:envelopes}, measure this displacement by
\begin{equation}
\begin{aligned}
\dpath&:=\sup_{t\in\mathcal I_k^{\mathrm{act}}}\|V_{\theta_{k,t}}-V_k\|_\infty,
&\dpath&\le\dnet\quad\text{a.s. for }\Arun,\\
\dsnap&:=\|W_k-V_k\|_\infty,
&\dsnap&\le\dnet\quad\text{a.s. for }\Aiid.
\end{aligned}
\label{eq:delta_two_objects}
\end{equation}
The deterministic envelope $\dnet$ bounds the appropriate displacement for
each analysis object.  Together with the score error $\etascorek$, it gives
the perturbation scale
\begin{equation}
\Lambda_k:=\geff\dnet+\etascorek,
\qquad
\etascorek\ \text{as bounded in~\eqref{eq:envelopes}} .
\label{eq:Lambda_def}
\end{equation}
Network drift and score error therefore enter on the same scale.
A score perturbation of size $\Lambda_k$ can change the greedy action only
at states whose action gap is at most $2\Lambda_k$.  The margin condition
controls how much replay probability lies in this set.  We use a
Mammen--Tsybakov-type condition on the frozen iterates' action gaps and relate
it to classical optimal-gap conditions
\cite{mammen1999smooth,tsybakov2004optimal,farahmand2011action_gap,bellemare2016actiongap}.

\subsection{Margin condition}\label{subsec:act_setup}

\begin{definition}[Frozen-iterate action-gap margin: local and global forms\obj{\Rabs}]\label{def:margin}
The replay laws satisfy the \emph{local frozen-iterate}
$(C_{\mathrm{marg}},\alpha)$ margin condition at scale $[u_0,\bar u]$ if
there are \emph{deterministic} constants $C_{\mathrm{marg}}<\infty$ and
$\alpha\ge0$ such that, almost surely and for every $k$,
\begin{equation}
\rho_{k,S}\bigl\{\stilde:\Delta_Q^{(k)}(\stilde)\le u\bigr\}\le C_{\mathrm{marg}} u^\alpha
\qquad\text{for all }u\in[u_0,\bar u],
\label{eq:margin}
\end{equation}
and the \emph{global frozen-iterate} condition if the same holds for all
$u>0$.  This is a condition on the random pair $(Q_{V_k},\rho_{k,S})$, uniform
over iterations with deterministic constants; it is distinct from a margin
stated only for the fixed optimal score $Q_{\Vstar}$.  The always-valid fallback
is $(C_{\mathrm{marg}},\alpha)=(1,0)$.  Each result invokes the condition only
at its displayed scale ($2\Lambda_k$ or $2\etascoreK$); a vanishing-scale claim
requires a common interval $(0,\bar u]$.
\end{definition}

\begin{remark}[Ties and a positive frozen-iterate exponent\obj{\Rabs}]
If~\eqref{eq:margin} with $\alpha>0$ holds down to zero, then
$\rho_{k,S}\{\Delta_Q^{(k)}=0\}=0$; on a finite full-support law this forbids
ties at every frozen iterate.  A condition imposed at one scale has no such
condition.  By contrast, a fixed-$Q^\star$ positive-gap condition can exclude
the zero-gap event and retain a separate optimal-tie mass, as follows.
\end{remark}

\begin{proposition}[Transfer from a fixed optimal gap to frozen iterates\obj{\Gen,\Rabs}]
\label{prop:reference_margin_transfer}
For each $k$, suppose a deterministic $\delta_k^\star\ge0$ satisfies
\begin{equation}
\|Q_{V_k}-Q_{\Vstar}\|_\infty\le\delta_k^\star
\qquad\text{almost surely}.
\label{eq:optimal_score_tube}
\end{equation}
Suppose also that deterministic $\tau_\star\in[0,1]$,
$C_\star<\infty$, and $\alpha_\star\ge0$ satisfy, almost surely and uniformly
in $k$,
\begin{equation}
\rho_{k,S}\{\Delta_Q^\star=0\}\le\tau_\star,
\qquad
\rho_{k,S}\{0<\Delta_Q^\star\le v\}
\le C_\star v^{\alpha_\star}
\label{eq:optimal_gap_margin}
\end{equation}
at every scale $v$ in a declared interval.  Whenever
$u+2\delta_k^\star$ belongs to that interval,
\begin{equation}
\rho_{k,S}\{\Delta_Q^{(k)}\le u\}
\le \tau_\star+C_\star(u+2\delta_k^\star)^{\alpha_\star}.
\label{eq:optimal_to_iterate_margin}
\end{equation}
In particular, if $\tau_\star=0$, the optimal-gap condition is available on
the required enlarged scales, and $\delta_k^\star\le c u_0$ uniformly for some
$c\ge0$, then
the local frozen-iterate condition on $[u_0,\bar u]$ holds with exponent
$\alpha_\star$ and constant $C_\star(1+2c)^{\alpha_\star}$.
At the action
scale $u=2\Lambda_k$, the weaker tube
$\delta_k^\star\le c\Lambda_k$ already preserves the exponent because the
right-hand side of~\eqref{eq:optimal_to_iterate_margin} becomes
$C_\star[2(1+c)\Lambda_k]^{\alpha_\star}$ when $\tau_\star=0$.
Finally,
\begin{equation}
\delta_k^\star=\geff b_k
\quad\text{is valid whenever}\quad
\|V_k-\Vstar\|_\infty\le b_k,
\label{eq:value_to_score_tube}
\end{equation}
so the transfer can be verified from a deterministic value-iterate tube.
\end{proposition}

\begin{proof}
Write $q=Q_{V_k}$ and $q^\star=Q_{\Vstar}$.  If
$\Delta_Q^\star(\stilde)\le2\delta_k^\star$, then it already lies below
$u+2\delta_k^\star$.  If instead
$\Delta_Q^\star(\stilde)>2\delta_k^\star$, the optimal maximizer is unique and
remains the maximizer of $q$, while its separation from every competitor is at
least $\Delta_Q^\star(\stilde)-2\delta_k^\star$.  Hence
\begin{equation*}
\{\Delta_Q^{(k)}\le u\}
\subseteq\{\Delta_Q^\star\le u+2\delta_k^\star\}.
\end{equation*}
Splitting the latter event into zero and positive gaps gives
\eqref{eq:optimal_to_iterate_margin}.  If
$\delta_k^\star\le c u_0\le cu$, then
$u+2\delta_k^\star\le(1+2c)u$, proving the local statement.  The action-scale
claim is the same calculation with $u=2\Lambda_k$.  Finally,
$\|Q_V-Q_W\|_\infty\le\geff\|V-W\|_\infty$ proves
\eqref{eq:value_to_score_tube}.
\end{proof}

\begin{lemma}[Exact gap identity and pointwise control\obj{\Rabs}]\label{lem:act_identity}
In the notation of \S\ref{subsec:act_setup}:
\begin{enumerate}[label=\textnormal{(\roman*)},leftmargin=2.2em,nosep]
\item $\|Q_{V_{\theta_k}}-Q_{V_k}\|_\infty\le\geff\dnet$, and hence
$\|\widehat Q_{V_{\theta_k}}-Q_{V_k}\|_\infty\le\Lambda_k$
with $\Lambda_k$ as in~\eqref{eq:Lambda_def};
\item if $\pi^{\mathrm{on}}_k=(1-\eps_k)\delta_{a^{\mathrm{on}}_k}
+\eps_k\xi_k$ is a single $\eps_k$-greedy policy, then
\[
\begin{aligned}
\bigl(\Tcal^{\pi^{\mathrm{on}}_k}V_k-
\Tcal^{\pi^{\mathrm{tgt}}_k}V_k\bigr)(\stilde)
&=(1-\eps_k)\bigl(Q_{V_k}(\stilde,a^{\mathrm{on}}_k)
-Q_{V_k}(\stilde,a^{\mathrm{tgt}}_k)\bigr).
\end{aligned}
\]
Its absolute value is $(1-\eps_k)\subopt_k(\stilde)$;
here
$\subopt_k:=Q_{V_k}(\cdot,a^{\mathrm{tgt}}_k)
-Q_{V_k}(\cdot,a^{\mathrm{on}}_k)$;
\item \textnormal{(generic perturbation)} let $q:\Stildeo\times\Aspace\to\R$ be
any score with $\|q-Q_{V_k}\|_\infty\le\Lambda$ for some $\Lambda\ge0$, and let
$a^q(\stilde)\in\arg\max_aq(\stilde,a)$.  Then the induced suboptimality
$\subopt^q_k:=Q_{V_k}(\cdot,a^{\mathrm{tgt}}_k)-Q_{V_k}(\cdot,a^q)$ satisfies
\[
0\le\subopt^q_k(\stilde)\le2\Lambda
\qquad\text{and}\qquad
\{\subopt^q_k>0\}\subseteq\{\Delta_Q^{(k)}\le2\Lambda\}.
\]
Applied to the implemented score $q=\widehat Q_{V_{\theta_k}}$ with
$\Lambda=\Lambda_k$, this gives $0\le\subopt_k\le2\Lambda_k$ and
$\{\subopt_k>0\}\subseteq\{\Delta_Q^{(k)}\le2\Lambda_k\}$; in the idealization
$\etascorek=0$ it gives the same statements with
$\Lambda_k=\geff\dnet$.
\end{enumerate}
\end{lemma}

\begin{proof}
$Q_V-Q_W=\geff\int(V-W)d\Pkernel$ proves~(i).  In~(ii), both policies share
$\eps_k\xi_k$, so exploration cancels exactly; its separate difference from
$\Tcal$ remains the $\eps_k\Dbar$ link.  For~(iii), apply the perturbation
bound twice:
$Q_{V_k}(\stilde,a^q)\ge q(\stilde,a^q)-\Lambda\ge q(\stilde,a^{\mathrm{tgt}}_k)-\Lambda
\ge Q_{V_k}(\stilde,a^{\mathrm{tgt}}_k)-2\Lambda$, so $\subopt^q_k\le2\Lambda$.
If $\Delta_Q^{(k)}(\stilde)>2\Lambda$ then every $a\ne a^{\mathrm{tgt}}_k$ has
$q(\stilde,a)\le Q_{V_k}(\stilde,a)+\Lambda<\max_aQ_{V_k}(\stilde,a)-\Lambda
\le q(\stilde,a^{\mathrm{tgt}}_k)$, so $a^q=a^{\mathrm{tgt}}_k$.
\end{proof}

\noindent For every $V\in\mathcal V^{\mathrm{clip}}$ and every Markov policy
$\pi$, the elementary pointwise inequality $\Tcal^\pi V\le\Tcal V$ follows by
averaging action scores below their maximum.

\subsection{Gap identity and upper bound}\label{subsec:act_upper}

\begin{theorem}[Norm-indexed action residual for the abstract recursion\obj{\Rabs}]\label{thm:act_upper}
In the notation of \S\ref{subsec:act_setup}, let $q_k$ be measurable with
$\|q_k-Q_{V_k}\|_\infty\le\Lambda_k$, let $a^{q_k}$ maximize $q_k$ using the
fixed tie rule, and put
$\pi^q_k=(1-\eps_k)\delta_{a^{q_k}}+\eps_k\xi_k$.  For every $r\in[1,2]$,
\begin{equation}
\bigl\|\Tcal^{\pi^q_k}V_k-\Tcal^{\pi^{\mathrm{tgt}}_k}V_k\bigr\|_{r,\rho_{k,S}}
\le(1-\eps_k)2\Lambda_k.
\label{eq:act_upper_linear}
\end{equation}
Under the global frozen-iterate margin~\eqref{eq:margin}, or its local form when
$2\Lambda_k\in[u_0,\bar u]$,
\begin{equation}
\bigl\|\Tcal^{\pi^q_k}V_k-\Tcal^{\pi^{\mathrm{tgt}}_k}V_k\bigr\|_{r,\rho_{k,S}}
\le(1-\eps_k)C_{\mathrm{marg}}^{1/r}
(2\Lambda_k)^{1+\alpha/r}.
\label{eq:act_upper}
\end{equation}
Alternatively, under Proposition~\ref{prop:reference_margin_transfer} at
$u=2\Lambda_k$,
\begin{equation}
\bigl\|\Tcal^{\pi^q_k}V_k-\Tcal^{\pi^{\mathrm{tgt}}_k}V_k\bigr\|_{r,\rho_{k,S}}
\le(1-\eps_k)2\Lambda_k
\Bigl\{\tau_\star+C_\star(2\Lambda_k+2\delta_k^\star)^{\alpha_\star}\Bigr\}^{1/r}.
\label{eq:act_upper_optimal_margin}
\end{equation}
Thus a zero optimal-tie mass and
$\delta_k^\star\le c\Lambda_k$ yield the same
$\Lambda_k^{1+\alpha_\star/r}$ exponent directly from a fixed-$Q^\star$
margin, with its explicit constant.
Consequently one may choose
\[
\epsactr{r}\le(1-\eps_k)\min\{2\Lambda_k,
C_{\mathrm{marg}}^{1/r}(2\Lambda_k)^{1+\alpha/r}\}.
\]
When the reference-gap route is used, one may instead choose
\begin{equation*}
\epsactr{r}\le(1-\eps_k)\min\!\left\{2\Lambda_k,
2\Lambda_k\Bigl[\tau_\star+C_\star
(2\Lambda_k+2\delta_k^\star)^{\alpha_\star}\Bigr]^{1/r}\right\}.
\end{equation*}
Taking $r=p=s/(s-1)$ supplies the propagation envelope $\epsactp$ with
exponent $1+\alpha(1-1/s)$; taking $r=2$ supplies the envelope $\epsact$
used in the regression bound, with exponent $1+\alpha/2$.  The two coincide at $s=2$.
\end{theorem}

\begin{proof}
Lemma~\ref{lem:act_identity}(ii)--(iii) makes the integrand
$(1-\eps_k)\subopt^{q_k}_k$, bounded by $(1-\eps_k)2\Lambda_k$ and supported,
when nonzero, on $\{\Delta_Q^{(k)}\le2\Lambda_k\}$.  This proves
\eqref{eq:act_upper_linear}; under the frozen-iterate margin its $r$th power is at most
$(1-\eps_k)^r C_{\mathrm{marg}}(2\Lambda_k)^{r+\alpha}$, proving
\eqref{eq:act_upper}.  Replacing the support probability by
\eqref{eq:optimal_to_iterate_margin} proves
\eqref{eq:act_upper_optimal_margin}.
\end{proof}

\begin{corollary}[Abstract-recursion mixtures and drift control\obj{\Arun,\Aiid}]\label{cor:act_discharge}
The unconditional bound~\eqref{eq:act_upper_linear} remains valid for any
state-dependent convex mixture of policies sharing $(\eps_k,\xi_k)$ whose
greedy scores are within $\Lambda_k$ of $Q_{V_k}$.  Under the global
frozen-iterate margin, or under its local form with
$2\Lambda_k\in[u_0,\bar u]$, the
margin-improved bound~\eqref{eq:act_upper} also remains valid: convexity
preserves both the pointwise $2\Lambda_k$ bound and its common gap-supported
set.  The same argument preserves the fixed-$Q^\star$ transfer
bound~\eqref{eq:act_upper_optimal_margin}.  Thus Theorem~\ref{thm:act_upper} applies directly to the snapshot law
$\beta_k$ of $\Aiid$ with $q_k=\widehat Q_{W_k}$, while this mixture statement
applies to $\pibar$ of $\Arun$ with the corresponding online scores; their
distance from $Q_{V_k}$ is at most
$\Lambda_k=\geff\dnet+\etascorek$ by~\eqref{eq:Lambda_def}.  Moreover, if
$V_{\theta_{k,0}}=V_k$, $\theta\mapsto V_\theta$ is $L$-Lipschitz in supremum
norm, and $\|\theta_{t+1}-\theta_t\|\le\lambda_tG$, then
$\dpath\le LG\sum_{t\in\mathcal I_k^{\mathrm{opt}}}\lambda_t$.  For a general
initialization, add $\|V_{\theta_{k,0}}-V_k\|_\infty$.  Here $L$ is an explicit
regularity assumption on the parameterization.
\end{corollary}

Combining Lemma~\ref{lem:varying_propagation} and
Theorem~\ref{thm:act_upper}
gives the continuous coverage--margin tradeoff: the propagated action power is
$1+\alpha(1-1/s)$ and the exact finite-$K$ clock multiplier is
$2\phi_{s,K}^{(H)}$.  This single result contains the $L^2$ and $L^\infty$
statements as endpoint cases.  The fixed-$Q^\star$ route has
the same power with $\alpha=\alpha_\star$ whenever
$\tau_\star=0$ and $\delta_k^\star=O(\Lambda_k)$ on the active window.

\subsection{One-step sharpness and lower bound}\label{sec:sharp}

Proposition~\ref{prop:act_tight} attains the action theorem's one-step
exponent.  Proposition~\ref{prop:model_floor_lower} produces a nonzero
limiting loss from a reward-score perturbation while all other residuals
vanish; the harmful behavior is generated by the implemented score.

\begin{proposition}[Componentwise sharpness of the abstract one-step residual\obj{\Rabs}]\label{prop:act_tight}
For every $\alpha>0$, $p\in[1,2]$, $\geff\in(0,1)$, and sufficiently small
drift $\delta>0$,
there exist a deterministic time-augmented MDP with an exact one-step model
(so that $\etascorek=0$ and $\Lambda_k=\geff\delta$), a frozen
$V_k$, an online $V_{\theta_k}$ with $\dnet=\delta$, and a replay
marginal $\rho_{k,S}$ that satisfies~\eqref{eq:margin} with equality for
$u\in(0,\geff c_\Delta]$, such that, when $\eps_k=0$,
\[
\bigl\|\Tcal^{\pi^{\mathrm{on}}_k}V_k-\Tcal^{\pi^{\mathrm{tgt}}_k}V_k\bigr\|_{p,\rho_{k,S}}
=\left(\frac{\alpha}{\alpha+p}\right)^{1/p}
C_{\mathrm{marg}}^{1/p}(2\geff\delta)^{1+\alpha/p}.
\]
Thus the exponent in the one-step bound~\eqref{eq:act_upper} cannot be
improved.  This establishes componentwise sharpness: the clock and action
constructions independently attain their respective exponents, while their
joint product defines a separate end-to-end lower-bound problem.
\end{proposition}

\begin{proof}
At a decision state $\sigma_t$, $t\in[0,1]$, let actions $a_\pm$ lead
deterministically to one-step states of values $v_0+c_\Delta t$ and $v_0$.
Choose $v_0,c_\Delta,\delta>0$ with
$v_0+c_\Delta+\delta\le\Rmax$, and shift the online successor values by
$\mp\delta$.  Then the true gap is $\geff c_\Delta t$, but the online action
switches exactly on $t<t_\delta:=2\delta/c_\Delta$.  Give $t$ density
$\alpha t^{\alpha-1}$; hence
$C_{\mathrm{marg}}=(\geff c_\Delta)^{-\alpha}$.  Direct integration gives
\[
\bigl\|\Tcal^{\pi^{\mathrm{on}}_k}V_k-\Tcal^{\pi^{\mathrm{tgt}}_k}V_k\bigr\|_{p,\rho_{k,S}}^p
=\int_0^{t_\delta}(\geff c_\Delta t)^p\alpha t^{\alpha-1}\,dt
=\frac{\alpha}{\alpha+p}(\geff c_\Delta)^{-\alpha}
(2\geff\delta)^{\alpha+p};
\]
taking $p$th roots proves the formula.
\end{proof}

\begin{proposition}[Abstract-recursion loss induced by reward-score error\obj{\Rabs}]\label{prop:model_floor_lower}
Fix $\geff\in(0,1)$ and normalize $\Rmax=1$.  For every
$\eta\in(0,\geff/2)$ and every $\epsilon>0$
there exist a deterministic time-augmented MDP with $H=3$ and $|\Aspace|=2$, an
evaluation law $\mu_S$, a replay marginal $\rho_{k,S}\equiv\rho_S$ of full
support, and an implemented one-step score $\widehat Q$ of the
form~\eqref{eq:greedy_action} with
$\|\widehat Q_{V_k}-Q_{V_k}\|_\infty=\eta$ for every $k$, such that the
population recursion initialized at $V_0=0$ and driven by the greedy branch of
$\widehat Q$ has zero drift, $\dnet=0$, and zero fit, kernel, target, replay,
and exploration residuals,
$\eps_{\mathrm{fit},k}=\eps_{\mathrm{ker},k}=\eps_{\mathrm{tgt},k}=\epsbuf=\eps_k=0$.
The score-induced online-action residual is the sole nonzero link, and yet
\begin{equation}
\lim_{K\to\infty}\bigl\|\Vstar-V^{\pi_K}\bigr\|_{1,\mu_S}
\;\ge\;2\geff\eta-\epsilon .
\label{eq:model_floor_lower}
\end{equation}
Thus the construction realizes the $\alpha=0$ score floor for
$0<\eta<\geff/2$.  After the finite clock transient, its harmful behavior is
generated endogenously by an implemented greedy score with
$\etascorek=\eta$ in every block.
\end{proposition}

\begin{proof}[Construction sketch]
Use the deterministic $H=3$ chain detailed in Appendix~\ref{app:proofs}, with
terminal rewards $1,1-u,v$, and perturb only the two modeled rewards at
$y_1$ by $-\eta,+\eta$.  The score error is exactly $\eta$ and switches the
choice at $y_1$ when $u<2\eta/\geff$.  Taking
$u=2\eta/\geff-\vartheta$ and
$v\in(1-u,1-u+\epsilon/(2\geff^2))$ makes the learned and optimal root actions
different after the clock transient, with
\[
\|\Vstar-V^{\pi_K}\|_{1,\mu_S}=\geff^2(1-v)
>2\geff\eta-\epsilon
\]
for sufficiently small $\vartheta$.  Exact full-support population updates,
matched collection and replay, frozen targets, and zero exploration make all
other residuals vanish.
\end{proof}

\section{Neural and tabular regression rates}\label{sec:vreg}

For $\Aiid$, we fit one sparse-ReLU spatial network at each remaining-horizon
level.  A fixed router selects the corresponding output.  This construction
allows smoothness assumptions on the spatial coordinates without imposing
smoothness across the discrete clock levels.
Assumption~\ref{ass:relu} collects the
approximation, closure, entropy, and nesting conditions.  The sparse-ReLU
exponent and corrected depth condition come from
\cite{schmidt_hieber_2020,schmidt_hieber_vu_2024}; Propositions
\ref{prop:finite_rank_closure} and~\ref{prop:routed_relu_admissibility} prove
Bellman closure and compatibility of the network class with these conditions.
The bounded-loss oracle inequality follows the covering approach of
\citet{gyorfi2002distribution}.  The mean regression target can differ from
the Bellman optimality update because it averages the behavior policy.
We retain this difference as
$D^{(2)}_k$~\cite{wang2021statistical,zanette2021exponential,chen2019information,xie2021realizability}.
Any class with the same approximation and entropy properties can replace the
network class.  The neural theorem fixes $H$; growing-horizon
comparisons require Remark~\ref{rem:routed_neural_budget}.

\begin{assumption}[Externally clipped clock-routed sparse ReLU class and Hölder closure\obj{\Aiid}]\label{ass:relu}
\emph{(Slice representation.)}  For each $h\in[H]:=\{1,\ldots,H\}$, fix a
bi-measurable spatial embedding
$\iota_h:\mathcal S\times\{h\}\to[0,1]^{d_{\mathrm{emb}}}$ onto its image,
through which the slice-$h$ rewards and kernel factor.  The deterministic
router observes the clock exactly, and the Hölder condition applies to the
spatial coordinates within each slice.  Any
other discrete task coordinates $\mathcal C$ are either routed in the same
way or included in $\iota_h$.  Smoothness, margin, and tube hypotheses are
uniform over the resulting finitely many slices.  The direct neural
convergence route uses this routed input as the Markov-sufficient analysis
state.  If the network receives only a compression $O$, the same architecture
can be analyzed through Proposition~\ref{prop:representation_routes}, with
visible-target aliasing included in the fit residual and closure checked on
the Markov state.

\emph{(Per-head and routed classes.)}  For a block of size $n\ge2H$, set
\begin{equation}
m_n:=\lfloor n/H\rfloor\ge2.
\label{eq:head_index}
\end{equation}
Let $K_{\mathcal G}\ge1$ be a common upper bound on the coordinatewise
H\"older radii and domain endpoints used to define the finitely many target
classes $\mathcal G_{0,h}$ below, and fix the raw-network envelope
\begin{equation}
B_{\mathrm{net}}\ge\max\{1,\Vmax,K_{\mathcal G}\}.
\label{eq:raw_network_envelope}
\end{equation}
For each $h$, let $\mathcal N_m^{(h)}$ be the \emph{raw} scalar sparse-ReLU
class on $[0,1]^{d_{\mathrm{emb}}}$ with weight and bias radius one, raw
output bounded by $B_{\mathrm{net}}$, exactly $L_m$ hidden layers, width
profile $\{d_{j,m}\}$, and sparsity at most $s_m$.  The budgets are
nondecreasing in $m$, dominate the constructive lower envelopes, and obey
\begin{equation}
\begin{aligned}
c_{\mathrm{lo}}\log m&\le L_m\le c_{\mathrm{hi}}(\log m)^{\xi^\star},\\
m^{\alpha^\star}&\lesssim\min_{1\le j\le L_m}d_{j,m}
\le\max_{1\le j\le L_m}d_{j,m}\lesssim m^{\xi^\star},\\
s_m&\asymp m^{\alpha^\star}(\log m)^{\xi^\star},\qquad \xi^\star\ge1.
\end{aligned}
\label{eq:architecture_scaling}
\end{equation}
Define the nested head class
$\mathcal R_m^{(h)}:=\bigcup_{r=2}^{m}\mathcal N_r^{(h)}$.
The raw clock-routed class is the finite product
\begin{equation}
\mathcal R_n^{\mathrm{clk}}
:=\left\{f:
f(s,h)=f_h\!\left(\iota_h(s,h)\right),\quad
f_h\in\mathcal R_{m_n}^{(h)},\ h\in[H],\quad
f(\stilde_{\mathrm{term}})=0\right\}.
\label{eq:routed_class}
\end{equation}
Thus the router is fixed and non-trainable, and only the selected scalar head
is evaluated.  A routed predictor has at most $Hs_{m_n}$ trainable nonzero
parameters across its heads, maximum head depth $L_{m_n}$, and no trainable
gating parameter.  This is still one global function class and one ERM under the
mixed design law $\varsigma$; it is not the direct-slice sampling object
$\Alev$.  Let $\FVn:=\Pi^{\mathrm{clip}}\mathcal R_n^{\mathrm{clk}}$, where
the level-wise projection is fixed external post-processing and contributes
no trainable parameter.  In particular, the ERM class is band-valued even
though its raw networks use the larger envelope~\eqref{eq:raw_network_envelope}.
Each finite-architecture bounded-parameter class is separable in supremum
norm; choose countable dense subsets, take their finite unions and $H$-fold
product, and then clip to obtain a fixed countable supremum-norm-dense subclass
$\mathcal F^{V,0}_n\subseteq\FVn$.

\emph{(Target and admissibility.)}  Let $\mathcal G_{0,h}$ be the slice-$h$
compositional Hölder class of~\cite[Assumption~4.2]{fan2020theoretical_full},
intersected with the band $\|g_h\|_\infty\le\Vmax^{(h)}$, with indices and
radii uniformly bounded in $h$, and define
\begin{equation}
\Gclk:=\left\{g\in\mathcal V^{\mathrm{clip}}:
\text{for every }h\text{ there is }\bar g_h\in\mathcal G_{0,h}
\text{ with }g(s,h)=\bar g_h(\iota_h(s,h))\right\}.
\label{eq:routed_target}
\end{equation}
Writing
$\mathrm{dist}^{\sup}_\infty(\mathcal C,\mathcal G)
:=\sup_{g\in\mathcal G}\inf_{V\in\mathcal C}\|V-g\|_\infty$, assume that
there is a target-uniform $C_{\mathrm{appx}}<\infty$ such that
\begin{equation}
\max_{h\in[H]}
\mathrm{dist}^{\sup}_\infty
\bigl(\Pi_h^{\mathrm{clip}}\mathcal R_m^{(h)},\mathcal G_{0,h}\bigr)
\le C_{\mathrm{appx}}m^{(\alpha^\star-1)/2},
\qquad m\ge2,
\label{eq:head_approximation}
\end{equation}
where $\Pi_h^{\mathrm{clip}}$ clips to
$[-\Vmax^{(h)},\Vmax^{(h)}]$ on slice $h$, and
$\alpha^\star=\max_jt_j/(2\beta_j^\star+t_j)\in(0,1)$.
Consequently, selecting one approximant independently in each of the finitely
many heads gives
\begin{equation}
\mathrm{dist}^{\sup}_\infty(\FVn,\Gclk)
\le C_{\mathrm{appx}}m_n^{(\alpha^\star-1)/2}.
\label{eq:routed_approximation}
\end{equation}

\emph{(Closure, entropy, nesting, and uniformity.)}  For every $n\ge2H$ and
$V\in\FVn$, assume $\Tcal V\in\Gclk$.  Put
$d_{0,m}:=d_{\mathrm{emb}}$, $d_{L_m+1,m}:=1$, and
$D_m:=\prod_{j=0}^{L_m+1}(d_{j,m}+1)$.  The head covering bound and the product
construction give, for $0<\delta\le1$,
\begin{equation}
\log N_\delta(\FVn)
\le H\left\{\log m_n+(s_{m_n}+1)
\log\!\left(\frac{2(L_{m_n}+1)D_{m_n}^{2}}{\delta}\right)\right\}.
\label{eq:routed_entropy}
\end{equation}
The head classes are nested in $m$; hence $m_n\le m_{n'}$ implies
$\mathcal F^V_n\subseteq\mathcal F^V_{n'}$.  The Hölder radii,
compositional indices, approximation and entropy constants are uniform in
$h,n,k$, and in the realized $V_k\in\FVn$.  Constants in the results may
depend on the fixed horizon $H$.

\emph{(Corrected depth check.)}  Let $i=0,\ldots,q$ index a head's composition
layers, with parameters
$(\beta_i,t_i,\beta_i^\star)$.  Set
\[
C_{\mathrm{SH}}:=\sum_{i=0}^{q}
\frac{\beta_i+t_i}{2\beta_i^\star+t_i}
\log_2(4t_i\vee4\beta_i),
\qquad
\phi_m:=\max_i m^{-2\beta_i^\star/(2\beta_i^\star+t_i)}.
\]
Then $m\phi_m=m^{\alpha^\star}$.  Consequently
$c_{\mathrm{lo}}\ge C_{\mathrm{SH}}/\log2$ verifies the corrected lower
depth requirement $C_{\mathrm{SH}}\log_2m\le L_m$, while the
polylogarithmic upper bound in~\eqref{eq:architecture_scaling} is eventually
$L_m\lesssim m\phi_m$.  The same display gives
$m\phi_m\lesssim\min_jd_{j,m}$, and its sparsity budget contains the
constructive scale $s'_m\asymp m\phi_m\log m\le s_m$.  Together with
\eqref{eq:raw_network_envelope}, these are the output-envelope, corrected
depth, width, and sparsity compatibility checks.
Equation~\eqref{eq:head_approximation},
Bellman closure, and uniform target-class radii remain explicit hypotheses in
the general compositional setting and are verified for the finite-rank model
below.
\end{assumption}

\begin{proposition}[Finite-rank smoothing implies slice-wise Bellman closure and coverage\obj{\Aiid}]\label{prop:finite_rank_closure}
Let each clock slice be $[0,1]^d$ with reference law $\nu_h$, let
$\mu_S=\nu_H\otimes\delta_H$, and use the reset law
$\varsigma=H^{-1}\sum_{h=1}^H\nu_h\otimes\delta_h$.  Suppose that, for $h>1$,
the action-$a$ transition has density
\begin{equation*}
\begin{gathered}
p_{h,a}(y\mid x)=1+\sum_{j=1}^r\phi_{h,a,j}(x)\psi_{h,a,j}(y)
\quad\text{relative to $\nu_{h-1}$},\\
\int\psi_{h,a,j}\,d\nu_{h-1}=0,\qquad 0\le p_{h,a}\le M .
\end{gathered}
\end{equation*}
and that termination occurs after level one.  If, uniformly in $(h,a,j)$,
$R_{h,a}$ and $\phi_{h,a,j}$ lie in a fixed $\beta$-H\"older ball on
$[0,1]^d$, $0<\beta\le1$, and $\|\psi_{h,a,j}\|_1$ is bounded, then
$\{\Tcal V:V\in\mathcal V^{\mathrm{clip}}\}$ lies in a fixed $\beta$-H\"older
ball on every clock slice, with a radius uniform in $h$ and $V$.  Moreover, for
every policy sequence and $1\le m<H$,
\[
c_2(m)\le HM,\qquad c_\infty(m)\le HM.
\]
Consequently, if $\mathcal G_{0,h}$ is this common-radius H\"older ball
intersected with the level-$h$ band, then
$\{\Tcal V:V\in\mathcal V^{\mathrm{clip}}\}\subset\Gclk$, and the closure and
uniform-radius clauses of Assumption~\ref{ass:relu} hold.  This proposition
makes no neural-network approximation, nesting, or entropy claim.
\end{proposition}

\begin{proof}
For $V\in\mathcal V^{\mathrm{clip}}$,
\[
(\Pkernel_{h,a}V)(x)=\int V\,d\nu_{h-1}
+\sum_{j=1}^r\phi_{h,a,j}(x)\int\psi_{h,a,j}V\,d\nu_{h-1}.
\]
The coefficients are bounded uniformly by $\Vmax\|\psi_{h,a,j}\|_1$; hence
$R_{h,a}+\geff\Pkernel_{h,a}V$ has a uniform $\beta$-H\"older radius.  A finite
maximum preserves that radius when $\beta\le1$, proving closure; the finite
number of slices makes the radius uniform in $h$.  If a slice law has density
$g$ relative to $\nu_h$, its successor density $g'$ is bounded by
$M\int g\,d\nu_h=M$.  Relative to $\varsigma$ it is therefore at most $HM$,
so $c_\infty(m)\le HM$ and
$c_2(m)\le H\int g_m^2d\nu_{H-m}\le HM$, where $g_m\le M$ is the propagated
spatial density at depth $m$.
\end{proof}

\begin{proposition}[Clock-routed ReLU admissibility for the finite-rank example\obj{\Aiid}]\label{prop:routed_relu_admissibility}
Under Proposition~\ref{prop:finite_rank_closure}, take
$d_{\mathrm{emb}}=d$, let $\iota_h$ be the identity on the spatial coordinate,
and let every $\mathcal G_{0,h}$ be the corresponding band-intersected,
common-radius $\beta$-H\"older ball.  Let $K_{\mathcal G}$ bound that radius
and choose $B_{\mathrm{net}}\ge\max\{1,\Vmax,K_{\mathcal G}\}$.  There exist
nondecreasing fixed-architecture budgets and nested per-head raw sparse-ReLU
families $\{\mathcal R_m^{(h)}\}_{m\ge2}$ satisfying
Assumption~\ref{ass:relu} with
\[
q=0,\qquad t=d,\qquad
\alpha^\star=\frac{d}{2\beta+d},
\]
including the global approximation~\eqref{eq:routed_approximation}, product
entropy~\eqref{eq:routed_entropy}, countable dense subclasses, nesting, the
output-envelope condition, the corrected depth condition, and the width and
sparsity requirements.  Hence the
finite-rank model, together with the externally clipped clock-routed
architecture, verifies all clauses of Assumption~\ref{ass:relu}.
\end{proposition}

\begin{proof}
For this $q=0$ class,
$\beta^\star=\beta$, $\phi_m=m^{-2\beta/(2\beta+d)}$, and
$m\phi_m=m^{d/(2\beta+d)}=m^{\alpha^\star}$.  Choose the exact depth, hidden
widths, and sparsity budgets nondecreasingly so that, for all sufficiently
large $m$,
\begin{equation}
\begin{aligned}
\frac{\beta+d}{2\beta+d}\log_2(4d\vee4\beta)\log_2m
&\le L_m\lesssim m\phi_m,\\
m\phi_m&\lesssim\min_jd_{j,m},\\
m\phi_m\log m&\lesssim s_m.
\end{aligned}
\label{eq:finite_rank_architecture_check}
\end{equation}
while retaining the upper envelopes in~\eqref{eq:architecture_scaling}.
Enlarge the finitely many small-$m$ budgets if necessary.  These inequalities
are mutually compatible because $\alpha^\star\in(0,1)$ and $\xi^\star\ge1$.

The approximation part of Schmidt--Hieber's construction, with condition
\textnormal{(ii')} of the correction, now applies with sample-size index $m$
and output envelope $B_{\mathrm{net}}$~\cite{schmidt_hieber_2020,schmidt_hieber_vu_2024}.
Uniformly over the fixed H\"older ball it produces a unit-parameter raw
network $\widetilde f_{h,m}$, bounded by $B_{\mathrm{net}}$, with
\[
\|\widetilde f_{h,m}-g_h\|_\infty
\le C m^{-\beta/(2\beta+d)}
=C m^{(\alpha^\star-1)/2}.
\]
The constructive network embeds into the declared fixed architecture: pad
hidden layers by inactive neurons, and insert any additional identity layers
before the constructive network.  Inputs lie in $[0,1]^d$, so unit-weight,
zero-bias identity layers survive ReLU unchanged; their
$O(dL_m)$ additional nonzero parameters are absorbed by $s_m$.  This is the
standard width/depth embedding used in the cited construction.  Thus
$\widetilde f_{h,m}\in\mathcal N_m^{(h)}$ after padding.  For the finitely many
small indices, the zero network belongs to the class and the constant $C$ can
be enlarged, so the bound holds for every $m\ge2$.

Because $g_h$ lies in the level-$h$ band, Lemma~\ref{lem:clip_projection}
gives
\[
\|\Pi_h^{\mathrm{clip}}\widetilde f_{h,m}-g_h\|_\infty
\le\|\widetilde f_{h,m}-g_h\|_\infty.
\]
This proves~\eqref{eq:head_approximation}.  Take the nested hull
$\mathcal R_m^{(h)}=\bigcup_{r=2}^m\mathcal N_r^{(h)}$; the finite number of
slices makes the approximation constant uniform in $h$.  For $g\in\Gclk$,
choose the $H$ clipped approximants independently and route by the observed
clock.  The global supremum error is the maximum of the head errors, proving
\eqref{eq:routed_approximation}; exact routing therefore preserves the spatial
approximation rate without introducing a clock-gate approximation term.

For entropy, Lemma~6.4 of~\cite{fan2020theoretical_full} bounds each base
raw class $\mathcal N_r^{(h)}$; imposing the additional raw-output envelope
only takes a subclass.  A union bound over $2\le r\le m$ contributes $\log m$
to the log covering number.  Take a $\delta$-net for each resulting raw head
class and then apply $\Pi_h^{\mathrm{clip}}$.  The Cartesian product of the
$H$ clipped nets is a $\delta$-net for the routed class because its metric is
the maximum of the slice metrics.  Summing the log covering numbers gives
\eqref{eq:routed_entropy}; clipping is nonexpansive by
Lemma~\ref{lem:clip_projection}.  Finite-parameter raw classes are separable,
and finite unions, clipping, and the finite routed product preserve a
countable dense subclass.

The nested hulls and deterministic router give nesting of the global clipped
classes.  Equation~\eqref{eq:finite_rank_architecture_check} verifies the
corrected depth condition and the width/sparsity embedding, while
\eqref{eq:raw_network_envelope} verifies the output-envelope condition.
Closure and its uniform radius come from
Proposition~\ref{prop:finite_rank_closure}, completing every clause of
Assumption~\ref{ass:relu}.
\end{proof}

\begin{lemma}[Oracle inequality for squared loss with bounded targets\obj{\Gen}]\label{lem:erm_oracle}
Let $B>0$ and let $\mathcal C\subset\{V:\Stilde\to[-B,B]\}$ be deterministic and
admit a fixed countable supremum-norm-dense subclass $\mathcal C_0$.  Let
$g:\Stilde\to[-B,B]$ be measurable, and let $(S_i,Y_i)_{i\le n}$ be
i.i.d.\ with $S_i\sim\rho_{k,S}$, $|Y_i|\le B$ and $\E[Y_i\mid S_i]=g(S_i)$.
Let $\widehat V_{k+1}$ be any measurable $\zeta_n$-approximate empirical risk
minimizer over $\mathcal C_0$ (or over $\mathcal C$ itself when such a
measurable minimizer is supplied):
$n^{-1}\sum_{i\le n}(\widehat V_{k+1}(S_i)-Y_i)^2
\le\inf_{V\in\mathcal C_0}n^{-1}\sum_{i\le n}(V(S_i)-Y_i)^2+\zeta_n$ with
$\zeta_n\ge0$, exact ERM being $\zeta_n=0$.  Then, for every
$\delta\in(0,2B]$,
\begin{equation}
\E\|\widehat V_{k+1}-g\|_{2,\rho_{k,S}}^2
\le 2\bigl[\mathrm{dist}_{2,\rho_{k,S}}(\mathcal C,g)\bigr]^2
+\frac{C_1B^2}{n}\bigl(1+\log N_\delta(\mathcal C)\bigr)
+C_1B\delta+2\zeta_n ,
\label{eq:erm_oracle}
\end{equation}
with an absolute $C_1$; the proof is valid for every $C_1\ge140$ and makes no
use of its value.  Here $\mathrm{dist}_{2,\rho_{k,S}}(\mathcal C,g)$ is the
$L^2(\rho_{k,S})$ distance from the class to a \emph{single} target and
$N_\delta(\mathcal C)$ the supremum-norm covering number, assumed finite.
The same inequality holds conditionally on a sigma-algebra relative to which
$\mathcal C$, $g$, and $\rho_{k,S}$ are fixed, provided the sample is
conditionally i.i.d.\ and
$\E[Y_i\mid S_i,\mathcal K]=g(S_i)$; it is used below with
$\mathcal K=\mathcal H_k$ for $\Aiid$.
\emph{The approximation term is in $L^2(\rho_{k,S})$ and only the covering
radius is in supremum norm}; since $\rho_{k,S}$ is a probability measure the
statement implies, and is stronger than, its supremum-norm form.
If, additionally, $g\in\mathcal C$ and $\widehat V_{k+1}$ is a supplied
measurable exact ERM over $\mathcal C$, then for every $\tau\in(0,1)$,
conditionally on the same fixed information when present, with probability at
least $1-\tau$,
\begin{equation}
\|\widehat V_{k+1}-g\|_{2,\rho_{k,S}}^2
\le\frac{C_1B^2}{n}
\bigl(1+\log N_\delta(\mathcal C)+\log(1/\tau)\bigr)
+C_1B\delta .
\label{eq:erm_oracle_hp}
\end{equation}
\end{lemma}

\noindent\emph{Proof:} See Appendix~\ref{app:proofs}.

\begin{proposition}[Generative-reset clock-routed scalar-$V$ regression\obj{\Aiid}]\label{prop:vreg_rate}
Under Assumption~\ref{ass:relu}, fix $n\ge2H$, condition on $\mathcal H_k$, and let
$V_k\in\FVn$.  Let $(S_i,\mathbf a_i,r_i,\stilde'_i,d_i)_{i\le n}$ be the
i.i.d.\ sample of $\Aiid$~\eqref{eq:generative_variant}, with
$S_i\sim\rho_{k,S}\equiv\varsigma$,
$\mathbf a_i\sim\beta_k(\cdot\mid S_i)$, fresh true-kernel outcomes and labels
$Y_i=r_i+\geff(1-d_i)V_k(\stilde'_i)$, and let $\widehat V_{k+1}$ be any
measurable $\zeta_n$-approximate ERM over the fixed countable dense subclass
$\mathcal F^{V,0}_n$.  By construction,
$\E[Y_i\mid S_i,\mathcal H_k]=G_k(S_i)$.  Define the \emph{$L^2$
policy-image defect}
\begin{equation}
D^{(2)}_k:=\bigl\|G_k-\Tcal V_k\bigr\|_{2,\rho_{k,S}},
\label{eq:l2_defect}
\end{equation}
a quantity and not a hypothesis.  Then there is a constant
$C_{\mathrm{Vreg},H}$, uniform in $k$, $n$, and the realized
$V_k\in\FVn$ under the uniformity clause of Assumption~\ref{ass:relu}.
It may depend on $H,\geff,\Rmax$ through $\Vmax$, on $|\Aspace|$ through the
assumed target class, and on $C_1$, $C_{\mathrm{appx}}$, the architecture and
covering constants, the fixed router and slice embeddings, and the
compositional indices and H\"older radii, but not on $n$, $k$, or the realized
iterate.  Then
\begin{equation}
\E\bigl[\|\widehat V_{k+1}-G_k\|_{2,\rho_{k,S}}\,\big|\,\mathcal H_k\bigr]
\le C_{\mathrm{Vreg},H}\,(\log n)^{(1+2\xi^\star)/2}m_n^{(\alpha^\star-1)/2}
+\sqrt{2\zeta_n}+\sqrt2\,D^{(2)}_k .
\label{eq:vreg_offgreedy}
\end{equation}
Greedy collection is the case $D^{(2)}_k=0$: if $\mathbf a_i=a^\star(S_i;V_k)$
then $G_k=\Tcal V_k$ by Lemma~\ref{lem:v_q_backup}(iii) and
\begin{equation}
\E\!\left[\|\widehat V_{k+1}-\Tcal V_k\|_{2,\rho_{k,S}}\,\middle|\,\mathcal H_k\right]
\le C_{\mathrm{Vreg},H}(\log n)^{(1+2\xi^\star)/2}m_n^{(\alpha^\star-1)/2}+\sqrt{2\zeta_n}.
\label{eq:vreg_derived}
\end{equation}
The required closure is precisely the \emph{Bellman} condition
$\Tcal V\in\Gclk$ with uniform compositional H\"older radii from
Assumption~\ref{ass:relu}.  The policy-image discrepancy of $G_k$ remains in
the bound: choosing the comparator in $L^2(\rho_{k,S})$ produces exactly
$D^{(2)}_k$, which Proposition~\ref{prop:l2_routing} bounds by the named
residuals.
\end{proposition}

\noindent\emph{Proof:} See Appendix~\ref{app:proofs}.

\begin{proposition}[Bounding the regression defect in $L^2$\obj{\Rabs}]\label{prop:l2_routing}
For every block $k$, the defect in~\eqref{eq:l2_defect} satisfies
\begin{align}
D^{(2)}_k
&\le\eps_{\mathrm{ker},k}+\eps_{\mathrm{tgt},k}+\epsbuf
+\epsact+\eps_k\Dbar
\label{eq:closure_defect_bound}
\\[-0.2ex]
&\le\eps_{\mathrm{ker},k}+\eps_{\mathrm{tgt},k}+\epsbuf
+(1-\eps_k)2\Lambda_k+\eps_k\Dbar .
\label{eq:closure_defect_linear}
\end{align}
If the global frozen-iterate margin holds, or its local form holds with
$2\Lambda_k\in[u_0,\bar u]$, then the sharper bound
\begin{equation}
D^{(2)}_k\;\le\;\eps_{\mathrm{ker},k}+\eps_{\mathrm{tgt},k}+\epsbuf
+(1-\eps_k)\min\bigl\{2\Lambda_k,\sqrt{C_{\mathrm{marg}}}(2\Lambda_k)^{1+\alpha/2}\bigr\}
+\eps_k\Dbar
\label{eq:closure_defect_margin}
\end{equation}
also holds.
Under Proposition~\ref{prop:reference_margin_transfer}, the same display holds
with its action term replaced by
\begin{equation*}
(1-\eps_k)\min\!\left\{2\Lambda_k,
2\Lambda_k\bigl[\tau_\star+C_\star
(2\Lambda_k+2\delta_k^\star)^{\alpha_\star}\bigr]^{1/2}\right\}.
\end{equation*}

The norm choice is essential.  Consider the favorable no-mismatch case
$\eps_k=0$, $\beta_k^{\mathrm{rep}}=\pi_k^{\mathrm{on}}$, and
$\eps_{\mathrm{ker},k}=\eps_{\mathrm{tgt},k}=0$, so
$G_k=\Tcal^{\pi_k^{\mathrm{on}}}V_k$.  Suppose the selected action switches at
$\stilde_0$ between $a$ and $a'$, the two functions
$Q_{V_k}(\cdot,a),Q_{V_k}(\cdot,a')$ are continuous there, and every
neighborhood of $\stilde_0$ has positive $\rho_{k,S}$-mass in both selection
regions.  If
$J_k:=|Q_{V_k}(\stilde_0,a)-Q_{V_k}(\stilde_0,a')|>0$, write
\[
\operatorname{dist}_{\infty,\rho}(\mathcal C,g)
:=\inf_{f\in\mathcal C}\operatorname*{ess\,sup}_{\rho}|f-g|.
\]
Then every class
$\mathcal C$ whose members are continuous at $\stilde_0$ satisfies
\begin{equation}
\operatorname{dist}_{\infty,\rho_{k,S}}(\mathcal C,G_k)\ge\tfrac12J_k.
\label{eq:supnorm_obstruction}
\end{equation}
The construction of Proposition~\ref{prop:act_tight} realizes
$J_k=2\Lambda_k$; hence a continuous supremum-norm comparator remains linear
in $\Lambda_k$ for every margin exponent.
\end{proposition}

\begin{proof}
$G_k-\Tcal V_k$ is the tail of the telescope of
Lemma~\ref{lem:composition} from $G_k$ to $\Tcal V_k$, so the triangle
inequality in the common space $L^2(\rho_{k,S})$ along those links gives
\eqref{eq:closure_defect_bound}.  Bound~\eqref{eq:act_upper_linear} at $r=2$, with
Corollary~\ref{cor:act_discharge} for a period-level mixture, then gives
\eqref{eq:closure_defect_linear}.  Under the stated margin hypotheses,
Bound~\eqref{eq:act_upper} at $r=2$ gives the margin term; taking the smaller of
the two valid action bounds proves~\eqref{eq:closure_defect_margin}.  The same
argument using~\eqref{eq:act_upper_optimal_margin} proves the reference-gap
version.
For~\eqref{eq:supnorm_obstruction}, $G_k$ has one-sided essential limits
$Q_{V_k}(\stilde_0,a)$ and $Q_{V_k}(\stilde_0,a')$ along the two selection
regions.  A function continuous at $\stilde_0$ has one limit and therefore
cannot lie within less than half their separation of both; the positive-mass
condition turns this into an essential-supremum bound.  At the switching point
in Proposition~\ref{prop:act_tight}, that separation is $2\geff\delta=2\Lambda_k$.
\end{proof}

\begin{corollary}[Tabular generative-reset bound with known envelopes\obj{\Aiid}]\label{cor:tabular_discharge}
Assume $H\ge2$.  Let $\Stildeo$ be finite, $N:=|\Stildeo|$, and replace
Assumption~\ref{ass:relu} by the fixed level-wise clipped tabular class
$\mathcal F^{\mathrm{tab}}:=\mathcal V^{\mathrm{clip}}$.  This class is exactly
Bellman closed and admits a measurable exact coordinatewise ERM: the label
average at a visited state and zero at an unvisited state.  Condition on the
pre-block history and write $p_{k,s}:=\rho_{k,S}\{s\}$,
$g_{k,s}:=\E[Y\mid S=s,\mathcal H_k]$,
$\sigma_{k,s}^2:=\operatorname{Var}(Y\mid S=s,\mathcal H_k)$, and $N_{k,s}$ for the number
of visits to $s$ in the $n_k$-sample block.  Then, with zero-mass coordinates
understood to contribute zero, its exact conditional squared risk is
\begin{equation}
\E\!\left[\|V_{k+1}-g_k\|_{2,\rho_{k,S}}^2\mid\mathcal H_k\right]
=\sum_{s\in\Stildeo}p_{k,s}\left[
 \sigma_{k,s}^2\E\!\left\{\frac{\mathbf1\{N_{k,s}>0\}}{N_{k,s}}\,\middle|\,\mathcal H_k\right\}
 +g_{k,s}^2(1-p_{k,s})^{n_k}\right].
\label{eq:tabular_rate}
\end{equation}
Since $|Y|\le\Vmax$, Jensen and the binomial count calculation in the proof
give the log-free deterministic fit envelope
\begin{align}
\mathrm{stat}^{\mathrm{tab}}_k
&:=\Vmax\sqrt{\frac{N\{2+(n_k/(n_k+1))^{n_k}\}}{n_k+1}}
\le \Vmax\sqrt{\frac{5N}{2(n_k+1)}},
\label{eq:tabular_rate_expected}\\[-2pt]
n_k\ge2&:\qquad
\mathrm{stat}^{\mathrm{tab}}_k\le\Vmax\sqrt{\frac{22N}{9(n_k+1)}}.
\notag
\end{align}
Suppose further that a fixed full-support reset law $\varsigma$ is also the
state design law $\rho_{k,S}$; the fresh collection and replay action laws
coincide; slot~\textnormal{(S6)} uses $(\textsc{sample},\textsc{obs})$;
$O=S$; and $W_k=V_k$.  Then
\begin{equation}
\eps_{\mathrm{ker},k}=\eps_{\mathrm{tgt},k}=\epsbuf
=\eps_{\mathrm{alias},k}=\dnet=0,
\qquad \Lambda_k=\etascorek,
\label{eq:tabular_discharge}
\end{equation}
and, for every $K\ge H-1$,
\begin{equation}
\E\|\Vstar-V^{\pi_K}\|_{1,\mu_S}
\le2\phi_s^{(H)}
\max_{K-H<k<K}\!\left[
\mathrm{stat}^{\mathrm{tab}}_k+A_{\eta,k}+\eps_k\Dbar\right],
\label{eq:fully_discharged}
\end{equation}
where $A_{\eta,k}$ is the chosen envelope $\epsactp$:
$A_{\eta,k}:=(1-\eps_k)2\etascorek$ without a margin, while under the
frozen-iterate margin at
scale $2\etascorek$ it may be replaced by
$(1-\eps_k)\min\{2\etascorek,
C_{\mathrm{marg}}^{1/p}(2\etascorek)^{1+\alpha/p}\}$.
Thus every residual slot is zero or explicitly bounded; if $\etascorek\to0$,
$n_k\to\infty$, and $\eps_k\to0$, the right side tends to zero.

There is also a finite-confidence form whose tabular regression term has no
additional minimum-state-mass factor; coverage in the policy bound still
enters through $\phi_s^{(H)}$.  For $\delta\in(0,1)$ and fixed
$K\ge H-1$, put
\begin{equation}
\mathrm{stat}^{\mathrm{tab,hp}}_{k,K}(\delta):=
\Vmax\sqrt{\frac{C_1\{2+N\log(2n_k+1)+\log((H-1)/\delta)\}}{n_k}}.
\label{eq:tabular_rate_hp}
\end{equation}
Then, with probability at least $1-\delta$ over the adaptive fresh blocks,
\begin{equation}
\|\Vstar-V^{\pi_K}\|_{1,\mu_S}
\le2\phi_s^{(H)}
\max_{K-H<k<K}\!\left[
\mathrm{stat}^{\mathrm{tab,hp}}_{k,K}(\delta)
+A_{\eta,k}+\eps_k\Dbar\right].
\label{eq:fully_discharged_hp}
\end{equation}
\end{corollary}

\noindent\emph{Proof:} See Appendix~\ref{app:proofs}.

\section{Deployment and consistency}\label{sec:discussion}

The main bound evaluates a true-$Q$ greedy policy, whereas the deployed
controller uses the implemented final score.  Theorem~\ref{thm:deployed_transfer}
bounds the resulting difference, and Corollary~\ref{cor:deployed_convergence}
gives consistency when the residuals decay.

\subsection{Implemented-score controller}\label{subsec:deployment}

\begin{theorem}[Survival-aware transfer to the implemented-score controller\obj{\Rabs}]\label{thm:deployed_transfer}
Assume the hypotheses of Theorem~\ref{thm:hmpdrl_end_to_end}, let
$\widehat\pi_K=\widehat a^\star(\cdot;V_K)$ be the deployed policy from~\eqref{eq:greedy_action}, stationary on the clock-augmented state (and generally
nonstationary on the physical state alone), and let $\etascoreK$ be a deterministic upper bound on
the score error at the final frozen network,
$\|\widehat Q_{V_K}-Q_{V_K}\|_\infty\le\etascoreK$ almost surely.  Suppose
also that deterministic $\bar q_\ell\in[0,1]$ satisfy, almost surely,
\begin{equation}
\mu_S^\circ(\Pkernel^{\widehat\pi_K}_\circ)^\ell\mathbf1
\le\bar q_\ell,\qquad 0\le\ell<H;
\label{eq:deployment_survival}
\end{equation}
the universal choice is $\bar q_\ell=1$.  Then
\begin{equation}
\E\!\left[\|\Vstar-V^{\widehat\pi_K}\|_{1,\mu_S}\right]
\le \mathcal B_K
+\sum_{k=0}^{K-1}w^{(H)}_{K,k}e_{k,p}^{\mathrm{Bell}}
+2\etascoreK\sum_{\ell=0}^{H-1}\geff^\ell\bar q_\ell,
\label{eq:deployed_transfer}
\end{equation}
that is, decomposition~\eqref{eq:hmpdrl_end_to_end} holds for the deployed
controller with a single survival-weighted score term.  It is never worse than
$2\etascoreK\min\{H,(1-\geff)^{-1}\}$, and no margin assumption is used.
\end{theorem}

\begin{proof}
Let $d_K:=\Tcal V_K-\Tcal^{\widehat\pi_K}V_K$.  The fixed Borel rule makes
$\widehat\pi_K$ measurable, deterministic and stationary, and
Lemma~\ref{lem:act_identity}(iii) gives pointwise
\[
0\le d_K(\stilde)=Q_{V_K}(\stilde,a^\star(\stilde;V_K))
-Q_{V_K}(\stilde,\widehat\pi_K(\stilde))\le2\etascoreK .
\]
Repeating Lemma~\ref{lem:varying_propagation} changes only its nonnegative
Singh--Yee loss-resolvent step~\cite{singh1994upper}, adding after integration
\[
\sum_{\ell=0}^{H-1}\geff^\ell
\int d_K\,d\{\mu_S^\circ
(\Pkernel^{\widehat\pi_K}_\circ)^\ell\}
\le2\etascoreK\sum_{\ell=0}^{H-1}\geff^\ell\bar q_\ell.
\]
Clock nilpotence gives the finite sum.  All other residual branches and the
coverage step are unchanged because $\widehat\pi_K$ is a measurable
deterministic policy, proving~\eqref{eq:deployed_transfer}.
\end{proof}

\begin{corollary}[Final-law coverage and margin for deployment\obj{\Rabs}]
\label{cor:deployed_transfer_margin}
In the setting of Theorem~\ref{thm:deployed_transfer}, set $\Gamma_K=0$ when
$\etascoreK=0$, in which case no gap condition is needed.  When
$\etascoreK>0$, suppose the final true gap either satisfies the global
frozen-iterate margin~\eqref{eq:margin} under $\rho_{K,S}$ almost surely (or
its local form at $2\etascoreK$), or satisfies the fixed-$Q^\star$ transfer
conditions of Proposition~\ref{prop:reference_margin_transfer} at $k=K$ and
$u=2\etascoreK$.  Set $\Gamma_K$ by the corresponding case:
\begin{equation}
\Gamma_K:=
\begin{cases}
C_{\mathrm{marg}}^{1/p}(2\etascoreK)^{1+\alpha/p},
&\text{under the frozen-iterate margin},\\
2\etascoreK\bigl\{\tau_\star+C_\star
(2\etascoreK+2\delta_K^\star)^{\alpha_\star}\bigr\}^{1/p},
&\text{under the fixed-$Q^\star$ transfer}.
\end{cases}
\label{eq:final_gap_envelope}
\end{equation}
Assume also that, for the same
$s\in[2,\infty]$ and $p=s/(s-1)$ as in
Assumption~\ref{ass:concentrability}, deterministic envelopes
$d_s^{\mathrm{fin}}(\ell)<\infty$ satisfy
\begin{equation}
\sup_{\pi_{1:\ell}}
\left\|\frac{d\{\mu_S^\circ\Pkernel^{\pi_1}_\circ\cdots
\Pkernel^{\pi_\ell}_\circ\}}{d\rho_{K,S}}\right\|_{s,\rho_{K,S}}
\le d_s^{\mathrm{fin}}(\ell),\qquad 0\le\ell<H,
\label{eq:final_concentrability}
\end{equation}
almost surely, with the empty product at $\ell=0$.  Then
\begin{equation}
\E\!\left[\|\Vstar-V^{\widehat\pi_K}\|_{1,\mu_S}\right]
\le \mathcal B_K
+\sum_{k=0}^{K-1}w^{(H)}_{K,k}e_{k,p}^{\mathrm{Bell}}
+\Gamma_K\sum_{\ell=0}^{H-1}\geff^\ell d_s^{\mathrm{fin}}(\ell).
\label{eq:deployed_transfer_margin}
\end{equation}
If both~\eqref{eq:deployment_survival} and~\eqref{eq:final_concentrability}
hold, the final term can be sharpened depthwise to
\begin{equation}
\sum_{\ell=0}^{H-1}\geff^\ell\min\!\left\{
2\etascoreK\bar q_\ell,
\Gamma_Kd_s^{\mathrm{fin}}(\ell)\right\}.
\label{eq:deployed_transfer_combined}
\end{equation}
\end{corollary}

\begin{proof}
Condition on the final history and put
$d_K:=\Tcal V_K-\Tcal^{\widehat\pi_K}V_K$.  The score comparison gives
$0\le d_K\le2\etascoreK$ and
$\{d_K>0\}\subseteq\{\Delta_Q^{(K)}\le2\etascoreK\}$.  Hence
$\|d_K\|_{p,\rho_{K,S}}\le\Gamma_K$: use~\eqref{eq:margin} on the
support event for the first route and
\eqref{eq:optimal_to_iterate_margin} for the second.
The exact extra resolvent contribution is
$\sum_{\ell<H}\geff^\ell\int d_K\,
d\{\mu_S^\circ(\Pkernel^{\widehat\pi_K}_\circ)^\ell\}$.
H\"older and~\eqref{eq:final_concentrability} prove
\eqref{eq:deployed_transfer_margin}.  At each depth the same integral is also
at most $2\etascoreK\bar q_\ell$; taking the smaller bound before summing gives
\eqref{eq:deployed_transfer_combined}.  There is no additional factor two:
the resolvent inserts the single defect $d_K$.
\end{proof}

\subsection{Consistency of the abstract recursion}\label{subsec:convergence}
\leavevmode\par

\begin{theorem}[Nonasymptotic generative-reset policy-loss bound under uniform closure\obj{\Aiid}]\label{thm:nonasymptotic}
Assume the hypotheses of Theorem~\ref{thm:hmpdrl_end_to_end} and
Assumption~\ref{ass:relu}, specializing
Assumption~\ref{ass:concentrability} to $s=2$ (and hence $p=2$), with
$H\ge2$.  In every block, additionally use the fresh-block
protocol~\eqref{eq:generative_variant}: conditional on $\mathcal H_k$, draw $n_k\ge2H$
i.i.d.\ true-kernel outcomes with the fixed state law
$\rho_{k,S}\equiv\varsigma$ and action law
$\beta_k$, rebuild $Y_i=r_i+\geff(1-d_i)V_k(\stilde'_i)$, and take a measurable
$\zeta_{n_k}$-approximate ERM $V_{k+1}\in\mathcal F^{V,0}_{n_k}$ with
$V_k\in\mathcal F^V_{n_k}$.  These labels are conditionally unbiased for
$G_k$, and
$\eps_{\mathrm{ker},k}=\eps_{\mathrm{tgt},k}=0$.
Write $e_{k,2}^{\mathrm{Bell}}:=e_{k,p}^{\mathrm{Bell}}|_{p=2}$ and
$\mathrm{stat}_k:=C_{\mathrm{Vreg},H}(\log
n_k)^{\frac{1+2\xi^\star}{2}}m_{n_k}^{\frac{\alpha^\star-1}{2}}+\sqrt{2\zeta_{n_k}}$,
where $m_{n_k}=\lfloor n_k/H\rfloor$.
Then the fit envelope in~\eqref{eq:fit_residual} may be chosen from
Propositions~\ref{prop:vreg_rate} and~\ref{prop:l2_routing}, so for every $k$
the Bellman-residual envelope may be chosen to satisfy
\begin{equation}
e_{k,2}^{\mathrm{Bell}}\;\le\;\mathrm{stat}_k
+(1+\sqrt2)\bigl(\epsact+\epsbuf+\eps_k\Dbar\bigr),
\label{eq:rk_l2}
\end{equation}
and hence, for every finite $K\ge1$,
\begin{equation}
\E\|\Vstar-V^{\pi_K}\|_{1,\mu_S}
\;\le\;\mathcal B_K
+2\phi_{2,K}^{(H)}\!\!\max_{\max\{0,K-H+1\}\le k<K}\!\! e_{k,2,\mathrm{rhs}}^{\mathrm{Bell}},
\label{eq:nonasymptotic}
\end{equation}
$e_{k,2,\mathrm{rhs}}^{\mathrm{Bell}}$ denoting the right-hand side of~\eqref{eq:rk_l2}.
Under the global frozen-iterate $(C_{\mathrm{marg}},\alpha)$ margin, or its
local form with $2\Lambda_k\in[u_0,\bar u]$, $\epsact$ may be replaced by
$(1-\eps_k)\min\{2\Lambda_k,
\sqrt{C_{\mathrm{marg}}}(2\Lambda_k)^{1+\alpha/2}\}$, so the exponent reaches the propagated residual
bound at the cost of the constant $1+\sqrt2$.  The fixed-$Q^\star$ route uses
the $r=2$ envelope in~\eqref{eq:act_upper_optimal_margin} and preserves
the same propagated exponent when $\tau_\star=0$ and
$\delta_k^\star=O(\Lambda_k)$ on the active window.
\end{theorem}

\noindent\emph{Proof:} See Appendix~\ref{app:proofs}.

\begin{corollary}[Expected policy-loss consistency with an exact-score oracle\obj{\Aiid}]\label{cor:fresh_consistency}
Assume Theorem~\ref{thm:nonasymptotic}, fixed $H$, and
$\phi^{(H)}_{\mu_S,\rho}<\infty$.  Initialize
$V_0\in\mathcal F^V_{n_0}$; use a nondecreasing integer schedule
$n_k\ge2H$ with $n_k\to\infty$ and $\zeta_{n_k}\to0$; set $W_k=V_k$; use the
additional exact expectation oracle $\Qoracle_V=Q_V$ on $\mathcal F$ with
the common tie rule; draw fresh
$(\textsc{sample},\textsc{obs})$ blocks; and let $\eps_k\to0$.  Then, for
every $K\ge H-1$,
\begin{equation}
\E\|\Vstar-V^{\pi_K}\|_{1,\mu_S}
\le2\phi^{(H)}_{\mu_S,\rho}
\max_{K-H<k<K}\left[
\mathrm{stat}_k+2(1+\sqrt2)\Vmax\eps_k\right]
\longrightarrow0.
\label{eq:fresh_consistency}
\end{equation}
If the same oracle is used at deployment, its final greedy policy equals
$\pi_K$, so the same conclusion holds for that deployed oracle policy.
\end{corollary}

\begin{proof}
By nesting, $V_{k+1}\in\mathcal F^{V,0}_{n_k}\subseteq\mathcal F^V_{n_k}
\subseteq\mathcal F^V_{n_{k+1}}$, so the initialization closes the iterate
membership required by Theorem~\ref{thm:nonasymptotic}.  Fresh observed
blocks give $\eps_{\mathrm{ker},k}=\eps_{\mathrm{tgt},k}=\epsbuf=0$, while
synchronization and the exact score give $\Lambda_k=\epsact=0$; take the
safe deterministic exploration envelope $\Dbar=2\Vmax$.  Substitution
in~\eqref{eq:nonasymptotic} proves the displayed bound, whose right-hand side
vanishes because $\alpha^\star<1$, fixed $H$ and $n_k\to\infty$ imply
$m_{n_k}=\lfloor n_k/H\rfloor\to\infty$,
$\zeta_{n_k}\to0$, and $\eps_k\to0$; oracle final scoring and the common tie
rule identify its deployed policy with $\pi_K$.
\end{proof}

\begin{corollary}[Known deterministic-model consistency\obj{\Aiid}]
\label{cor:deterministic_consistency}
Assume all hypotheses of Corollary~\ref{cor:fresh_consistency} except for its
separate exact-score oracle.  Suppose that the joint kernel is deterministic,
$(r,\stilde')=(R(\stilde,a),F(\stilde,a))$, and that
the controller has the exact maps $R,F$.  Then the point score
$\widehat Q_V(\stilde,a)=R(\stilde,a)+\geff V(F(\stilde,a))$ equals
$\Qoracle_V=Q_V$
without a separate expectation oracle, so the same policy-loss consistency
conclusion holds.
\end{corollary}

\begin{proof}
The deterministic kernel makes the Bellman integral equal evaluation at
$F(\stilde,a)$; hence $\etascorek=\etascoreK=0$, and
Corollary~\ref{cor:fresh_consistency} applies verbatim.
\end{proof}

\begin{corollary}[Deployed-policy convergence and upper neighborhood\obj{\Rabs,\Aiid}]\label{cor:deployed_convergence}
Let $\widehat\pi_K$ and $\etascoreK$ be as in
Theorem~\ref{thm:deployed_transfer}, and fix $H$.  \emph{(a)} \obj{\Rabs}
$e_{k,p}^{\mathrm{Bell}}\to0$ implies
$\E\|\Vstar-V^{\pi_K}\|_{1,\mu_S}\to0$; if also
$\etascoreK\to0$, then
$\E\|\Vstar-V^{\widehat\pi_K}\|_{1,\mu_S}\to0$.
\emph{(b)} \obj{\Aiid} Assume Theorem~\ref{thm:nonasymptotic} and either the
global frozen-iterate margin or a local frozen-iterate margin whose interval contains $2\Lambda_k$
for every sufficiently large $k$.  Suppose
$n_k\to\infty$, $\zeta_{n_k}\to0$, $\eps_{\mathrm{buf},k},\eps_k,\dnet\to0$,
while eventually $\etascorek\le\etascore$ and
$\etascoreK\le\etascore$.  Then
\begin{align}
\limsup_{K\to\infty}\;
\E\!\left[\|\Vstar-V^{\pi_K}\|_{1,\mu_S}\right]
&\le R_{\mathrm{sc}},\qquad
R_{\mathrm{sc}}:=2(1+\sqrt2)\,\phi^{(H)}_{\mu_S,\rho}
\sqrt{C_{\mathrm{marg}}}\,(2\etascore)^{1+\alpha/2},
\label{eq:score_neighborhood}\\
\limsup_{K\to\infty}\;
\E\!\left[\|\Vstar-V^{\widehat\pi_K}\|_{1,\mu_S}\right]
&\le R_{\mathrm{sc}}
+2\etascore\sum_{\ell=0}^{H-1}\geff^\ell\bar q_\ell.
\label{eq:deployed_neighborhood}
\end{align}
Without a margin the same two bounds hold with
$R_{\mathrm{sc}}:=4(1+\sqrt2)\phi^{(H)}_{\mu_S,\rho}\etascore$.
In particular, if $\etascorek\to0$ and $\etascoreK\to0$, then both the
true-score and deployed-policy losses converge to zero.
\end{corollary}

\begin{proof}
In Theorem~\ref{thm:hmpdrl_end_to_end}, the boundary vanishes for $K\ge H-1$
and only the last $H-1$ Bellman residuals have nonzero weights, proving the
first claim in~(a); Theorem~\ref{thm:deployed_transfer} proves the second.
For~(b), Proposition~\ref{prop:l2_routing} and~\eqref{eq:act_upper} at $r=2$
give the
true-score radius $R_{\mathrm{sc}}$ after all statistical and nonaction
residuals vanish; the survival term of Theorem~\ref{thm:deployed_transfer}
gives the deployed radius.  Without a margin, use
\eqref{eq:act_upper_linear} at $r=2$ in the same calculation.
If $\etascorek\to0$, that linear bound and the remaining hypotheses give
$e_{k,2}^{\mathrm{Bell}}\to0$; with $\etascoreK\to0$, part~(a) then yields the
last assertion.
\end{proof}

\subsection{Convergence conditions for implementations}\label{subsec:limitations}

To apply the policy-loss bounds to $\Arun$, we bound its six residuals and
the final score error, as indicated in~\eqref{eq:object_contract}.
For observation-based navigation,
Proposition~\ref{prop:representation_routes} permits either a
Markov-sufficient simulator or belief state with an exposed clock, or a
measurable compression whose visible-target aliasing and score errors are
controlled by the corresponding residuals.

\paragraph{Sufficient conditions for FIFO/interleaved SGD.}
For fixed $H$, the same policy-loss theory applies to a concrete
FIFO/interleaved-SGD implementation when: (i) its state satisfies the
standard-Borel controlled Markov condition or its compression residuals are
controlled; (ii) finite concentrability constants hold uniformly over its
induced replay laws; and (iii)
\begin{equation}
\max_{K-H<k<K}\!\left(
\eps_{\mathrm{fit},k}+\eps_{\mathrm{ker},k}+\eps_{\mathrm{tgt},k}
+\epsbuf+\epsactp+\eps_k\Dbar\right)\longrightarrow0,
\qquad \etascoreK\longrightarrow0.
\label{eq:future_convergence_contract}
\end{equation}
The unconditional linear bound gives action-residual decay when $\dnet\to0$
and $\etascorek\to0$, while the margin condition upgrades this decay to the
sharp exponent.  For stored prediction labels, $\eps_{\mathrm{tgt},k}$ must
also cover the predictor staleness in~\eqref{eq:tgt_predictor_staleness};
current score accuracy and target-network staleness alone do not control it.
Setting $\eps_{\mathrm{ker},k}=0$ for stored labels requires a justification
such as the metadata-conditioned law~\eqref{eq:conditional_kernel_retention}.
Bounds for a particular replay scheme, optimizer, score estimator, or state
representation enter the propagation and deployment theorems through the
corresponding residuals.  Their decay gives convergence through
\eqref{eq:future_convergence_contract}.

\subsection{Related work}\label{subsec:related}

\paragraph{Fitted value iteration.}
Munos and Szepesv\'ari analyze fitted value iteration by estimating every
action's continuation and then maximizing~\cite{munos_szepesvari_2008}.  The
population response in our state-only reset construction instead averages the
executed action and carries a finite clock.  It also differs from the
action-indexed response used in fitted $Q$ regression
\cite{ernst2005tree,riedmiller2005neural}.  Once this operator distinction is
isolated, the propagation analysis follows the $L^p$ approximate-dynamic-
programming tradition~\cite{farahmand2010error,munos2007performance,scherrer2015ampi}.
As in that literature, stability depends on the sampling
and update structure~\cite{tsitsiklis1997analysis,baird1995residual}.

\paragraph{SARSA and Expected SARSA.}
SARSA provides the closest target semantics.  A current observed-successor
state-value target under fresh sampling averages the acting law, whereas Expected
SARSA updates $Q(s,a)$ and averages the next action~\cite{sutton_barto_2018}.
Convergence results for SARSA with function approximation require smooth
improvement, ergodicity, or linear approximation
\cite{perkins2002convergent,melo2008analysis,zou2019sarsa}.  Here the acting
law's same-state departure from its frozen-target greedy comparator remains a
separate residual.

\paragraph{Neural $Q$ and TD analyses.}
Neural fitted $Q$-iteration separates statistical and iterative error
\cite{fan2020theoretical_full}.  Its action-indexed response and the scalar
executed-action response studied here define distinct population operators.
Neural TD treats fixed-policy evaluation~\cite{cai2019neural}, while neural
$Q$-learning and DQN use different heads and sampling or update hypotheses
\cite{xu2020finite,zhang2023epsilon_dqn}.  Our analysis identifies replay,
scalar-target, and policy-drift residuals as the additional quantities for a
deep-$V$ implementation analysis.  DQN, Double $Q$ bias correction, and
residual-gradient TD address related algorithmic objects
\cite{mnih2015humanlevel,vanhasselt2016double,baird1995residual}.

\paragraph{Action gaps.}
Action-gap regularity can convert uniform value or score approximation into
superlinear greedy-policy bounds
\cite{farahmand2011action_gap,bellemare2016actiongap}.  The classical
fixed-optimal positive-gap law controls
$\{0<\Delta_Q^\star\le u\}$ and can retain mass on optimal ties.  The condition
used here controls $\{\Delta_Q^{(k)}\le u\}$, including ties, uniformly over
the random frozen iterates and replay laws.  Proposition
\ref{prop:reference_margin_transfer} connects the two: a fixed-$Q^\star$ margin
and an iterate score tube yield a shifted frozen-gap bound.  When the tie mass
is zero and the tube is on the action-error scale, the exponent is preserved.
The resulting Tsybakov-type bound~\cite{tsybakov2004optimal} has one-step
exponent $1+\alpha/p$ under conjugate $L^s/L^p$ coverage;
Proposition~\ref{prop:act_tight} attains each intermediate exponent.

\paragraph{Replay and dependent data.}
Antos et al.\ require a stationary exponentially mixing path for their
dependent fitted-policy analysis~\cite{antos2008learning}.  Replay-buffer,
$Q$-learning, and TD analyses impose different processes
\cite{dicastro2022replay,szlak2021experience,lim2024tdreplay}.  Their
process-specific techniques provide tools for bounding the named FIFO and SGD
residuals in~\eqref{eq:object_contract}.  In our notation, $\epsbuf$ measures a
same-state operator distance, while concentrability controls the relation
between replay and evaluation occupancy laws~\cite{kakade_langford}.

\section{Conclusion}\label{sec:conclusion}

We established a finite-horizon convergence framework that converts the six
deep-$V$-learning residuals into expected policy loss.  For fixed $H$, only the
last $H-1$ update blocks carry residual weight, together with an initialization
term for shorter runs.
For $K_H\ge H-1$, the near-unit one-state-per-level witness attains the optimal
shared-law propagation-coefficient order $\Theta(H^{5/2})$ at $s=2$;
bounded direct-level coefficients give $O(H^2)$.
A frozen-gap margin changes the action term to exponent
$1+\alpha(1-1/s)$, and the
one-step construction attains that exponent.  The fixed-$Q^\star$ transfer
keeps the tie mass visible.  Survival and final-law arguments extend the
bound to the policy selected by the implemented score.

The matched-budget result gives the unconstrained continuous optimum, the
constrained water-filling solution, and an integer schedule within a factor
$2^\nu$ of the latter.  On the near-unit one-state-per-level witness, root-$n$
rates give a clock term $H^3/\sqrt{\mathsf N}$ before regression constants;
under equal, sufficiently large terminal-window label budgets, direct tabular
reset improves the statistical bound by a factor of order $\sqrt H$ compared
with a shared $H$-state fit.  Bellman--H\"older closure and the
clock-routed ReLU class give the fixed-$H$ neural rate, while the tabular case
gives a log-free expected fit rate.  These results prove expected policy-loss
consistency for the exact-score generative-reset approximate-ERM procedure at
fixed $H$.  For FIFO/interleaved SGD,
\eqref{eq:future_convergence_contract} gives an explicit residual-decay criterion
linking replay, optimization, score-estimation, and representation rates to
full policy-loss convergence.

\bibliographystyle{unsrtnat}
\bibliography{bibliography}

@inproceedings{CADRL,
  author    = {Chen, Yu Fan and Liu, Miao and Everett, Michael and How, Jonathan P.},
  title     = {Decentralized Non-communicating Multiagent Collision Avoidance with Deep Reinforcement Learning},
  booktitle = {Proceedings of the IEEE International Conference on Robotics and Automation},
  pages     = {285--292},
  year      = {2017}
}

@inproceedings{SACADRL,
  author        = {Chen, Yu Fan and Everett, Michael and Liu, Miao and How, Jonathan P.},
  title         = {Socially Aware Motion Planning with Deep Reinforcement Learning},
  booktitle     = {Proceedings of the IEEE/RSJ International Conference on Intelligent Robots and Systems},
  year          = {2017},
  eprint        = {1703.08862},
  archivePrefix = {arXiv},
  primaryClass  = {cs.RO},
  url           = {https://arxiv.org/abs/1703.08862}
}

@inproceedings{SARL,
  author    = {Chen, Changan and Liu, Yuejiang and Kreiss, Sven and Alahi, Alexandre},
  title     = {Crowd-Robot Interaction: Crowd-Aware Robot Navigation with Attention-Based Deep Reinforcement Learning},
  booktitle = {Proceedings of the IEEE International Conference on Robotics and Automation},
  pages     = {6015--6022},
  year      = {2019}
}

@article{EBCADRL,
  title={Robot Navigation with Entity-Based Collision Avoidance using Deep Reinforcement Learning},
  author={Kolomeytsev, Yury and Golembiovsky, Dmitry},
  journal={arXiv preprint arXiv:2408.14183},
  year={2024},
  url={https://arxiv.org/abs/2408.14183}
}

@article{HMPDRL,
  title={Hybrid Motion Planning with Deep Reinforcement Learning for Mobile Robot Navigation},
  author={Kolomeytsev, Yury and Golembiovsky, Dmitry},
  journal={arXiv preprint arXiv:2512.24651},
  year={2025},
  url={https://arxiv.org/abs/2512.24651}
}

@article{antos2008learning,
  author  = {Antos, Andr{\'a}s and Szepesv{\'a}ri, Csaba and Munos, R{\'e}mi},
  title   = {Learning Near-Optimal Policies with Bellman-Residual Minimization Based Fitted Policy Iteration and a Single Sample Path},
  journal = {Machine Learning},
  volume  = {71},
  number  = {1},
  pages   = {89--129},
  year    = {2008}
}

@book{bertsekas1996neuro,
  author    = {Bertsekas, Dimitri P. and Tsitsiklis, John N.},
  title     = {Neuro-Dynamic Programming},
  publisher = {Athena Scientific},
  year      = {1996}
}

@book{bertsekas_shreve_1978,
  author    = {Bertsekas, Dimitri P. and Shreve, Steven E.},
  title     = {Stochastic Optimal Control: The Discrete-Time Case},
  publisher = {Academic Press},
  year      = {1978}
}

@book{hernandez_lerma_lasserre_1996,
  author    = {Hern{\'a}ndez-Lerma, On{\'e}simo and Lasserre, Jean B.},
  title     = {Discrete-Time {M}arkov Control Processes: Basic Optimality Criteria},
  publisher = {Springer},
  year      = {1996}
}

@inproceedings{cai2019neural,
  author    = {Cai, Qi and Yang, Zhuoran and Lee, Jason D. and Wang, Zhaoran},
  title     = {Neural Temporal-Difference Learning Converges to Global Optima},
  booktitle = {Advances in Neural Information Processing Systems},
  volume    = {32},
  year      = {2019}
}

@misc{fan2020theoretical_full,
  author        = {Fan, Jianqing and Wang, Zhaoran and Xie, Yuchen and Yang, Zhuoran},
  title         = {A Theoretical Analysis of Deep {Q}-Learning},
  year          = {2020},
  eprint        = {1901.00137},
  archivePrefix = {arXiv},
  primaryClass  = {cs.LG},
  url           = {https://arxiv.org/abs/1901.00137v3},
  note          = {arXiv:1901.00137v3}
}

@inproceedings{kakade_langford,
  author    = {Kakade, Sham M. and Langford, John},
  title     = {Approximately Optimal Approximate Reinforcement Learning},
  booktitle = {Proceedings of the Nineteenth International Conference on Machine Learning},
  pages     = {267--274},
  year      = {2002}
}

@book{kreyszig1989functional,
  author    = {Kreyszig, Erwin},
  title     = {Introductory Functional Analysis with Applications},
  publisher = {Wiley},
  year      = {1989}
}

@article{mnih2015humanlevel,
  author  = {Mnih, Volodymyr and others},
  title   = {Human-Level Control through Deep Reinforcement Learning},
  journal = {Nature},
  volume  = {518},
  number  = {7540},
  pages   = {529--533},
  year    = {2015},
  doi     = {10.1038/nature14236}
}

@article{munos_szepesvari_2008,
  author  = {Munos, R{\'e}mi and Szepesv{\'a}ri, Csaba},
  title   = {Finite-Time Bounds for Fitted Value Iteration},
  journal = {Journal of Machine Learning Research},
  volume  = {9},
  pages   = {815--857},
  year    = {2008}
}

@inproceedings{riedmiller2005neural,
  author    = {Riedmiller, Martin},
  title     = {Neural Fitted Q Iteration: First Experiences with a Data Efficient Neural Reinforcement Learning Method},
  booktitle = {Proceedings of the European Conference on Machine Learning},
  pages     = {317--328},
  year      = {2005},
  doi       = {10.1007/11564096_32}
}

@article{schmidt_hieber_2020,
  author  = {Schmidt-Hieber, Johannes},
  title   = {Nonparametric Regression Using Deep Neural Networks with ReLU Activation Function},
  journal = {The Annals of Statistics},
  volume  = {48},
  number  = {4},
  pages   = {1875--1897},
  year    = {2020},
  doi     = {10.1214/19-AOS1875}
}

@article{schmidt_hieber_vu_2024,
  author  = {Schmidt-Hieber, Johannes and Vu, Don},
  title   = {Correction to ``Nonparametric Regression Using Deep Neural Networks with ReLU Activation Function''},
  journal = {The Annals of Statistics},
  volume  = {52},
  number  = {1},
  pages   = {413--414},
  year    = {2024},
  doi     = {10.1214/24-AOS2351}
}

@article{singh1994upper,
  author  = {Singh, Satinder P. and Yee, Richard C.},
  title   = {An Upper Bound on the Loss from Approximate Optimal-Value Functions},
  journal = {Machine Learning},
  volume  = {16},
  number  = {3},
  pages   = {227--233},
  year    = {1994}
}

@book{sutton_barto_2018,
  author    = {Sutton, Richard S. and Barto, Andrew G.},
  title     = {Reinforcement Learning: An Introduction},
  edition   = {2},
  publisher = {MIT Press},
  year      = {2018}
}

@article{tsitsiklis1997analysis,
  author  = {Tsitsiklis, John N. and Van Roy, Benjamin},
  title   = {An Analysis of Temporal-Difference Learning with Function Approximation},
  journal = {IEEE Transactions on Automatic Control},
  volume  = {42},
  number  = {5},
  pages   = {674--690},
  year    = {1997},
  doi     = {10.1109/9.580874}
}

@inproceedings{xu2020finite,
  author    = {Xu, Pan and Gu, Quanquan},
  title     = {A Finite-Time Analysis of Q-Learning with Neural Network Function Approximation},
  booktitle = {Proceedings of the 37th International Conference on Machine Learning},
  series    = {Proceedings of Machine Learning Research},
  volume    = {119},
  pages     = {10555--10565},
  publisher = {PMLR},
  year      = {2020},
  url       = {https://proceedings.mlr.press/v119/xu20c.html}
}

@inproceedings{zhang2023epsilon_dqn,
  author    = {Zhang, Shuai and Li, Hongkang and Wang, Meng and Liu, Miao and Chen, Pin-Yu and Lu, Songtao and Liu, Sijia and Murugesan, Keerthiram and Chaudhury, Subhajit},
  title     = {On the Convergence and Sample Complexity Analysis of Deep Q-Networks with $\epsilon$-Greedy Exploration},
  booktitle = {Advances in Neural Information Processing Systems},
  volume    = {36},
  year      = {2023}
}

@inproceedings{dicastro2022replay,
  author    = {Di Castro Shashua, Shirli and Mannor, Shie and Di Castro, Dotan},
  title     = {Analysis of Stochastic Processes through Replay Buffers},
  booktitle = {Proceedings of the 39th International Conference on Machine Learning},
  series    = {Proceedings of Machine Learning Research},
  volume    = {162},
  pages     = {5039--5060},
  publisher = {PMLR},
  year      = {2022}
}

@misc{szlak2021experience,
  author        = {Szlak, Liran and Shamir, Ohad},
  title         = {Convergence Results for Q-Learning with Experience Replay},
  year          = {2021},
  eprint        = {2112.04213},
  archivePrefix = {arXiv},
  primaryClass  = {cs.LG},
  url           = {https://arxiv.org/abs/2112.04213}
}

@article{lim2024tdreplay,
  author  = {Lim, Han-Dong and Lee, Donghwan},
  title   = {Finite-Time Analysis of Temporal Difference Learning with Experience Replay},
  journal = {Transactions on Machine Learning Research},
  year    = {2024},
  url     = {https://openreview.net/forum?id=A5ulGfDBON}
}

@inproceedings{farahmand2010error,
  author    = {Farahmand, Amir-massoud and Munos, R{\'e}mi and Szepesv{\'a}ri, Csaba},
  title     = {Error Propagation for Approximate Policy and Value Iteration},
  booktitle = {Advances in Neural Information Processing Systems},
  volume    = {23},
  pages     = {568--576},
  year      = {2010}
}

@inproceedings{chen2019information,
  author    = {Chen, Jinglin and Jiang, Nan},
  title     = {Information-Theoretic Considerations in Batch Reinforcement Learning},
  booktitle = {Proceedings of the 36th International Conference on Machine Learning},
  series    = {Proceedings of Machine Learning Research},
  volume    = {97},
  pages     = {1042--1051},
  publisher = {PMLR},
  year      = {2019},
  url       = {https://proceedings.mlr.press/v97/chen19e.html}
}

@book{gyorfi2002distribution,
  author    = {Gy{\"o}rfi, L{\'a}szl{\'o} and Kohler, Michael and Krzy{\.z}ak, Adam and Walk, Harro},
  title     = {A Distribution-Free Theory of Nonparametric Regression},
  publisher = {Springer},
  series    = {Springer Series in Statistics},
  year      = {2002}
}

@inproceedings{farahmand2011action_gap,
  author    = {Farahmand, Amir-massoud},
  title     = {Action-Gap Phenomenon in Reinforcement Learning},
  booktitle = {Advances in Neural Information Processing Systems},
  volume    = {24},
  pages     = {172--180},
  year      = {2011}
}

@article{tsybakov2004optimal,
  author  = {Tsybakov, Alexandre B.},
  title   = {Optimal Aggregation of Classifiers in Statistical Learning},
  journal = {The Annals of Statistics},
  volume  = {32},
  number  = {1},
  pages   = {135--166},
  year    = {2004},
  doi     = {10.1214/aos/1079120131}
}

@inproceedings{baird1995residual,
  author    = {Baird, Leemon},
  title     = {Residual Algorithms: Reinforcement Learning with Function Approximation},
  booktitle = {Proceedings of the Twelfth International Conference on Machine Learning},
  pages     = {30--37},
  publisher = {Morgan Kaufmann},
  year      = {1995}
}

@article{mammen1999smooth,
  author  = {Mammen, Enno and Tsybakov, Alexandre B.},
  title   = {Smooth Discrimination Analysis},
  journal = {The Annals of Statistics},
  volume  = {27},
  number  = {6},
  pages   = {1808--1829},
  year    = {1999},
  doi     = {10.1214/aos/1017939240}
}

@inproceedings{melo2008analysis,
  author    = {Melo, Francisco S. and Meyn, Sean P. and Ribeiro, M. Isabel},
  title     = {An Analysis of Reinforcement Learning with Function Approximation},
  booktitle = {Proceedings of the 25th International Conference on Machine Learning},
  pages     = {664--671},
  publisher = {ACM},
  year      = {2008},
  doi       = {10.1145/1390156.1390240}
}

@inproceedings{perkins2002convergent,
  author    = {Perkins, Theodore J. and Precup, Doina},
  title     = {A Convergent Form of Approximate Policy Iteration},
  booktitle = {Advances in Neural Information Processing Systems},
  volume    = {15},
  pages     = {1627--1634},
  year      = {2002}
}

@inproceedings{wang2021statistical,
  author        = {Wang, Ruosong and Foster, Dean P. and Kakade, Sham M.},
  title         = {What Are the Statistical Limits of Offline {RL} with Linear Function Approximation?},
  booktitle     = {International Conference on Learning Representations},
  year          = {2021},
  eprint        = {2010.11895},
  archivePrefix = {arXiv},
  primaryClass  = {cs.LG},
  url           = {https://arxiv.org/abs/2010.11895}
}

@inproceedings{zanette2021exponential,
  author    = {Zanette, Andrea},
  title     = {Exponential Lower Bounds for Batch Reinforcement Learning: Batch {RL} Can Be Exponentially Harder Than Online {RL}},
  booktitle = {Proceedings of the 38th International Conference on Machine Learning},
  series    = {Proceedings of Machine Learning Research},
  volume    = {139},
  pages     = {12287--12297},
  publisher = {PMLR},
  year      = {2021}
}

@inproceedings{zou2019sarsa,
  author    = {Zou, Shaofeng and Xu, Tengyu and Liang, Yingbin},
  title     = {Finite-Sample Analysis for {SARSA} with Linear Function Approximation},
  booktitle = {Advances in Neural Information Processing Systems},
  volume    = {32},
  year      = {2019}
}

@inproceedings{vanhasselt2016double,
  author    = {van Hasselt, Hado and Guez, Arthur and Silver, David},
  title     = {Deep Reinforcement Learning with Double {Q}-Learning},
  booktitle = {Proceedings of the Thirtieth AAAI Conference on Artificial Intelligence},
  pages     = {2094--2100},
  year      = {2016},
  url       = {https://arxiv.org/abs/1509.06461}
}

@article{munos2007performance,
  author  = {Munos, R{\'e}mi},
  title   = {Performance Bounds in $L_p$-Norm for Approximate Value Iteration},
  journal = {SIAM Journal on Control and Optimization},
  volume  = {46},
  number  = {2},
  pages   = {541--561},
  year    = {2007},
  doi     = {10.1137/040614384}
}

@article{scherrer2015ampi,
  author  = {Scherrer, Bruno and Ghavamzadeh, Mohammad and Gabillon, Victor and Lesner, Boris and Geist, Matthieu},
  title   = {Approximate Modified Policy Iteration and its Application to the Game of {T}etris},
  journal = {Journal of Machine Learning Research},
  volume  = {16},
  pages   = {1629--1676},
  year    = {2015},
  url     = {https://jmlr.org/papers/v16/scherrer15a.html}
}

@inproceedings{bellemare2016actiongap,
  author    = {Bellemare, Marc G. and Ostrovski, Georg and Guez, Arthur and Thomas, Philip S. and Munos, R{\'e}mi},
  title     = {Increasing the Action Gap: New Operators for Reinforcement Learning},
  booktitle = {Proceedings of the Thirtieth AAAI Conference on Artificial Intelligence},
  pages     = {1476--1483},
  year      = {2016}
}

@article{ernst2005tree,
  author  = {Ernst, Damien and Geurts, Pierre and Wehenkel, Louis},
  title   = {Tree-Based Batch Mode Reinforcement Learning},
  journal = {Journal of Machine Learning Research},
  volume  = {6},
  pages   = {503--556},
  year    = {2005},
  url     = {https://jmlr.org/papers/v6/ernst05a.html}
}

@inproceedings{xie2021realizability,
  author    = {Xie, Tengyang and Jiang, Nan},
  title     = {Batch Value-Function Approximation with Only Realizability},
  booktitle = {Proceedings of the 38th International Conference on Machine Learning},
  series    = {Proceedings of Machine Learning Research},
  volume    = {139},
  pages     = {11404--11413},
  publisher = {PMLR},
  year      = {2021}
}

\appendix

\section{Deferred proofs}\label{app:proofs}

\paragraph{Structural and target calculations.}
\begin{proof}[Proof of Lemma~\ref{lem:clip_projection}]
At level $h$ the map projects $\R$ onto
\[
[-\Vmax^{(h)},\Vmax^{(h)}].
\]
It is therefore $1$-Lipschitz and idempotent.  If $g$ lies in the interval,
the scalar projection inequality gives
$|\operatorname{clip}_h(x)-g|\le |x-g|$ pointwise; integration proves the
$L^2$ metric-projection assertion, up to the usual a.e.\ identification.
Applying the scalar Lipschitz inequality statewise proves the supremum- and
$L^2$-nonexpansiveness claims.  The score bound follows
from~\eqref{eq:vmax_recursion} and
\[
|Q_V(s,h,a)|\le\Rmax+\geff\Vmax^{(h-1)}=\Vmax^{(h)}.
\]
The same recursion bounds every observed-transition label and, after
maximization over actions, places $\Tcal V$ in the level-wise band.
Nonexpansiveness maps every supremum-norm $\delta$-net of a raw class to a
$\delta$-net of its clipped image, proving the covering-number assertion.
Finally, because every Bellman target $\Tcal V$ lies in the band, the
pointwise projection inequality proves the Bellman-approximation assertion.
\end{proof}

\noindent\emph{The statewise version of the disintegration}
(Assumption~\ref{ass:residuals}).  A disintegration is determined only
$\mathsf M_{k,S}$-a.e., whereas $\epsbuf$ and $\epsactr{r}$ are norms under
$\rho_{k,S}$, so a version must be fixed on all of $\Stildeo$.  Each
round's policy is $\eps_k$-greedy for a score defined at every state, so
\begin{equation}
\begin{aligned}
f_{k,j}&:=\frac{d(\omega_{k,j}\mu_{k,j})}{d\mathsf M_{k,S}},
&w_{k,j}(\stilde)&:=\frac{f_{k,j}(\stilde)}{\sum_i f_{k,i}(\stilde)},\\
\pibar(\cdot\mid\stilde)&:=\sum_j w_{k,j}(\stilde)\,
\pi^{\mathrm{on}}_{k,j}(\cdot\mid\stilde),
\end{aligned}
\label{eq:disintegration_version}
\end{equation}
using the bounded density versions of Lemma~\ref{lem:measurable_framework}
and $w_{k,j}(\stilde):=\omega_{k,j}$ where the denominator vanishes.
This is a statewise Markov kernel agreeing with the disintegration
$\mathsf M_{k,S}$-a.e.; the same convention fixes $\beta^{\mathrm{rep}}_k$.
Without it, a residual norm under a law not dominated by
$\mathsf M_{k,S}$ would depend on values on an $\mathsf M_{k,S}$-null set,
which the disintegration does not determine; this is why
Assumption~\ref{ass:residuals} requires
$\rho_{k,S}\ll\mathsf M_{k,S}$ for the $\Arun$ comparison.
Corollary~\ref{cor:act_discharge} is stated for exactly this
state-dependent mixture.


\begin{lemma}[Detailed target table for Lemma~\ref{lem:tgt_residual}\obj{\Rabs}]\label{supp:lem:tgt_residual}
Let $A\in\{0,\ldots,k\}$ be the record's \emph{age}, the number of target copies between the
writing of a replayed record and the current block, and
$\bar\delta_k(A):=\|V_k-V_{k-A}\|_\infty$; $A$ is a function of the record and
hence in general state-dependent.  Use the current and writing-time
predictors $\Phat_k$ and $\Phat_{\mathrm{write}}$ and the metadata $M_k$
of Lemma~\ref{lem:tgt_residual}.  The four cells of slot~(S6) are
\begin{equation}
\begin{array}{ccl}
\textsc{when}&\textsc{from}&\multicolumn{1}{c}{Y}\\ \hline
\textsc{sample}&\textsc{obs}&r+\geff(1-d)V_k(\stilde')\\
\textsc{sample}&\textsc{pred}&r+\geff(1-d)V_k(\Phat_k(\stilde,a))\\
\textsc{store}&\textsc{obs}&r+\geff(1-d)V_{k-A}(\stilde')\\
\textsc{store}&\textsc{pred}&r+\geff(1-d)V_{k-A}(\Phat_{\mathrm{write}}(\stilde,a))
\end{array}
\label{eq:four_labels}
\end{equation}
in each of which the reward $r$ and the mask $d$ are the \emph{observed} ones;
only the bootstrapping network and evaluation successor change.  The ideal
response uses the same row with $Z^\circ$ in place of $Z$, retaining the
writing-time parameters and age as in~\eqref{eq:ideal_metadata_redraw}.
Define the
continuation component of the score error by
\[
\eta_k^{\mathrm{cont}}:=\sup_{V\in\mathcal F}\sup_{\stilde,a}
\geff\left|V(\Phat_k(\stilde,a))-\int V\,d\Pkernel(\cdot\mid\stilde,a)\right|.
\]
Then
$\eta_k^{\mathrm{cont}}\le\etascorek$, without a factor $2$, whenever the
point score is the score bounded in~\eqref{eq:envelopes}.  Under a different
slot~\textnormal{(S2)} score, retain $\eta_k^{\mathrm{cont}}$ separately.  For the
two $\textsc{pred}$ cells, and only those, assume in addition that the
terminal event is decided by the state and action,
\begin{equation}
\textnormal{(T)}\qquad
\Prob\bigl[\,d=1\bigm|\stilde,a\,\bigr]\in\{0,1\},
\label{eq:det_termination}
\end{equation}
for $\rho_{k,S}(d\stilde)\beta^{\mathrm{rep}}_k(da\mid\stilde)$-almost every
$(\stilde,a)$,
which requires the full conditioning state and joint dynamics to determine the
terminal event.  Condition~\textnormal{(T)} is therefore an explicit
assumption on this joint structure.  Then the
actual target-construction distance obeys
\begin{equation}
\begin{aligned}
\bigl\|G^{\mathrm{ideal}}_k-
\Tcal^{\beta^{\mathrm{rep}}_k}V_k\bigr\|_{2,\rho_{k,S}}
&\le \Indic\{\textsc{from}=\textsc{pred}\}\eta_k^{\mathrm{cont}}\\[-2pt]
&\quad+\Indic\{\textsc{when}=\textsc{store}\}\geff
\bigl\|\E[\bar\delta_k(A)\mid S,\mathcal J_k]\bigr\|_{2,\rho_{k,S}}\\
&\quad+\Indic\{(\textsc{when},\textsc{from})=(\textsc{store},\textsc{pred})\}
\|C_{\mathrm{pred},k}\|_{2,\rho_{k,S}}.
\end{aligned}
\label{supp:eq:tgt_residual_bound}
\end{equation}
so~\eqref{eq:tgt_residual} holds with $\eps_{\mathrm{tgt},k}$ equal to any
deterministic almost-sure majorant of the right-hand side.
Without \textnormal{(T)}, set $p_T(\stilde,a):=\Prob[d=1\mid\stilde,a]$ and
\begin{align}
C_{\mathrm{mask},k}(\stilde)
&:=\geff\!\int p_T(\stilde,a)
\bigl|V_k\bigl(\Phat_k(\stilde,a)\bigr)\bigr|\,
\beta_k^{\mathrm{rep}}(da\mid\stilde), \notag\\
\|C_{\mathrm{mask},k}\|_{2,\rho_{k,S}}
&\le\geff\Vmax\left\|\int p_T(\cdot,a)
\beta_k^{\mathrm{rep}}(da\mid\cdot)\right\|_{2,\rho_{k,S}}.
\label{eq:tgt_mask_correction}
\end{align}
The right side of~\eqref{supp:eq:tgt_residual_bound} then acquires
$\Indic\{\textsc{from}=\textsc{pred}\}\|C_{\mathrm{mask},k}\|_{2,\rho_{k,S}}$,
which $\eta_k^{\mathrm{cont}}$ does not cover.  Thus the replay-action average is
taken before the $L^2(\rho_{k,S})$ norm.  The conditional expectation in
either staleness term may not be replaced by an unconditional one, since the
age and writing-time predictor can depend on the replayed state.
For the target-network term, conditional Jensen
gives the $\mathcal J_k$-measurable majorant
$\geff(\E[\bar\delta_k(A)^2\mid\mathcal J_k])^{1/2}$; the conditional essential
supremum of $\geff\bar\delta_k(A)$ is another.  A deterministic envelope must
dominate either choice almost surely.
\end{lemma}
\begin{proof}[Proof of Lemma~\ref{supp:lem:tgt_residual}]
When the point score is installed,
$\eta_k^{\mathrm{cont}}\le\etascorek$ \emph{without} the factor
$2$ that the triangle inequality would give: since
$\widehat Q_V-Q_V=(\widehat R-R)+\geff(V\!\circ\!\Phat_k-\!\int\!V\,d\Pkernel)$
and $\mathcal F$ is symmetric under $V\mapsto-V$ by construction, evaluating at
$V$ and at $-V$ and subtracting cancels the reward term, leaving
$2\geff|V\!\circ\!\Phat_k-\int V\,d\Pkernel|\le2\etascorek$
by~\eqref{eq:envelopes}.

Fix a cell of~\eqref{eq:four_labels} and condition on
$(S,\mathcal J_k)$.  Under
$(\textsc{sample},\textsc{obs})$ the ideal fresh-redraw label's conditional
mean under the true kernel is $(\Tcal^{\beta^{\mathrm{rep}}_k}V_k)(\stilde)$
by~\eqref{eq:ideal_metadata_redraw} and the terminal convention, so
$G^{\mathrm{ideal}}_k=\Tcal^{\beta^{\mathrm{rep}}_k}V_k$ and the left-hand side
vanishes.

In the ideal redraw, changing $\textsc{from}$ to $\textsc{pred}$ leaves
$r^\circ$ and $d^\circ$ untouched.
\emph{Both objects compared are conditional means, so the comparison is made
after conditioning, and the redrawn mask is random}: the $\textsc{pred}$ label
has conditional mean
$R(\stilde,a)+\geff(1-p_T(\stilde,a))V_k(\Phat_k(\stilde,a))$ and the
$\textsc{obs}$ label $R(\stilde,a)+\geff\int V_k\,d\Pkernel$, the terminal
convention having absorbed the mask on that side.  Adding and subtracting
$\geff V_k(\Phat_k(\stilde,a))$ bounds their difference by
$\eta_k^{\mathrm{cont}}+\geff p_T(\stilde,a)|V_k(\Phat_k(\stilde,a))|$,
whose average under $\beta_k^{\mathrm{rep}}(da\mid\stilde)$ is bounded by
$\eta_k^{\mathrm{cont}}+C_{\mathrm{mask},k}(\stilde)$.  Taking the state norm
gives~\eqref{eq:tgt_mask_correction}.  Under \textnormal{(T)} a direct cellwise
comparison removes this correction: where
$p_T(\stilde,a)=0$ by inspection, and where it is $1$ the fresh successor is
$\stilde_{\mathrm{term}}$ almost surely, so $\int V_k\,d\Pkernel=0$ while
$(1-d^\circ)$ annihilates the predicted continuation.  Only the \emph{continuation}
component of the score error is charged, the reward being observed and not modelled.

For $\textsc{store}/\textsc{obs}$, condition first on
$(S,\mathbf a,M_k,\mathcal J_k)$ and use the same fresh outcome for the stored
and sample-time labels.  Their pointwise difference is at most
$\geff\bar\delta_k(A)$, since $|1-d^\circ|\le1$.
For $\textsc{store}/\textsc{pred}$, there is also a change of predictor:
\begin{align*}
&V_{k-A}(\Phat_{\mathrm{write}}(S,\mathbf a))
       -V_k(\Phat_k(S,\mathbf a))\\
&\quad=[V_{k-A}-V_k](\Phat_{\mathrm{write}}(S,\mathbf a))\\
&\qquad+V_k(\Phat_{\mathrm{write}}(S,\mathbf a))
       -V_k(\Phat_k(S,\mathbf a)).
\end{align*}
Multiplication by $\geff(1-d^\circ)$, the triangle inequality, and
$|1-d^\circ|\le1$ bound the conditional mean of the absolute label difference
by $\geff\E[\bar\delta_k(A)\mid S,\mathcal J_k]+C_{\mathrm{pred},k}(S)$.
Here the action and metadata are averaged conditional on $(S,\mathcal J_k)$.
Taking the state norm gives the
second and third terms of~\eqref{supp:eq:tgt_residual_bound}.  The stated
target-network majorants follow from conditional Jensen and the essential
supremum bound.  Combining with the sample-time comparison proves the result,
including the general terminal-mask correction.
\end{proof}

\paragraph{Why both stored-label corrections are needed.}
Two finite-state examples isolate the issues.  Take $H=2$, $\Rmax=1$,
a level-two state $x$, zero reward there, and identical actions; all
level-one states terminate at the next step.  First, let the true successor
of $x$ be $v$, let the writing-time predictor return $u$, and let the current
predictor be exact.  Set $V_{k-A}=V_k=V$ with $V(u)=1$ and $V(v)=-1$.
Then a stored prediction label is $\geff$ while
$(\Tcal^{\beta_k^{\mathrm{rep}}}V_k)(x)=-\geff$.  Both current continuation
error and target-network staleness vanish, but
$C_{\mathrm{pred},k}(x)=2\geff$, exactly the missing distance.

Second, take $k\ge1$, let the true successor law at $x$ be uniform on $z_+$ and $z_-$,
let $V_k=0$, and set $V_{k-1}(z_+)=1$, $V_{k-1}(z_-)=-1$.
With fixed $\mathcal J_k$, suppose a selected stored observation record has
$(A,\stilde')=(1,z_+)$ or $(0,z_-)$, each with probability $1/2$.
Its outcome marginal given $(S,\mathbf a,\mathcal J_k)$ is the true kernel,
but $G_k(x)=\geff/2$.  Retaining $A$ and independently redrawing the
successor gives $G_k^{\mathrm{ideal}}(x)=0$.  Thus conditioning without
metadata cannot justify a zero kernel residual.  For $\Aiid$, the label
parameters are fixed before each fresh draw and the switches are
$\textsc{sample}/\textsc{obs}$, so both residuals remain zero.

\begin{proof}[Proof of Lemma~\ref{lem:varying_propagation}]
The decomposition follows the standard fitted-value-iteration argument
\cite[Lemmas~3--4]{munos_szepesvari_2008}; specific here are the truncation of
Step~3 and the weight sum of Step~4.

\emph{Step 1 (recursion).}  Let $z_k:=\Vstar-V_k$, so
$z_{k+1}=(\Tcal\Vstar-\Tcal V_k)-e_k$.  Lemma~\ref{lem:comparison_kernel} at
$(V,W)=(\Vstar,V_k)$ gives a substochastic
$\Pkernel^{(k)}_\circ:=\Pkernel^{\Vstar,V_k}_\circ$ with
$|\Tcal\Vstar-\Tcal V_k|\le\geff\Pkernel^{(k)}_\circ|z_k|$, so
$|z_{k+1}|\le\geff\Pkernel^{(k)}_\circ|z_k|+|e_k|$ pointwise, and unrolling
from $K$ to $0$,
\begin{equation}
|z_K|\le\geff^{K}\Pkernel^{(K-1)}_\circ\!\cdots\Pkernel^{(0)}_\circ|z_0|
+\sum_{k=0}^{K-1}\geff^{\,K-1-k}\,
\Pkernel^{(K-1)}_\circ\!\cdots\Pkernel^{(k+1)}_\circ|e_k| .
\label{eq:bK_unrolled}
\end{equation}

\emph{Step 2 (loss resolvent).}  With $\pi_K$ greedy for $V_K$ and $\pi^\star$
for $\Vstar$: from $\Vstar=\Tcal^{\pi^\star}\Vstar$,
$V^{\pi_K}=\Tcal^{\pi_K}V^{\pi_K}$, the greedy inequality
$\Tcal^{\pi^\star}V_K\le\Tcal V_K=\Tcal^{\pi_K}V_K$ and the splitting
$V_K-V^{\pi_K}=(V_K-\Vstar)+(\Vstar-V^{\pi_K})$,
$(I-\geff\Pkernel^{\pi_K}_\circ)(\Vstar-V^{\pi_K})
\le\geff(\Pkernel^{\pi^\star}_\circ-\Pkernel^{\pi_K}_\circ)z_K$.  The resolvent
$(I-\geff\Pkernel^{\pi_K}_\circ)^{-1}=\sum_{\ell\ge0}\geff^\ell(\Pkernel^{\pi_K}_\circ)^\ell$
is nonnegative;
applying it and $\pm z_K\le|z_K|$ yields the standard loss-resolvent bound in
substochastic form,
$\Vstar-V^{\pi_K}\le\sum_{\ell\ge0}\geff^{\ell+1}(\Pkernel^{\pi_K}_\circ)^\ell
(\Pkernel^{\pi^\star}_\circ+\Pkernel^{\pi_K}_\circ)|z_K|$.

\emph{Step 3 (truncation).}  Substituting~\eqref{eq:bK_unrolled} gives, for each
$|e_k|$, two families of kernel products of length $m=\ell+(K-k)$ and weight
$\geff^m$.  By the clock decrement~\eqref{eq:time_decrement} every
$\Pkernel^\pi_\circ$ strictly decreases $h$, so products of length $m\ge H$
vanish (the policies need not coincide), and the weight range is
$m<H$.  The families multiplying $|z_0|$ have length at least $K+1$
and vanish once $K\ge H-1$.

\emph{Step 4 (the exact realized occupancies).}  For
$m=\ell+K-k<H$, define the two subprobability measures
\begin{align*}
\nu^1_{K,k,\ell}
&:=\mu_S^\circ(\Pkernel^{\pi_K}_\circ)^\ell
  \Pkernel^{\pi^\star}_\circ
  \Pkernel^{(K-1)}_\circ\cdots\Pkernel^{(k+1)}_\circ,\\
\nu^2_{K,k,\ell}
&:=\mu_S^\circ(\Pkernel^{\pi_K}_\circ)^{\ell+1}
  \Pkernel^{(K-1)}_\circ\cdots\Pkernel^{(k+1)}_\circ,
\end{align*}
where the final product is the identity for $k=K-1$.  For conjugate
$s\in[2,\infty]$ and $p=s/(s-1)$, suppose only these realized measures are
absolutely continuous with respect to $\rho_{k,S}$, and put
$d^b_{K,k,\ell,s}:=\|d\nu^b_{K,k,\ell}/d\rho_{k,S}\|_{s,\rho_{k,S}}$.
H\"older's inequality applied separately to the two branches gives the exact
pathwise coefficient form
\begin{equation}
\|\Vstar-V^{\pi_K}\|_{1,\mu_S}
\le \mathcal B_K+
\sum_{k<K}\ \sum_{\substack{\ell\ge0:\\ \ell+K-k<H}}
\geff^{\ell+K-k}
\bigl(d^1_{K,k,\ell,s}+d^2_{K,k,\ell,s}\bigr)
\|e_k\|_{p,\rho_{k,S}}.
\label{eq:realized_occupancy_propagation}
\end{equation}
Thus the proof consumes coverage only for the two displayed families.  Their
coefficients remain multiplied by the residual norms; no independence or
product-of-expectations step is valid in general.

\emph{Step 5 (deterministic envelope, boundary, and count).}  Every constructed
kernel is a measurable policy kernel (Lemma~\ref{lem:comparison_kernel}).
Assumption~\ref{ass:concentrability} gives
$d^1_{K,k,\ell,s},d^2_{K,k,\ell,s}\le d_s(m)$, which reduces
\eqref{eq:realized_occupancy_propagation} to the residual sum in
\eqref{eq:varying_propagation}.  For the initialization term, start on clock
slice $h$.  A branch of total length $m$ vanishes for $m\ge h$ and otherwise
ends on slice $h-m$, where $|z_0|\le D_{0,h-m}$ almost surely.  Summing
the two branches against the initial masses $p_h$ bounds their contribution
by the deterministic $\mathcal B_K$ in~\eqref{eq:finite_K_boundary_phi}.
Level-wise clipping permits the choice $D_{0,j}=2V_{\max}^{(j)}$;
the boundary vanishes for $K\ge H-1$ regardless of that choice.

Finally, the coefficient of $\|e_k\|_{p,\rho_{k,S}}$ is
$w^{(H)}_{K,k}$.  With $m=\ell+K-k$, a fixed $1\le m<H$ admits exactly
$\min\{K,m\}$ indices $k$.  Fubini--Tonelli therefore gives
$\sum_{k<K}w^{(H)}_{K,k}=2\sum_{m<H}\min\{K,m\}\geff^m d_s(m)
=2\phi_{s,K}^{(H)}$.  The weights are deterministic, so expectations pass
term by term.
\end{proof}
\begin{proof}[Proof of Proposition~\ref{prop:model_floor_lower}]
\emph{MDP and laws.}  All unspecified rewards are zero.  Use $O=S$, the full
clipped tabular class, and exact fresh population
$(\textsc{sample},\textsc{obs})$ updates at every nonterminal state.
Nonterminal states are $s$ at $h=3$, $y_1,y_2$ at
$h=2$, and $z_1,z_2,w$ at $h=1$.  Transitions are deterministic and rewards
are nonzero only at $h=1$: $s:u_1\to y_1$, $s:u_2\to y_2$;
$y_1:b_1\to z_1$, $y_1:b_2\to z_2$; and both actions at $y_2$ lead to $w$.
At $z_1,z_2,w$, every action enters the zero-reward absorbing terminal state
after receiving, respectively, $R(z_1)=1$, $R(z_2)=1-u$, and $R(w)=v$, with
$u,v\in(0,1)$ fixed below.  Take $\mu_S=\delta_s$ and let $\rho_S$ be uniform
on these six nonterminal states; thus every nonterminal state has replay
probability $1/6$ and $\rho_S$ has full support.  Initialize the population
recursion at $V_0=0$.

\emph{Frozen values and the true gap.}  At $h=1$ the target carries no
bootstrap, so exact population regression gives $V_k(z_1)=1$, $V_k(z_2)=1-u$,
$V_k(w)=v$ for $k\ge1$; hence $Q_{V_k}(y_1,b_1)=\geff$ and
$Q_{V_k}(y_1,b_2)=\geff(1-u)$, so $b_1$ is target-greedy and
$\Delta_Q^{(k)}(y_1)=\geff u$.

\emph{A pure reward-model error of size $\eta$.}  Let $\Phat$ be exact and
$\widehat R=R$ except $\widehat R(y_1,b_1)=-\eta$, $\widehat R(y_1,b_2)=+\eta$.
Then $\|\widehat Q_{V_k}-Q_{V_k}\|_\infty=\|\widehat R-R\|_\infty=\eta$ for
every $k$.  Because the error is placed in $\widehat R$ alone, it does not
interact with the bootstrap.  The implemented branch prefers $b_2$ at $y_1$ exactly when
$\geff(1-u)+\eta>\geff-\eta$, i.e.\ $u<2\eta/\geff$.  Take
$0<\vartheta<\min\{\eta/\geff,\epsilon/(2\geff^2)\}$ and set
$u:=2\eta/\geff-\vartheta$, which lies in $(0,1)$ because
$\eta<\geff/2$.  The acting network is the frozen target, so
$\dnet=0$ and $\Lambda_k=\etascorek=\eta$: the harmful behavior
is \emph{produced} by the named score error, not stipulated.

\emph{Limit and loss.}  For every $k\ge1$, the recursion gives
$V_{k+1}(y_1)=\geff(1-u)$ and $V_{k+1}(y_2)=\geff v$; the clock makes all
relevant values exact after $H=3$ sweeps.  Choose
$v\in(1-u,\min\{1,1-u+\epsilon/(2\geff^2)\})$;
then $Q_{V_K}(s,u_1)=\geff^2(1-u)<\geff^2v=Q_{V_K}(s,u_2)$, so $\pi_K(s)=u_2$
strictly.  Meanwhile $\Vstar(y_1)=\geff$ through $b_1$, $\Vstar(y_2)=\geff v$
and $\Vstar(s)=\geff^2\max(1,v)=\geff^2$ since $v<1$, while
$V^{\pi_K}(s)=\geff\Vstar(y_2)=\geff^2v$; hence
$\|\Vstar-V^{\pi_K}\|_{1,\mu_S}=\geff^2(1-v)
>\geff^2u-\epsilon/2
=2\geff\eta-\geff^2\vartheta-\epsilon/2
>2\geff\eta-\epsilon$, which is~\eqref{eq:model_floor_lower}.

\emph{Other residuals.}  At each block the same current policy collects the
fresh data and is its replay law, so $\epsbuf=0$; the population update is
exact in the full tabular class, so $\eps_{\mathrm{fit},k}=0$;
fresh true-kernel outcomes with the default $(\textsc{sample},\textsc{obs})$
target give $G_k=G_k^{\mathrm{ideal}}=
\Tcal^{\beta_k^{\mathrm{rep}}}V_k$, hence
$\eps_{\mathrm{ker},k}=\eps_{\mathrm{tgt},k}=0$; collection is greedy, so
$\eps_k=0$.  Full support of $\rho_S$ gives finite $c_2(m)$ on this finite MDP.
\end{proof}
\paragraph{Statistical and consequence calculations.}
\begin{proof}[Proof of Lemma~\ref{lem:erm_oracle}]
Write $P$ and $P_n$ for the population and the empirical measure and, for
$V\in\mathcal C$, let $\ell_V:=(V-Y)^2-(g-Y)^2$ be the excess loss.  Expanding
and using $\E[Y\mid S]=g(S)$, the cross term vanishes, so
\[
P\ell_V=\|V-g\|^2_{2,\rho_{k,S}}=:\mathcal E(V)\ \ge0 .
\]
\emph{Step 1 (three elementary bounds).}  Since
$\ell_V=(V-g)(V+g-2Y)$ with $|V-g|\le2B$ and $|V+g-2Y|\le4B$, and since
$\ell_V$ is a difference of two quantities lying in $[0,4B^2]$,
\begin{equation}
|\ell_V|\le4B^2,\qquad
\E[\ell_V^2]\le16B^2\mathcal E(V),\qquad
|\ell_V-\ell_{V'}|\le4B\|V-V'\|_\infty ,
\label{eq:loss_ingredients}
\end{equation}
the last because $\ell_V-\ell_{V'}=(V-V')(V+V'-2Y)$.

\emph{Step 2 (reduction to a finite net).}  Let $\{V_1,\ldots,V_N\}$,
$N=N_\delta(\mathcal C)$, be a $\delta$-net of $\mathcal C$ in supremum norm and,
for $V\in\mathcal C$, let $j(V)$ index a net point with
$\|V-V_{j(V)}\|_\infty\le\delta$.  By the third bound
in~\eqref{eq:loss_ingredients}, the empirical-process difference is at most
$8B\delta$.  The reverse triangle inequality for
$\sqrt{\mathcal E(V)}=\|V-g\|_{2,\rho_{k,S}}$ gives, after separating the cases
$\sqrt{\mathcal E(V_{j(V)})}\gtreqless\delta$,
$-\mathcal E(V)/2\le-\mathcal E(V_{j(V)})/2+2B\delta$.  Hence, pointwise,
\begin{equation}
(P-P_n)\ell_V-\tfrac12\mathcal E(V)\;\le\;
\max_{j\le N}\Bigl[(P-P_n)\ell_{V_j}-\tfrac12\mathcal E(V_j)\Bigr]+10B\delta .
\label{eq:net_reduction}
\end{equation}
The right-hand side is a maximum of finitely many measurable functions.

\emph{Step 3 (Bernstein with a shifted threshold).}  The standard
localization device for bounded squared-loss
regression~\cite[Ch.~11]{gyorfi2002distribution}.  Fix $j$ and put
$\mathcal E_j:=\mathcal E(V_j)$.  By~\eqref{eq:loss_ingredients}, each centered
summand is bounded above by $8B^2$ and has variance at most
$16B^2\mathcal E_j$.  Bernstein's inequality at threshold
$t+\mathcal E_j/2$, followed by a union bound, yields
$\Prob(\max_j[(P-P_n)\ell_{V_j}-\mathcal E_j/2]>t)
\le N\exp\{-nt/(70B^2)\}$.  Integrating gives
\[
\E\Bigl[\max_{j\le N}\bigl((P-P_n)\ell_{V_j}-\tfrac12\mathcal E_j\bigr)\Bigr]_+
\le\frac{70B^2}{n}\bigl(1+\log N\bigr).
\]

\emph{Step 4 (assembly).}  Enumerate the fixed countable class
$\mathcal C_0$ and let $V^\star$ be the first element satisfying
$\|V^\star-g\|_{2,\rho_{k,S}}\le\mathrm{dist}_{2,\rho_{k,S}}(\mathcal C,g)+\tau$
for a slack $\tau>0$.  Supremum-norm density ensures existence and the
first-index rule is measurable.  The $L^2$ choice is essential because
$\mathcal E(V)=\|V-g\|^2_{2,\rho_{k,S}}$.  Approximate empirical optimality gives
$P_n\ell_{\widehat V_{k+1}}\le P_n\ell_{V^\star}+\zeta_n$, so
\[
\tfrac12\mathcal E(\widehat V_{k+1})
\le\Bigl[(P-P_n)\ell_{\widehat V_{k+1}}-\tfrac12\mathcal E(\widehat V_{k+1})\Bigr]
+(P_n-P)\ell_{V^\star}+\mathcal E(V^\star)+\zeta_n .
\]
Taking expectations, applying~\eqref{eq:net_reduction} and Step~3, and using
$\E(P_n-P)\ell_{V^\star}=0$ (conditionally on $\mathcal K$ in the conditional
form), yields
\[
\tfrac12\E \mathcal E(\widehat V_{k+1})\;\le\;
\frac{70B^2}{n}\bigl(1+\log N_\delta(\mathcal C)\bigr)+10B\delta
+\bigl[\mathrm{dist}_{2,\rho_{k,S}}(\mathcal C,g)+\tau\bigr]^2+\zeta_n .
\]
Multiplying by $2$ and letting $\tau\downarrow0$
gives~\eqref{eq:erm_oracle} for $C_1\ge140$.  Only boundedness and the
conditional mean of $Y$ were used, so bounded label noise is covered.

For~\eqref{eq:erm_oracle_hp}, use a fresh symbol $u\in(0,1)$ for the
confidence level during this proof.  Since $g\in\mathcal C$ and the supplied
ERM is exact over $\mathcal C$, comparison with $g$ gives
$P_n\ell_{\widehat V_{k+1}}\le P_n\ell_g=0$.  Hence
\[
\tfrac12\mathcal E(\widehat V_{k+1})
\le (P-P_n)\ell_{\widehat V_{k+1}}
-\tfrac12\mathcal E(\widehat V_{k+1}).
\]
By~\eqref{eq:net_reduction} and the tail bound of Step~3, with conditional
probability at least $1-u$ the right-hand side is at most
$70B^2\{\log N_\delta(\mathcal C)+\log(1/u)\}/n+10B\delta$.
Multiplication by $2$, enlargement to $C_1\ge140$, and renaming $u$ as
$\tau$ prove~\eqref{eq:erm_oracle_hp}.  This argument does not divide by any
state probability.
\end{proof}
\begin{proof}[Proof of Proposition~\ref{prop:vreg_rate}]
Throughout $g:=G_k$, so $\|g\|_\infty\le\Vmax$ by~\eqref{eq:vmax_recursion},
$|Y_i|\le\Vmax$ and $\E[Y_i\mid S_i]=g(S_i)$; Lemma~\ref{lem:erm_oracle}
applies with $B=\Vmax$ and $\mathcal C=\FVn$, being stated for an arbitrary bounded
measurable $g$ and \emph{not} requiring $g\in\Gclk$.

\emph{Step 1 (approximation).}  The routed approximation
bound~\eqref{eq:routed_approximation} already applies to the externally
clipped class:
\[
\mathrm{dist}^{\sup}_\infty(\FVn,\Gclk)
\le C_{\mathrm{appx}}m_n^{(\alpha^\star-1)/2}.
\]
Its proof uses Lemma~\ref{lem:clip_projection} head by head, exploiting that
every target in $\Gclk$ lies in the level-wise band.

\emph{Step 2 (covering).}  If $\Vmax=0$ the claim is trivial.  Otherwise set
$\delta_n=(1\wedge\Vmax)/n$.  The product entropy
bound~\eqref{eq:routed_entropy} gives
\begin{equation}
\begin{aligned}
\log N_{\delta_n}(\FVn)
&\le H\left\{\log m_n+(s_{m_n}+1)
\log\!\left(\frac{2(L_{m_n}+1)D_{m_n}^{2}}{\delta_n}\right)\right\}\\
&\lesssim Hm_n^{\alpha^\star}(\log n)^{1+2\xi^\star}.
\end{aligned}
\label{eq:covering}
\end{equation}
\emph{Step 3 (combination, comparator in $L^2(\rho_{k,S})$).}  Since
$\Tcal V_k\in\Gclk$ and $\|\cdot\|_{2,\rho_{k,S}}\le\|\cdot\|_\infty$,
the triangle inequality in $L^2(\rho_{k,S})$ gives
\begin{equation}
\mathrm{dist}_{2,\rho_{k,S}}(\FVn,g)
\le\mathrm{dist}^{\sup}_\infty(\FVn,\Gclk)+\bigl\|\Tcal V_k-G_k\bigr\|_{2,\rho_{k,S}}
\le C_{\mathrm{appx}}m_n^{(\alpha^\star-1)/2}+D^{(2)}_k .
\label{eq:vreg_split}
\end{equation}

\emph{Step 4 (take the square root before splitting).}  Apply Jensen, then
subadditivity of $\sqrt{\cdot}$, to the four summands of~\eqref{eq:erm_oracle}
at the radius $\delta_n=(1\wedge\Vmax)/n$:
\[
\begin{aligned}
\E\bigl[\|\widehat V_{k+1}-g\|_{2,\rho_{k,S}}\bigm|\mathcal H_k\bigr]
&\le\sqrt2\,\mathrm{dist}_{2,\rho_{k,S}}(\FVn,g)\\
&\quad+\Bigl(\tfrac{C_1\Vmax^2}{n}
  \bigl(1+\log N_{\delta_n}(\FVn)\bigr)\Bigr)^{1/2}\\
&\quad+\bigl(C_1\Vmax\delta_n\bigr)^{1/2}+\sqrt{2\zeta_n}.
\end{aligned}
\]
By~\eqref{eq:vreg_split} the first summand is at most
$\sqrt2C_{\mathrm{appx}}m_n^{(\alpha^\star-1)/2}+\sqrt2D^{(2)}_k$.  Since
$n\ge Hm_n$, equation~\eqref{eq:covering} yields
\[
\left(\frac{Hm_n^{\alpha^\star}}{n}\right)^{1/2}
\le m_n^{(\alpha^\star-1)/2}.
\]
The remaining $n^{-1/2}$ terms are no larger than a constant times this rate.
Collecting constants gives~\eqref{eq:vreg_offgreedy}.  For fixed $H$,
$m_n\asymp n/H$, so this has the same exponent in $n$ as the one-head rate;
the displayed $m_n$ records the router's $H$-head complexity cost.
\emph{Step 5 (greedy collection).}  If $\mathbf a_i=a^\star(S_i;V_k)$ then
$G_k=\Tcal V_k$ by Lemma~\ref{lem:v_q_backup}(iii), so $D^{(2)}_k=0$
and~\eqref{eq:vreg_offgreedy} reduces to~\eqref{eq:vreg_derived}.  Noisy
targets are covered because Lemma~\ref{lem:erm_oracle} requires only
$\E[Y_i\mid S_i]=g(S_i)$ with $|Y_i|\le\Vmax$.
\end{proof}
\begin{proof}[Proof of Corollary~\ref{cor:tabular_discharge}]
Lemma~\ref{lem:clip_projection} gives exact closure.  Empirical means at
visited states (zero at unvisited states) give a measurable exact ERM in the
level-wise band because every label at level $h$ is bounded by
$\Rmax+\geff\Vmax^{(h-1)}=\Vmax^{(h)}$.  Condition on $\mathcal H_k$.
Given $N_{k,s}=r>0$, the coordinatewise sample mean is unbiased with squared
error $\sigma_{k,s}^2/r$; when $r=0$, its squared error is $g_{k,s}^2$.
Multiplying by $p_{k,s}$ and summing proves the exact
identity~\eqref{eq:tabular_rate}; coordinates with $p_{k,s}=0$ vanish.

For $r\ge1$, $1/r\le2/(r+1)$.  If $M\sim\mathrm{Bin}(n,p)$, then
\[
 \E\frac1{M+1}=\int_0^1(1-p+px)^n\,dx
 =\frac{1-(1-p)^{n+1}}{(n+1)p},
\]
so $p\E[\mathbf1\{M>0\}/M]\le2/(n+1)$.  Also
$p\Pr(M=0)=p(1-p)^n\le(n/(n+1))^n/(n+1)$, by maximizing over $p$.
Using $\sigma_{k,s}^2,g_{k,s}^2\le\Vmax^2$, summing over the $N$ states and
applying Jensen proves the first bound in~\eqref{eq:tabular_rate_expected}.
Finally $(1+1/n)^n\ge2$, and for $n\ge2$ its first three binomial terms are
at least $9/4$, giving the constants $5/2$ and $22/9$.  These bounds are
deterministic after conditioning, so they also hold unconditionally.

The additional design choices respectively
remove kernel, target, replay, aliasing, and drift links; full support makes
every finite-state density coefficient finite.  Since
$G_k\in\mathcal F^{\mathrm{tab}}$, no policy-image approximation term is
needed, so the composition chain charges the action and exploration links
only once.  Substitution of~\eqref{eq:tabular_rate_expected} and
Theorem~\ref{thm:act_upper} at $r=p$ into
Theorem~\ref{thm:hmpdrl_end_to_end}, with
$K\ge H-1$, gives~\eqref{eq:fully_discharged}.

For fixed $K$, name the $k$th confidence event
\[
\mathcal E_{k,K}(\delta):=
\left\{\|V_{k+1}-G_k\|_{2,\rho_{k,S}}
\le\mathrm{stat}^{\mathrm{tab,hp}}_{k,K}(\delta)\right\}.
\]
Apply~\eqref{eq:erm_oracle_hp} conditionally on $\mathcal H_k$ with confidence
$1-\delta/(H-1)$ and covering radius $\Vmax/n_k$.  It gives
$\Prob(\mathcal E_{k,K}(\delta)^c\mid\mathcal H_k)\le\delta/(H-1)$ directly in
population $L^2$, including the zero values assigned to unvisited states, so
no minimum-state-mass factor occurs.  The tower property and a union bound over
the at most $H-1$ relevant blocks make all these events simultaneous with
probability at least $1-\delta$, despite adaptive histories.  On their
intersection, the pathwise triangle chain of Lemma~\ref{lem:composition} and
Lemma~\ref{lem:varying_propagation} gives~\eqref{eq:fully_discharged_hp}; no
cross-block independence is used.
\end{proof}
\begin{proof}[Proof of Theorem~\ref{thm:nonasymptotic}]
Substitution of Proposition~\ref{prop:l2_routing} into
Proposition~\ref{prop:vreg_rate} gives
\[
\eps_{\mathrm{fit},k}\;\le\;\mathrm{stat}_k
+\sqrt2\bigl(\epsbuf+\epsact+\eps_k\Dbar\bigr),
\]
the omitted kernel and target links being zero.  Since $p=2$ here,
$\epsactp=\epsact$; adding the remaining terms of $e_{k,2}^{\mathrm{Bell}}$
proves~\eqref{eq:rk_l2}.
Equation~\eqref{eq:nonasymptotic} follows from~\eqref{eq:hmpdrl_end_to_end},
since $w^{(H)}_{K,k}=0$ for $k\le K-H$ and
$\sum_kw^{(H)}_{K,k}=2\phi_{2,K}^{(H)}$.
\end{proof}
\section{Further refinements}\label{app:deferred_results}

\begin{proposition}[One-sided propagation]\label{supp:prop:onesided}
Under Theorem~\ref{thm:hmpdrl_end_to_end}, suppose additionally that, almost
surely,
\begin{equation}
V_0\le\Vstar,
\qquad e_k:=V_{k+1}-\Tcal V_k\le0
\quad\text{pointwise on $\Stildeo$, for every $k<K$}.
\label{supp:eq:onesided_hyp}
\end{equation}
Then $V_k\le\Vstar$ for every $k\le K$, and
\begin{equation}
\E\|\Vstar-V^{\pi_K}\|_{1,\mu_S}
\le\frac{\mathcal B_K}{2}
+\frac12\sum_{k=0}^{K-1}w^{(H)}_{K,k}e_{k,p}^{\mathrm{Bell}}
\le\frac{\mathcal B_K}{2}
+\phi_{s,K}^{(H)}\overline e_{K,H,p}^{\mathrm{Bell}}.
\label{supp:eq:onesided_bound}
\end{equation}
The sign premise is pointwise; an $L^p$ or replay-almost-everywhere sign
condition is insufficient.
\end{proposition}

\begin{proof}
For $z_k:=\Vstar-V_k$, monotonicity and~\eqref{supp:eq:onesided_hyp} give
$z_k\ge0$ inductively and
$z_{k+1}\le\geff\Pkernel^{\pi^\star}_\circ z_k+|e_k|$, so only the
$\pi^\star$ branch propagates.  At the loss-resolvent step,
$\geff(\Pkernel^{\pi^\star}_\circ-\Pkernel^{\pi_K}_\circ)z_K
\le\geff\Pkernel^{\pi^\star}_\circ z_K$.  Thus the pathwise proof of
Lemma~\ref{lem:varying_propagation} has half the initialization boundary and
half every residual coefficient; the conditional envelopes of
Theorem~\ref{thm:hmpdrl_end_to_end} give~\eqref{supp:eq:onesided_bound}.
\end{proof}

\begin{corollary}[$L^2$ final-score transfer]\label{supp:cor:deployment_l2}
In the setting of Theorem~\ref{thm:deployed_transfer}, define
\[
E_K(\stilde):=\max_{a\in\Aspace}
|\widehat Q_{V_K}(\stilde,a)-Q_{V_K}(\stilde,a)|,
\qquad \eta_{K,2}:=\|E_K\|_{2,\rho_{K,S}}.
\]
Assume the final-law density bound~\eqref{eq:final_concentrability} at $s=2$,
and let the deterministic $\bar\eta_{K,2}$ satisfy
$\eta_{K,2}\le\bar\eta_{K,2}$ almost surely.  Then
\begin{equation}
\E\|\Vstar-V^{\widehat\pi_K}\|_{1,\mu_S}
\le\mathcal B_K+\sum_{k=0}^{K-1}w^{(H)}_{K,k}e_{k,p}^{\mathrm{Bell}}
+2\bar\eta_{K,2}\sum_{\ell=0}^{H-1}\geff^\ell d_2^{\mathrm{fin}}(\ell).
\label{supp:eq:deployment_l2}
\end{equation}
If $\eta_{K,2}$ is integrable, the last term may instead retain
$2\E[\eta_{K,2}]$ times the same deterministic sum.
\end{corollary}

\begin{proof}
Score comparison gives
$0\le\Tcal V_K-\Tcal^{\widehat\pi_K}V_K\le2E_K$, whose resolvent contribution
after integration is
$2\sum_{\ell<H}\geff^\ell
\int E_K\,d\{\mu_S^\circ(\Pkernel^{\widehat\pi_K}_\circ)^\ell\}$.
Cauchy--Schwarz and~\eqref{eq:final_concentrability} bound it pathwise by
$2\eta_{K,2}\sum_{\ell<H}\geff^\ell d_2^{\mathrm{fin}}(\ell)$.
Use either its deterministic majorant or its expectation to conclude.
\end{proof}

\subsection*{OA.1: History-space moment coverage}\label{supp:oa1}

For the two realized occupancy branches in the proof of
Lemma~\ref{lem:varying_propagation}, define
\begin{equation}
A_{K,k}:=\sum_{\substack{\ell\ge0:\ell+K-k<H}}
\geff^{\ell+K-k}
\bigl(d^1_{K,k,\ell,s}+d^2_{K,k,\ell,s}\bigr),
\qquad R_k:=\|e_k\|_{p,\rho_{k,S}}.
\label{supp:eq:random_coefficients}
\end{equation}
The exact pathwise proof gives
$\|\Vstar-V^{\pi_K}\|_{1,\mu_S}\le\mathcal B_K+
\sum_{k<K}A_{K,k}R_k$.  Consequently, for conjugate history-space exponents
$u,v\in[1,\infty]$,
\begin{equation}
\E\|\Vstar-V^{\pi_K}\|_{1,\mu_S}
\le\mathcal B_K+
\sum_{k<K}\|A_{K,k}\|_{L^u(\Omega)}\|R_k\|_{L^v(\Omega)}.
\label{supp:eq:history_holder}
\end{equation}
This is H\"older's inequality on the history probability space.  Because
$A_{K,k}$ generally depends on future iterates through $\pi_K$ and the
comparison kernels, factorizing $\E[A_{K,k}R_k]$ is invalid; the deterministic
worst-policy theorem is the $u=\infty$ case.

Realized coverage is strictly weaker: in an $H=2$ deterministic MDP, let both
constructed branches select a covered bottom state $x$, set the design law to
$\delta_x$, and add an unused action leading to $y$.  All realized measures are
covered, whereas the unused policy produces $\delta_y\not\ll\delta_x$; hence
the all-policy supremum assumption fails.

\subsection*{OA.2: Geometric verification of the margin}\label{supp:oa2}

\begin{proposition}[Tube and transverse-growth criterion for one gap law]\label{supp:prop:tube_margin}
Let $\Sigma$ be a measurable switching set.  Suppose, for constants
$c,r,T,\kappa,u_0>0$,
\[
\Delta_Q(x)\ge\min\{c\,\operatorname{dist}(x,\Sigma)^r,u_0\},
\quad
\rho\{\operatorname{dist}(x,\Sigma)\le\epsilon\}\le T\epsilon^\kappa
\]
whenever $0<\epsilon\le(u_0/c)^{1/r}$.  Then the global action-gap margin holds
with
\begin{equation}
\alpha=\kappa/r,
\qquad
C_{\mathrm{marg}}=\max\{Tc^{-\kappa/r},u_0^{-\kappa/r}\}.
\label{supp:eq:geometric_margin}
\end{equation}
\end{proposition}

\begin{proof}
For $0<u<u_0$, the event $\{\Delta_Q\le u\}$ lies in the tube of radius
$(u/c)^{1/r}$ and has probability at most $Tc^{-\kappa/r}u^{\kappa/r}$.
For $u\ge u_0$, use $1\le u_0^{-\kappa/r}u^{\kappa/r}$.  The tube condition
also makes the zero-gap switching set null.
\end{proof}

For Definition~\ref{def:margin}, require this criterion almost surely for
each $(Q_{V_k},\rho_{k,S})$; alternatively verify it uniformly for
$(Q_{\Vstar},\rho_{k,S})$ and apply
Proposition~\ref{prop:reference_margin_transfer}.

\subsection*{OA.3: Matched-budget allocation and horizon calculations}\label{supp:oa3}

\begin{proof}[Proof of Theorem~\ref{thm:matched_budget}]
Substitute~\eqref{eq:budget_rate_envelopes} into the first inequalities of
Theorems~\ref{thm:hmpdrl_end_to_end} and~\ref{thm:level_indexed}.  Inactive
coordinates have zero propagation weight, while the nonstatistical terms give
$F_K^{\mathrm{sh}}$ and $F_K^{\mathrm{lev}}$.  It remains to minimize
$\sum_i c_i n_i^{-\nu}$ under a common terminal-window budget.

For positive $c_i$, the objective is strictly convex on the positive orthant.
The Lagrange equations
\[
-\nu c_i n_i^{-\nu-1}+\lambda=0
\]
give $n_i\propto c_i^{1/(1+\nu)}$.  Normalizing by the budget and substituting
back proves~\eqref{eq:matched_budget_allocation}, hence
\eqref{eq:shared_matched_budget}--\eqref{eq:level_matched_budget}.

With lower bounds, strict convexity still gives a unique minimizer.  The KKT
conditions say that an interior coordinate satisfies
$n_i=(\nu c_i/\lambda)^{1/(1+\nu)}$, whereas a coordinate whose unconstrained
value is at most $L_i$ is fixed at $L_i$.  This is exactly
\eqref{eq:water_filling_allocation}.  For
$\mathsf N>\sum_iL_i$, the sum of its right-hand side is continuous and
strictly decreasing in $\lambda$ over the range relevant to the budget, from
infinity to $\sum_iL_i$; hence the required $\lambda$ is unique.  Substitution
defines~\eqref{eq:constrained_budget_value} and proves the stated constrained
bounds.  The unconstrained closed form applies precisely when none of
its coordinates violates a lower bound.

For integer $L_i\ge1$ and integer $\mathsf N$, each
$\lfloor n_i^L\rfloor\ge L_i$, and the number of undistributed labels is the
nonnegative integer $\mathsf N-\sum_i\lfloor n_i^L\rfloor$.  Allocating each
one according to~\eqref{eq:integer_budget_allocation} preserves feasibility
and uses the entire budget.  Moreover,
$\lfloor n_i^L\rfloor\ge n_i^L/2$ because $n_i^L\ge1$, while adding labels can
only decrease the objective.  Therefore the final integer allocation obeys
\[
\sum_i c_i n_i^{-\nu}
\le2^\nu\sum_i c_i(n_i^L)^{-\nu}
=2^\nu\Psi_\nu(c,L,\mathsf N).
\]
Finally, setting $j=K-k$ gives
$|\mathcal K_K|=J$ and
$|\mathcal I_K|=\sum_{j=1}^{J}(H-j)=JH-J(J+1)/2$ for
$J=\min\{K,H-1\}$.
\end{proof}

\begin{proof}[Proof of Corollary~\ref{cor:matched_budget_geometry}]
Put $q:=1/(1+\nu)$ and index the active shared-reset blocks by
$j=K_H-k\in\{1,\ldots,H-1\}$.  On the one-state-per-level chain, the uniform
clock law has $d_{2,H}(m)=\sqrt H$, and therefore
\begin{equation}
w^{\mathrm{unif}}_{H,j}
=2\sqrt H\sum_{m=j}^{H-1}\geff_H^m.
\label{supp:eq:uniform_budget_weight}
\end{equation}
In the near-unit regime, $\geff_H^m$ is bounded above and below by positive
constants uniformly for $m<H$.  Thus
$w^{\mathrm{unif}}_{H,j}=\Theta(H^{3/2})$ for $j\le H/2$ and is
$O(H^{3/2})$ everywhere.  Consequently
\[
\left\{\sum_{j=1}^{H-1}
(b_H^{\mathrm{sh}}w^{\mathrm{unif}}_{H,j})^q\right\}^{1/q}
=\Theta\bigl(b_H^{\mathrm{sh}}H^{\nu+5/2}\bigr).
\]

For the coefficient-optimal shared law at $s=2$ and $K_H\ge H-1$, write
\[
S_H:=\sum_{m=1}^{H-1}(m\geff_H^m)^{2/3},\qquad
r_{H-m}=\frac{(m\geff_H^m)^{2/3}}{S_H}.
\]
Then
$d_{2,H}(m)=S_H^{1/2}(m\geff_H^m)^{-1/3}$ and
\begin{equation}
w^{\mathrm{opt}}_{H,j}
=2S_H^{1/2}\sum_{m=j}^{H-1}m^{-1/3}\geff_H^{2m/3}.
\label{supp:eq:optimal_budget_weight}
\end{equation}
Near unit, $S_H=\Theta(H^{5/3})$; the last sum is
$\Theta(H^{2/3})$ for $j\le H/2$ and $O(H^{2/3})$ everywhere.  Hence the same
calculation gives
$\mathsf C_{\nu,K_H}^{\mathrm{sh,opt}}
=\Theta(b_H^{\mathrm{sh}}H^{\nu+5/2})$.

For direct reset, a fixed depth $m$ occurs in exactly $m$ active pairs and
$A^{\mathrm{lev}}_{k,H-m}=2\geff_H^m$.  Therefore, in the near-unit regime,
\[
\mathsf C_{\nu,K_H}^{\mathrm{lev}}
\asymp b_H^{\mathrm{lev}}
\left(\sum_{m=1}^{H-1}m\right)^{1/q}
=\Theta\bigl(b_H^{\mathrm{lev}}H^{2\nu+2}\bigr),
\]
proving~\eqref{eq:matched_near_unit_geometry}.

Under fixed discount,~\eqref{supp:eq:uniform_budget_weight} is
$\Theta(\sqrt H\,\geff^j)$, so its $q$th-power sum gives
$\Theta(b_H^{\mathrm{sh}}\sqrt H)$.  In
\eqref{supp:eq:optimal_budget_weight}, $S_H$ is bounded above and below and the
weights decay geometrically; their $q$th-power sum is finite and bounded away
from zero.  Likewise
$\sum_{m\ge1}m(2\geff^m b_H^{\mathrm{lev}})^q$ is finite and positive.  This
proves~\eqref{eq:matched_fixed_geometry}.  Finally, the tabular fit envelope
\eqref{eq:tabular_rate_expected} has statistical constant
$\Theta(\Vmax\sqrt H)$ for a shared $H$-state clock law and
$\Theta(\Vmax)$ on each singleton slice.  Substitution at $\nu=1/2$ proves
\eqref{eq:matched_tabular_geometry}.
\end{proof}

\subsection*{OA.4: Tie-sensitive equality at the score floor}\label{supp:oa4}

In Proposition~\ref{prop:model_floor_lower}, take $u=2\eta/\geff$ and $v=1-u$,
$0<\eta<\geff/2$.  If the harmful actions win the resulting ties at $y_1$
and the root, exact population updates give
\begin{equation}
V^\star(s)-V^{\pi_K}(s)=\geff^2(1-v)=\geff^2u=2\geff\eta.
\label{supp:eq:tie_floor}
\end{equation}
All other residuals vanish; without coordinated tie-breaking, use the main
$2\geff\eta-\epsilon$ bound.

\subsection*{OA.5: Polynomial-schedule consistency rate}\label{supp:oa5}

Let $a=(1-\alpha^\star)/2$ and $b=(1+2\xi^\star)/2$.  In the setting of
Corollary~\ref{cor:fresh_consistency}, suppose
$n_k\asymp k^r$, $\zeta_{n_k}=O(n_k^{-t})$, and
$\eps_k=O(k^{-s_0})$, with $r,t,s_0>0$.  For fixed $H$ and $K\ge2H$,
\begin{equation}
\E\|\Vstar-V^{\pi_K}\|_{1,\mu_S}
=O\!\left(\phi_2^{(H)}\left[
(\log K)^bK^{-ra}+K^{-rt/2}+K^{-s_0}\right]\right).
\label{supp:eq:polynomial_schedule}
\end{equation}
For $K\ge2H$, the boundary vanishes and active $k\asymp K$;
$m_{n_k}\asymp n_k$ at fixed $H$, so~\eqref{eq:fresh_consistency} applies.
The statistical term governs if $t\ge2a$ and $s_0\ge ra$.

\subsection*{OA.6: Monte Carlo score bound}\label{supp:oa6}

Assume arbitrary-query access, and let $A:=|\Aspace|$.  For each action, let
$Z_a$ be a centered $M$-sample average of conditionally independent variables
in $[-\Vmax,\Vmax]$; dependence across actions is allowed.  With deterministic
reward error $\eta_R$, the score error $e^M$ satisfies
\begin{equation}
\bigl(\E[e^M(s)^2]\bigr)^{1/2}
\le\eta_R+\frac{\geff\Vmax}{\sqrt M}
\min\!\left\{\sqrt A,\sqrt{2\{\log(2A)+1\}}\right\}.
\label{supp:eq:mc_score}
\end{equation}
Hoeffding and a union bound give
$\Pr(\max_a|Z_a|\ge t)\le
\min\{1,2A\exp[-Mt^2/(2\Vmax^2)]\}$.  Integrating at
$t_0^2=2\Vmax^2\log(2A)/M$ yields
$\E\max_a|Z_a|^2\le2\Vmax^2\{\log(2A)+1\}/M$; also
$\E\max_a|Z_a|^2\le\sum_a\E Z_a^2\le A\Vmax^2/M$.
Minkowski and $e^M\le\eta_R+\geff\max_a|Z_a|$ prove
\eqref{supp:eq:mc_score}.  The reward error is added once, not per action.

\subsection*{OA.7: Moment-localized score perturbations}\label{supp:oa7}

Let $E(x):=\max_a|q(x,a)-Q_V(x,a)|$ and let $D(x)$ be the true regret of the
$q$-greedy action.  Assume $\|E\|_{r,\rho}\le\delta_r$ with $p<r<\infty$.
For every $t>0$ at which the margin is valid at scale $2t$,
\begin{equation}
\|D\|_{p,\rho}^p
\le C_{\mathrm{marg}}(2t)^{p+\alpha}
+2^p\delta_r^r t^{p-r}.
\label{supp:eq:moment_localized}
\end{equation}
Indeed, on $\{E\le t\}$, $D\le2t$ and $D>0$ implies $\Delta_Q\le2t$; on
$\{E>t\}$, $D\le2E$ and
$\E[E^p\mathbf1\{E>t\}]\le\delta_r^r t^{p-r}$.

For $C_{\mathrm{marg}},\delta_r>0$, optimizing the right-hand side gives
\begin{equation}
\begin{aligned}
t_*&=\left[\frac{(r-p)\delta_r^r}
{2^\alpha(p+\alpha)C_{\mathrm{marg}}}\right]^{1/(r+\alpha)},\\
\|D\|_{p,\rho}&\lesssim_{p,r,\alpha}
C_{\mathrm{marg}}^{(r-p)/(p(r+\alpha))}
\delta_r^{r(p+\alpha)/(p(r+\alpha))}.
\end{aligned}
\label{supp:eq:moment_optimizer}
\end{equation}
This optimized bound requires the margin at scale $2t_*$; otherwise minimize
\eqref{supp:eq:moment_localized} over admissible thresholds and compare with
$\|D\|_{p,\rho}\le2\delta_r$.  If $\delta_r=0$, then $D=0$ almost everywhere.
At $r=p$ the split gives no improved power; the formal $r\to\infty$ limit
recovers the uniform-error exponents only with $L^\infty$ control.

\end{document}